\documentclass[onefignum,onetabnum]{siamart250211}
\usepackage{amsfonts}
\usepackage[margin=1in]{geometry}
\usepackage{todonotes}
\usepackage{enumerate}
\usepackage{tikz-cd}
\usepackage{amssymb}
\usepackage{graphicx}  % for including images
\usepackage{subcaption}  % for subfigures
\usepackage{mathtools}
\usepackage{enumitem}

\usepackage[normalem]{ulem}
\usepackage{enumitem}
\usepackage{bbm}
\usepackage{bm}
\usepackage{hyperref}
\usepackage{mathtools}
\mathtoolsset{showonlyrefs}

\newtheorem{convention}[theorem]{Convention}
\newtheorem{remark}[theorem]{Remark}
\newtheorem{mydef}[theorem]{Definition}

\newcommand{\CCOT}{\mathsf{CCOT}}
\newcommand{\COT}{\mathsf{COT}}
\newcommand{\UCOT}{\mathsf{UCOT}}
\newcommand{\W}{\mathsf{W}}

\newcommand{\CGW}{\mathsf{CGW}}
\newcommand{\GW}{\mathsf{GW}}

\newcommand{\UOT}{\mathsf{UOT}}
\newcommand{\Mm}{\mathbb{M}}
\newcommand{\Co}{\mathcal{C}}
\newcommand{\R}{\mathbb{R}}

\newcommand{\bsum}[2]{\sum\limits_{#1}^{#2}}

\newcommand{\define}[1]{\textbf{#1}}

\title{Conic Formulations of Transport Metrics for Unbalanced Measure Networks and Hypernetworks}
\author{
  Mary Chriselda Antony Oliver\thanks{Department of Applied Mathematics and Theoretical Physics, The University of Cambridge, Wilberforce Rd, CB3 OWA, Cambridge, UK
  (\email{mca52@cam.ac.uk}).}
  \and
  Emmanuel Hartman\thanks{ Department of Mathematics, University of Houston, Houston, USA
  (\email{ehartma2@cougarnet.uh.edu}).}
  \and
  Tom Needham\thanks{Mathematics Department, Florida State University, Academic Way, Tallahassee, FL, 32306, USA
  (\email{tneedham@fsu.edu}).}
}

\begin{document}
\maketitle
\begin{abstract}
The Gromov-Wasserstein (GW) variant of optimal transport, designed to compare probability densities defined over distinct metric spaces, has emerged as an important tool for the analysis of data with complex structures, such as ensembles of point clouds or networks. To overcome certain limitations, such as the restriction to comparisons of measures of equal mass and sensitivity to outliers, several unbalanced or partial transport relaxations of the GW distance have been introduced in recent literature. This paper focuses on the conic Gromov-Wasserstein (CGW) distance introduced by S\'{e}journ\'{e}, Vialard, and Peyr\'{e} \cite{NEURIPS2021}. We provide a novel formulation in terms of semi-couplings, and extend the framework beyond the metric measure space setting, to compare more general network and hypernetwork structures. With this new formulation, we establish several fundamental properties of the CGW metric, including its scaling behavior under dilation, variational convergence in the limit of vanishing mass variation, and comparison bounds with established optimal transport metrics. We further derive quantitative bounds that characterize the robustness of the CGW metric to perturbations in the underlying measures. The hypernetwork formulation of CGW admits a simple and provably convergent block coordinate ascent algorithm for its estimation. We demonstrate the computational tractability and scalability of our approach through experiments on synthetic and real-world high-dimensional and structured datasets.
\end{abstract}


\medskip

\section{Introduction}

Optimal transport (OT) techniques have emerged as a fundamental tool in applied mathematics~\cite{Villani_2003,santambrogio2015optimal,
peyre2019computational}. Broadly speaking, OT theory encompasses methods for comparing probability distributions by determining the best way to allocate mass from one distribution to the other, where optimality is measured by a chosen objective function. Classical OT is defined in the setting of probability distributions over a fixed metric space, where the objective involves an integral of the metric function. The present paper is primarily concerned with the \emph{Gromov-Wasserstein} (GW) ~\cite{Memoli_2007,F_Memoli_2011} variant of OT, which allows comparison of distributions defined on distinct metric spaces, by replacing the linear objective of classical OT with a quadratic objective defined in terms of both of the relevant metrics. In particular, we focus on a recently introduced extension of the GW distance, namely, the \emph{conic Gromov-Wasserstein distance} (CGW) which allows comparisons between general positive measures (i.e., not necessarily probability distributions) on distinct metric spaces~\cite{NEURIPS2021}. 

In our study of the CGW distance, we provide a new formulation, leading to extensions which enable comparison in more general classes of objects. Our formulation allows us to derive  theoretical properties of the metric, including a precise characterization of how the classical GW distance arise as a variational limit of CGW distance, and a result which shows that the CGW distance is robust to noise. The novel formulation and extension also suggests an efficient numerical scheme for approximating CGW distance. 

With a view toward precisely describing the contributions of this paper, we now briefly discuss some details of the OT framework and its extensions. Let $(X,\mathsf{d}_X)$ be a Polish space. For $p \in [1,\infty]$, the \emph{$p$-Wasserstein distance} between Borel probability measures $\mu$ and $\nu$ is 
\begin{equation}\label{eqn:wasserstein_distance}
\W_p(\mu,\nu) \coloneqq \inf_{\pi \in \Pi(\mu,\nu)} \|\mathsf{d}_X\|_{\mathrm{L}^p(\pi)},
\end{equation}
where $\Pi(\mu,\nu)$ is the space of \emph{measure couplings} of $\mu$ and $\nu$ (see Section \ref{subsection:notation} for details). This family of metrics is canonical, due to its well-understood theoretical~\cite{Villani_2003} and computational~\cite{peyre2019computational} properties. However, it has some obvious shortcomings. First, it is unable to compare distributions with different total masses, which may be unnatural from the perspective of modeling in real applications. Second, due to its requirement of exactly coupling masses between $\mu$ and $\nu$, it is sensitive to noise in the form of outlier contamination of the measures. There has been significant interest in extensions of this framework which overcome these deficiencies, see~\cite{guittet2002extended,
figalli2010optimal,
caffarelli2010free,Laetitia_2020,raghvendra2024new} for more details. Of particular interest to this paper is the \emph{Wasserstein-Fisher-Rao} framework~\cite{chizat2018interpolating,chizat2018unbalanced,Liero_2015}, which defines a family of metrics interpolating between the Wasserstein distance and the Fisher-Rao metric from information theory, via a certain construction involving the metric cone over $X$.

Now, consider the GW setting of distinct metric spaces. Let $\mathcal{M}_X = (X,\mu_X,\mathsf{d}_X)$ and $\mathcal{M}_Y = (Y,\mu_Y,\mathsf{d}_Y)$ be \emph{metric measure spaces}; i.e., compact metric spaces endowed with fully supported Borel probability measures. For $p \in [1,\infty]$, the \emph{$p$-Gromov-Wasserstein distance} between $\mathcal{M}_X$ and $\mathcal{M}_Y$ is 
\begin{align}\label{eqn:gromov_wasserstein_distance}
\GW_p(\mathcal{M}_X,\mathcal{M}_Y) \coloneqq \frac{1}{2} \inf_{\pi \in \Pi(\mu_X,\mu_Y)} \|\mathsf{d}_X - \mathsf{d}_Y\|_{\mathrm{L}^p(\pi \otimes \pi)}.
\end{align}
Explicitly, for $p<\infty$ we have
\begin{align}
&\GW_p(\mathcal{M}_X,\mathcal{M}_Y) \\
& \quad \quad =\frac{1}{2} \inf_{\pi \in \Pi(\mu_X,\mu_Y)} 
\Biggl(
\int_{(X \times Y)^2} 
\bigl| \mathsf{d}_X(x,x') - \mathsf{d}_Y(y,y') \bigr|^p \; 
\mathrm{d}\pi(x,y) \, \mathrm{d}\pi(x',y') 
\Biggr)^{1/p}.
\end{align}
where $\mathsf{d}_X - \mathsf{d}_Y$ is considered as a function $(X \times Y) \times (X \times Y) \to \R$. Similar to the situation in classical OT, the GW distance is limited by its restriction to comparison of probability measures, and by its sensitivity to noise; several approaches have been introduced to mitigate these issues, e.g.,~\cite{Laetitia_2020,Fatras_2021,raghvendra2024new,Robust_GW,chhoa2024metric}. We focus on the \emph{conic Gromov-Wasserstein  distance} ($\CGW$)~\cite{NEURIPS2021}, which is the appropriate extension of the Wasserstein-Fisher-Rao construction to the GW setting, once again established via conic geometry. 

The structure of the paper and our main contributions are described below. 

\smallskip
\noindent {\bf Section \ref{sec:background}} introduces notational conventions that will be used throughout the rest of the paper, as well as some overarching background concepts from optimal transport.

\smallskip
\noindent {\bf Section \ref{sec:CGWD}} describes, in detail, the main construction to be studied in the paper, the \emph{conic Gromov-Wasserstein distance} ($\CGW$). In particular:
\begin{itemize}[leftmargin = *]
\item {\bf New Formulation.} We give an alternative formulation of the CGW distance in terms of \emph{semi-couplings}; this is Definition \ref{CGW_ours}, which is shown in Proposition \ref{prop:CGW_semi_coupling} to be equivalent to the original definition of~\cite{NEURIPS2021} (recalled in Definition \ref{def:CGW_Sejourne}). 
\item {\bf Generalized Setting.} In defining our version of CGW distance, we extend its purview beyond the space of metric measure spaces to the space of \emph{measure networks}, or Polish spaces endowed with arbitrary measurable kernels~\cite{Chowdhury_2019}. We show that the CGW distance defines a pseudometric on the space of these objects, and characterize the distance-zero equivalence relation, in Theorem \ref{thm:psuedometric}.
\end{itemize}

\smallskip
\noindent {\bf Section \ref{section:properties_CGW}} is concerned with theoretical properties of the CGW distance. We consider:
\begin{itemize}[leftmargin = *]
\item {\bf Scaling Properties.} Since the CGW distance is no longer constrained to comparisons of distributions with the same total mass, it is important to understand how the scaling of mass affects the distances. We collect several results which characterize this behavior in Section \ref{subsec:scaling}.
\item {\bf Variational Convergence.} Our definition of CGW distance includes a parameter $\delta$ which essentially controls the objective cost of rescaling masses when performing transport between distributions. We show in Theorem \ref{thm:Gamma_convergence} and Corollary \ref{cor:Gamma_convergence_to_GW} that as $\delta \to \infty$, the CGW distance converges to the classical GW distance, in a precise sense (via the language of $\Gamma$-convergence). This is analogous to \cite[Theorem 5.10]{chizat2018unbalanced}, which shows a similar convergence in the Wasserstein-Fisher-Rao setting.
\item {\bf Robustness Properties.} We prove in Theorem \ref{thm:robustness} and Corollary \ref{cor:CGW_robustness} that, unlike the classical GW distance (see Proposition \ref{prop:GW_not_robust}), the CGW distance is robust to noise in the input distributions. 
\end{itemize}

\smallskip
\noindent {\bf Section \ref{section:CCOOT}} further extends the CGW framework to a Co-Optimal transport formulation~\cite{vayer2020co, Chowdhury_2024}. While this extension is interesting in its own right, a major motivation is that it allows us to develop a new computational framework for approximating CGW distance.
\begin{itemize}[leftmargin = *]
\item {\bf Extension and Metric Properties.} Intuitively, the CGW distance is designed to compare data structures consisting of a kernel $\omega:X \times X \to \R$ defined over a measure space $(X,\mu)$. This section broadens the scope to compare kernels of the form $\omega: X \times Y \to \R$, defined over a pair of measure spaces $(X,\mu)$ and $(Y,\nu)$. Such a structure is referred to as a \emph{measure hypernetwork}, following the terminology of~\cite{Chowdhury_2024}. Theorem \ref{thm:psuedometric2} shows that the extended distance, referred to as the \emph{conic Co-Optimal Transport distance} ($\CCOT$), is a pseudometric on the space of measure hypernetworks.
\item {\bf Connection to CGW.} There is a natural embedding of the space of measure networks into the space of measure hypernetworks; we show in Theorem \ref{thm:CGW_CCOOT_equivalence} that (under certain conditions), the CCOT and CGW distances agree on the image of this embedding.
\item {\bf Numerical Algorithm.} The semi-coupling formulation leads to a new algorithm for numerically approximating the CGW distance between hypernetworks, which is inspired by the one developed in~\cite{bauer2022SRNF} for the Wasserstein-Fisher-Rao distance. This leads to the main narrative of the paper. We first develop an algorithm based on our new semi-coupling distance formulation. Combining this algorithm with the equivalence between CCOT and CGW presented above yields a novel method for approximating CGW distances. This combination of ideas provides the first computationally viable method for applying CGW distances to real-world datasets.
\end{itemize}

\smallskip
\noindent {\bf Section \ref{section:Numerical_Experiments}} illustrates the computational framework described above on numerical experiments on real and synthetic data. These show, in particular, that our new algorithm for computing CGW distance is substantially more efficient than the naive one presented in \cite{NEURIPS2021}. 

\section{Background on Optimal Transport}\label{sec:background}
This section collects background material which will be useful throughout the paper.

\subsection{Notation}\label{subsection:notation}
We begin by recalling standard terminology and fixing general conventions and notation.

\begin{itemize}[leftmargin=*]
\item Let $X$ and $Y$ be Polish spaces, i.e., separable, completely metrizable topological spaces. 
\item Metric spaces $(X,\mathsf{d}_X)$ will be assumed to be separable and complete.
\item Let $\mathbb{M}(X)$ denote the set of positive Radon measures on $X$. 
 \item Absolute continuity of a measure $\mu$ with respect to a measure $\nu$ is denoted $\mu \ll \nu$.
\item We use $\mathbb{P}(X)$ to denote the set of Borel probability measures on $X$.
 \item Let $\mu_X \in \mathbb{P}(X)$ be a probability measure and $T: X \to Y$ be a Borel-measurable function. The \define{pushforward} of $\mu_X$ by $T$ is the measure $T_\#\mu_X \in \mathbb{P}(Y)$ defined by $T_\#\mu_X(B)=\mu_X(T^{-1}(B))$, for all measurable sets $B \subset Y$. 
 \item We generically use $\operatorname{Pr}_0$ and $\operatorname{Pr}_1$ to denote projections from a product of two sets to its left and right factors, respectively. That is, given $X$ and $Y$, $\operatorname{Pr}_0:X \times Y \to X$ and $\operatorname{Pr}_1:X \times Y \to Y$ are the obvious maps.
 \item Given two probability measures $\mu_X \in \mathbb{P}(X)$ and $\mu_Y \in \mathbb{P}(Y)$, we denote the set of \define{couplings} as $\Pi(\mu_X,\mu_Y)$. That is, $\pi \in \Pi(\mu_X,\mu_Y)$ if $\pi \in \mathbb{P}(X \times Y)$ and $\operatorname{Pr}_0\#\pi=\mu_X$, $\operatorname{Pr}_1\#\pi=\mu_Y$.
\end{itemize}

\subsection{Conic Formulations for Unbalanced Optimal Transport}
As we described in the introduction, the classical Wasserstein distance \eqref{eqn:wasserstein_distance} between probability  distributions over a fixed metric space has been extended in several ways to allow comparison of measures which may have unequal total mass. This type of extension is generally referred to as \emph{unbalanced optimal transport}. We now describe the details of the particular approach taken in, e.g.,~\cite{chizat2018interpolating,chizat2018unbalanced,NEURIPS2021}.

In unbalanced optimal transport, \emph{cone spaces} offer a natural geometric setting to define and compute transport distances. Topologically, the \define{cone} over a metric space $(X,\mathsf{d}_X)$ is simply $(X\times\R^{\geq0})/(X\times\{0\})$~\cite{Burago_2001}. However, in the context of unbalanced mass transport, it is important to consider how the cone inherits a distance from the base space $X$. We will consider a general class of distances defined on the cone defined as follows.

\begin{mydef}[Cone Space]\label{def:cone_dist}
    Given a metric space $(X,\mathsf{d}_X)$ and $\delta\in \R^{>0}$ the \define{cone over $X$} with angle $\delta$, denoted $\Co_\delta[X]$, is defined to be the set $(X\times\R^{\geq0})/(X\times\{0\})$, equipped with a distance inherited from the metric on $X$ given by
    \begin{align}
        \mathsf{d}_{\Co_\delta[X]}:\Co_\delta[X]\times \Co_\delta[X]&\to \R^{\geq0}\\
        \mathsf{d}_{\Co_\delta[X]}([x,r],[y,s])^2&= 4\delta^2\left(r^2 + s^2 - 2rs\Omega\left(\frac{\mathsf{d}_X(x,y)}{2\delta}\right)\right) 
    \end{align}    
    where $\Omega:\R^{\geq0}\to [0,1]$ is a function satisfying the axioms   
    \begin{itemize}[leftmargin=*]
        \item $\Omega(0)=1$;
        \item for all $\varepsilon>0$, $\Omega(\varepsilon)<1$;
        \item $\Omega(\varepsilon + \varepsilon')\leq \Omega(\varepsilon)\Omega(\varepsilon')$.
    \end{itemize}
\end{mydef}

\begin{convention}
     For simplicity in our notation, if $\delta=1/2$ (a typical choice) we denote the cone and the associated distance by $(\Co[X],\mathsf{d}_{\Co[X]})$.  
     
     Also note that we suppress the dependence on the function $\Omega$ from the notation $\mathsf{d}_{\Co_\delta[X]}$. In general, various quantities defined below which depend on a choice of $\Omega$ will have this dependency suppressed from notation. On the other hand, the choice of $\delta$ is typically represented in the notation, as the value of $\delta$ will play a direct role in the analysis at various points.
\end{convention}

We note that $\mathsf{d}_{\Co[X]}$, in general, does not define a true metric since it may not satisfy the triangle inequality; however, it is always a positive, symmetric and definite. Another ingredient for comparing measures of different total mass is to relax the classical notion of couplings between measures. Instead of requiring both marginals of a single joint distribution to match the input measures, we fix only one marginal of two separate joint distributions. This is made precise in the following definition.

\begin{mydef}[Semi-couplings \cite{chizat2018unbalanced}]\label{def:semi_coupling}
Let $(X,\mu_X)$, $(Y,\mu_Y)$ be two measure spaces. The set of all \define{semi-couplings} from $\mu_X$ to $\mu_Y$ is defined as,
\begin{equation*}\label{eqn:semi_couplings_def}
    \Gamma(\mu_X,\mu_Y)=\left\{ (\gamma_0,\gamma_1)\in \Mm(X\times Y)^2| (\operatorname{Pr}_0)_\#\gamma_0=\mu_X,(\operatorname{Pr}_1)_\#\gamma_1=\mu_Y\right\}.
\end{equation*}
where $\operatorname{Pr}_0$ and $\operatorname{Pr}_1$ are the projections onto the first and second factor respectively.
\end{mydef}  
Using cone geometry and semi-couplings, we can now formulate a generalized transport cost that accounts for both spatial displacement and mass variation defining a distance between measures with unequal mass.
\begin{mydef}[conic Formulation of Unbalanced Optimal Transport \cite{Liero_2015}]\label{def:WFR}
Given a metric space $(X,\mathsf{d}_X)$, two Radon measures $\mu, \nu \in \Mm(X)$, $\delta\in\R^{>0}$, and choice of cone metric $\mathsf{d}_{\Co_\delta[X]}$,  the \define{unbalanced optimal transport distance} between $\mu$ and $\nu$ is defined as, 
\begin{equation*}
    \UOT_\delta(\mu,\nu)^2 = \inf\limits_{(\gamma_0,\gamma_1)\in \Gamma(\mu,\nu)} \mathbf{J}_\delta (\gamma_0,\gamma_1),
\end{equation*}
where $\gamma_0,\gamma_1 \ll \gamma$ such that,
\begin{equation}\label{eqn:UOT_loss}
\mathbf{J}_\delta(\gamma_0,\gamma_1) = \int_{X\times X}\mathsf{d}_{\Co_\delta[X]}\left(\left[x,\sqrt{\frac{\textnormal{d}\gamma_0}{\textnormal{d}\gamma}(x,y)}\right],\left[y,\sqrt{\frac{\textnormal{d}\gamma_1}{\textnormal{d}\gamma}(x,y)}\right]\right)^2 \textnormal{d}\gamma.
\end{equation}
\end{mydef} 

This is well defined, as the value of \eqref{eqn:UOT_loss} is 1-homogeneous with respect to the mass variable. We note that, if $\mathsf{d}_{\Co_\delta[X]}$ satisfies the triangle inequality, then $\UOT_\delta$ defines a metric for Radon measures on $X$. Different choices of the function $\Omega$ yield different notions of unbalanced optimal transport, some of which correspond to well-known metrics. Below we provide two examples of such UOT metrics.
\begin{itemize}[leftmargin=*]
\item \textbf{Wasserstein-Fisher-Rao (or commonly known as Hellinger-Kantorovich) Distance.}
Let $\Omega(z)= \overline{\cos}\left(z\right):=\cos(\min(z,\pi/2))$. Then the corresponding cone distance is given by
\begin{equation}\label{eqn:HK_kernel}
        \mathsf{d}_{\Co_\delta[X]}([x,r],[y,s])^2 = 4\delta^2\left(r^2 + s^2 - 2rs \overline{\cos}\left(  \frac{\mathsf{d}_X(x,y)}{2\delta} \right) \right).
\end{equation}
This defines a metric on $\Co_\delta[X]$ \cite[Definition~3.6.16]{Burago_2001}. Moreover, the associated unbalanced optimal transport distance is the Wasserstein-Fisher-Rao (or Hellinger Kantorovich) distance which was introduced independently by \cite{Liero_2015} and \cite{chizat2018unbalanced}.
\item \textbf{Gaussian Hellinger Distance.} 
Let $\Omega(z) = \exp(-z^2)$. Then the cone distance takes the following form
\begin{equation}\label{eqn:GH_kernel}
        \mathsf{d}_{\Co_\delta[X]}([x,r],[y,s])^2 = 4\delta^2 \left( r^2 + s^2 - 2rs \exp\left( -\frac{\mathsf{d}_X(x,y)^2}{4\delta^2} \right) \right).
\end{equation}
Similarly, this gives a metric on the cone and the associated UOT distance was proposed in \cite{Liero_2015}.
\end{itemize}

\subsection{Gromov-Wasserstein Distances} The Wasserstein distance was adapted by M\'{e}moli~\cite{Memoli_2007,F_Memoli_2011} to allow for comparisons between probability measures defined over distinct metric spaces, via the formula \eqref{eqn:gromov_wasserstein_distance}. In fact, this formula makes sense even when the functions $\mathsf{d}_X$ and $\mathsf{d}_Y$ are not assumed to be metrics. This led to a generalized notion of Gromov-Wasserstein distance, introduced in~\cite{Chowdhury_2019}, whose details we now recall. 

\begin{mydef}[Measure Network~\cite{Chowdhury_2019}]\label{def:measure_network} A \define{measure network} is of the form $\mathcal{N}_X=(X,\mu_X,\omega_X)$, where $X$ is a Polish space, $\mu_X \in \Mm(X)$ with $\mu_X(X) < \infty$, and $\omega_X:X\times X \to \R^{\geq 0}$ is a measurable, bounded function. If $\mu_X(X) = 1$, we refer to $\mathcal{N}_X$ as a \define{probability measure network}. 
\end{mydef}

In the case that $\omega_X$ is a metric on $X$, the above recovers the notion of a \emph{metric measure space}. On the other hand, this more general setting allows for kernels which appear naturally in applications; e.g., $X$ is the node set of a graph and $\omega_X$ is a (weighted) adjacency function.

The definition of Gromov-Wasserstein distance extends without change to compare probability measure networks. That is, if $\mathcal{N}_X$ and $\mathcal{N}_Y$ are probability measure networks, the \define{$p$-Gromov-Wasserstein (GW) distance} between them is given by 
\begin{align*}
&\GW_p(\mathcal{N}_X,\mathcal{N}_Y) \coloneqq \frac{1}{2} \inf_{\pi \in \Pi(\mu_X,\mu_Y)} \|\omega_X - \omega_Y\|_{\mathrm{L}^p(\pi \otimes \pi)}\\
& \qquad \qquad = \frac{1}{2} \inf_{\pi \in \Pi(\mu_X,\mu_Y)} \left(\int_{(X \times Y)^2} |\omega_X(x,x') - \omega_Y(y,y')|^p \mathrm{d}\pi(x,y) \mathrm{d}\pi(x',y') \right)^{1/p},
\end{align*}
where the integral expression is valid in the $p < \infty$ case. It is shown in \cite{Chowdhury_2019} that $\GW_p$ still defines a pseudometric on the space of probability measure networks, and the distance-zero equivalence relation is completely characterized (cf.~Theorem \ref{thm:psuedometric} below). 

\subsection{Conic Gromov-Wasserstein Distance}

For the rest of the paper, a \define{metric measure space (mm-space)} $\mathcal{M}_X = (X,\mu_X,\mathsf{d}_X)$ is allowed to have arbitrary (non-negative, finite) total mass $\mu_X(X)$. Analogous to the conic unbalanced optimal transport constructions described above, a similar trick can be applied to the GW distance to extend it to a distance between mm-spaces with distinct total masses. Such a construction was introduced in~\cite{NEURIPS2021}, and we recall the definition here.

\begin{mydef}[conic Gromov-Wasserstein \cite{NEURIPS2021}]\label{def:CGW_Sejourne}
Let $\mathcal{M}_X=(X,\mu_X,\mathsf{d}_X)$ and $\mathcal{M}_Y=(Y,\mu_Y,\mathsf{d}_Y)$ be mm-spaces and let $\mathsf{d}_{\Co[\R]}$ be a metric on the ($\delta = \frac{1}{2}$) cone over $\R$. The  \define{S\'ejourn\'e-Vialard-Peyr\'e conic Gromov-Wasserstein distance} is defined as
\begin{equation*}
\CGW_{\mathrm{SVP}}(\mathcal{M}_X,\mathcal{M}_Y) \coloneqq \inf\limits_{\alpha\in\mathrm{A}(\mu_X,\mu_Y)}\mathbf{K}(\alpha)^{1/2}, 
\end{equation*} 
where:
\begin{itemize}[leftmargin=*]
    \item $\mathrm{A}(\mu_X,\mu_Y)$ is the space of measures 
    \begin{equation}\label{eq-const-conic}
    \begin{split}
\mathrm{A}(\mu_X,\mu_Y) \coloneqq &\left\{
\alpha\in\Mm(\Co[X]\times\Co[Y]) \mid 
 \int_{\R_+} r^2 \,\textnormal{d}(\operatorname{Pr}_0)_\# \alpha (\cdot,r)=\mu_X,\right.\\
&\hspace{2in} \left.\int_{\R_+} s^2 \,\textnormal{d}(\operatorname{Pr}_1)_\# \alpha(\cdot,s)=\mu_Y\right\};
\end{split}
\end{equation}
\item the \define{SVP loss} $\mathbf{K}:\mathrm{A}(\mu_X,\mu_Y) \to \R$ is defined by 
\begin{align}
 \mathbf{K}(\alpha)\coloneqq &\int_{(\Co[X]\times \Co[Y])^2}  \mathsf{d}_{\Co[\R]}\left([\mathsf{d}_X(x,x'), rr'],[\mathsf{d}_Y(y,y'), ss']\right)^2 \label{eqn:def_cgw} \\
 &\hspace{2in} \textnormal{d}\alpha([x,r], [y,s])\textnormal{d}\alpha([x',r'], [y',s']). \nonumber
\end{align}
\end{itemize}
\end{mydef}

\begin{remark}[Intuition for the cone over $\mathbb{R}$]\label{remark:intuition_cone_over_R}
The conic Gromov--Wasserstein (CGW) distance can be interpreted as lifting measures to a cone over \(\mathbb{R}\). Each point \(x \in X\) with mass \(\mu(x)\) is assigned a radial coordinate \(r = \sqrt{\mu(x)} \ge 0\), so that larger masses lie farther from the cone apex. A pair \((x,x')\) is then represented by \((\mathsf d_X(x,x'), rr')\), where \(\mathsf d_X(x,x')\) captures geometry and \(rr'\) encodes mass interaction; similarly for pairs in \(Y\). The CGW functional therefore compares these lifted pairwise representations. Unlike classical GW, the interaction term \(rr'\) alters the comparison cost, so CGW generally differs from GW. The lifting is performed over \(\mathbb{R}\), rather than over \(X\) or \(Y\), because the functional depends only on pairwise distances and mass interactions. From this viewpoint, the weighted projection operators in~\eqref{eq-const-conic} arise naturally from the quadratic radial representation of measures; see, e.g.,~\cite{Liero_2015, friesecke2021barycenters}. This perspective clarifies the form of~\eqref{eqn:def_cgw} and motivates Definition~\ref{CGW_ours}, where the comparison is carried out on lifted pairwise representations while preserving the underlying relational structure.
\end{remark}

Observe that, although it also relies on cone metrics, this definition has a rather different flavor than the UOT distance construction in Definition \ref{def:WFR}. The first goal of this paper is to reformulate CGW distance in terms of semi-couplings. In doing so, we also extend it to the measure network setting.

\section{Conic Gromov-Wasserstein Distance for Unbalanced Measure Networks}\label{sec:CGWD}
In this section, we explore a semi-coupling formulation for an unbalanced variant of the Gromov-Wasserstein distance within the general framework of measure networks, analogous to the formulation of the Wasserstein-Fisher-Rao distance described in  \cite{chizat2018unbalanced,chizat2018interpolating} and stated in Definition \ref{def:semi_coupling}. We, in turn, show that a special case of our formulation is equivalent to the conic Gromov-Wasserstein distance introduced in~\cite{NEURIPS2021}.

\subsection{Semi-Coupling Formulation of Conic GW Distance}\label{subsection:conic_formulation_setting}

The main construction that we study in this paper is defined as follows. 

\begin{mydef}[Semi-Coupling Formulation of Conic Gromov-Wasserstein Distance]\label{CGW_ours}
Let $\mathcal{N}_X= (X,\mu_X,\omega_X)$ and $\mathcal{N}_Y=(Y,\mu_Y,\omega_Y)$ be measure networks, $\delta$ a positive real number, and $\mathsf{d}_{\Co_\delta[\R]}$ a metric on the cone over $\R$, as in Definition \ref{def:cone_dist}. We define the associated \define{conic Gromov-Wasserstein distance} ($\CGW$) as
\begin{equation*}
\CGW_\delta(\mathcal{N}_X,\mathcal{N}_Y) \coloneqq \inf\limits_{(\gamma_0,\gamma_1)\in \Gamma(\mu_X,\mu_Y)} \mathbf{H}_\delta (\gamma_0,\gamma_1)^{1/2},
\end{equation*}
where:
\begin{itemize}[leftmargin=*]
\item $\Gamma(\mu_X,\mu_Y)$ is the space of semi-couplings, as in Definition \ref{def:semi_coupling};
    \item the \define{$\CGW$ loss} $\mathbf{H}_\delta: \Gamma(\mu_X,\mu_Y) \to \R$ is defined by 
\begin{equation}\label{eqn:H_delta}
\mathbf{H}_\delta (\gamma_0,\gamma_1) =\int_{(X\times Y)^2}\mathsf{d}_{\Co_\delta[\R]}\left(p^{\gamma_0,\gamma}_X(x,y,x',y'),p^{\gamma_1,\gamma}_Y(x,y,x',y')\right)^2\textnormal{d}\gamma(x,y)\textnormal{d}\gamma(x',y'),
\end{equation}
with $p^{\gamma_0,\gamma}_X, p^{\gamma_1,\gamma}_Y: X \times Y \times X \times Y \to \Co_\delta[\R]$ the functions defined by
\[
    p^{\gamma_0,\gamma}_X(x,y,x',y') = \left[\omega_X(x,x'),\sqrt{\frac{\textnormal{d}\gamma_0}{\textnormal{d}\gamma}(x,y)\frac{\textnormal{d}\gamma_0}{\textnormal{d}\gamma}(x',y')}\right]
\]
and
\[
    p^{\gamma_1,\gamma}_Y(x,y,x',y') = \left[\omega_Y(y,y'),\sqrt{\frac{\textnormal{d}\gamma_1}{\textnormal{d}\gamma}(x,y)\frac{\textnormal{d}\gamma_1}{\textnormal{d}\gamma}(x',y')}\right]
\]
where the square-root of the mass is motivated by Remark \ref{remark:intuition_cone_over_R};
    \item $\gamma$ is any reference measure such that $\gamma_0,\gamma_1 \ll\gamma$. The value of $\mathbf{H}_\delta(\gamma_0,\gamma_1)$ does not depend on the choice of $\gamma$, due to 1-homogeneity (cf. \cite[Definition 3.3]{chizat2018unbalanced}). 
\end{itemize}
\end{mydef} 

\begin{figure}
    \centering
    \includegraphics[width=0.95\linewidth]{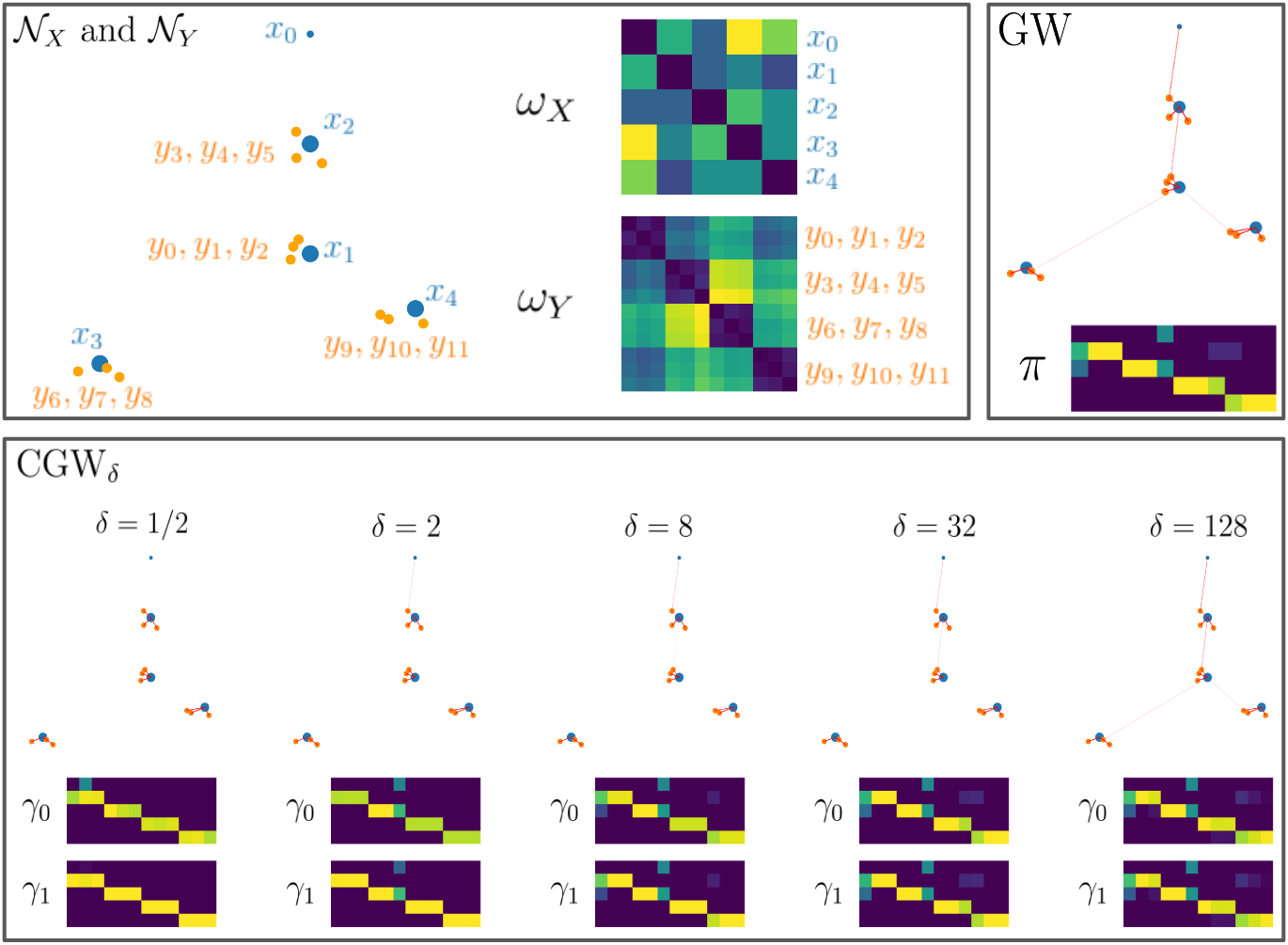}
    \caption{Example of matching probability networks with the $\GW_2$ distance and the $\CGW_\delta$ distance for increasing values of $\delta$. We visualize the matching with lines between the support points of $\mathcal{N}_X$ (blue) and $\mathcal{N}_Y$ (orange) with the opacity of the lines determined by $\pi$ and $\frac{\gamma_0}{\gamma} \cdot \frac{\gamma_1}{\gamma}$ respectively. The $i$th row of each semi-coupling corresponds to the $x_i$ and the $j$th column corresponds to $y_j$. For small values of $\delta$, note that $\gamma_0$ and $\gamma_1$ do not agree. In particular, $\gamma_0$ transports mass from $x_0$ to an arbitrary target point where the corresponding entry of $\gamma_1$ is negligible. This effectively destroys the mass of $x_0$ which is expensive to transport. In contrast, for large values of $\delta$, the outlier is matched to the same orange point as in the GW distance. This illustrates the $\Gamma$-convergence of CGW to GW distance as $\delta \to \infty$ (see Section ~\ref{sec:Variational_Convergence} for more details).}
    \label{fig:matching}
\end{figure}

\begin{convention}
    For the remainder of the paper, unless stated otherwise, we consider a fixed but arbitrary cone metric $\mathsf{d}_{\Co_\delta[\R]}$, so that we can avoid reference to the dependence on this choice when discussing the CGW distance.
\end{convention}

\begin{convention}\label{conv:simplified_integrand}
    Below, for brevity, we denote by $\mathsf{d}_{\Co_\delta[\R]}\left(\omega_X,\frac{\textnormal{d}\gamma_0}{\textnormal{d}\gamma},\omega_Y,\frac{\textnormal{d}\gamma_1}{\textnormal{d}\gamma}\right)$ the quantity
\begin{equation*}
\mathsf{d}_{\Co_\delta[\R]}\left(\left[\omega_X(x,x'),\sqrt{\frac{\textnormal{d}\gamma_0}{\textnormal{d}\gamma}(x,y)\frac{\textnormal{d}\gamma_0}{\textnormal{d}\gamma}(x',y')}\right],\left[\omega_Y(y,y'),\sqrt{\frac{\textnormal{d}\gamma_1}{\textnormal{d}\gamma}(x,y)\frac{\textnormal{d}\gamma_1}{\textnormal{d}\gamma}(x',y')}\right]\right).   
\end{equation*}
In particular, this notation suppresses function arguments, and will be used when doing so doesn't cause confusion.  With this convention, the CGW loss can be expressed as 
\[
\mathbf{H}_\delta(\gamma_0,\gamma_1) = \int_{(X \times Y)^2} \mathsf{d}_{\Co_\delta[\R]}\left(\omega_X,\frac{\textnormal{d}\gamma_0}{\textnormal{d}\gamma},\omega_Y,\frac{\textnormal{d}\gamma_1}{\textnormal{d}\gamma}\right)^2  {\textnormal{d}\gamma}{\textnormal{d}\gamma}.
\]
\end{convention}

\subsection{Metric Properties}
A certain equivalence relation for measure networks, called \emph{weak isomorphism}, was introduced by Chowdhury and M\'{e}moli in \cite{Chowdhury_2019}, and an appropriate generalization will be useful in the conic setting. We give the definition below in the context of measure networks without any mass constraint, whereas \cite{Chowdhury_2019} specifically considered \emph{probability} measure networks; this generalization is only superficial.

\begin{mydef}[Weak Isomorphism]\label{def:wisomorphism}
 Measure networks $\mathcal{N}_X=(X,\mu_X,\omega_X)$ and $\mathcal{N}_Y=(Y,\mu_Y,\omega_Y)$ are called \define{weakly isomorphic} if there exists a Polish space $Z$ equipped with a Borel measure $\mu_Z$ and measurable maps $f:Z\to X$ and $g:Z\to Y$ such that the following holds:
 \begin{itemize}[leftmargin=*]
        \item $f_\#\mu_Z=\mu_X$,
        \item $g_\#\mu_Z=\mu_Y$,
        \item and $\|f^\#\omega_X-g^\#\omega_Y\|_{\textnormal{L}^\infty(\mu_Z \otimes \mu_Z)} = 0$.
    \end{itemize}
where $f^\#\omega_X: Z \times Z \to \R$ is the pullback weight function  $(z,z') \mapsto \omega_X(f(z),f(z'))$, and  $g^\#\omega_Y$ is defined analogously.
\end{mydef}

A key result of Chowdhury and M\'{e}moli is that the Gromov-Wasserstein distance $\GW_p$ defines a metric on the space of probability measure networks, considered up to weak isomorphism~\cite[Theorems 2.3 and 2.4]{Chowdhury_2019}. The following result for conic Gromov-Wasserstein is analogous.

\begin{theorem}\label{thm:psuedometric}
    The conic Gromov-Wasserstein distance $\CGW_{\delta}$ is a metric on the space of measure networks, considered up to weak isomorphism.
\end{theorem}

The proof uses the following lemmas.

\begin{lemma}\label{lemma:lsc_functional}
    Given two compact measure networks $\mathcal{N}_X= (X,\mu_X,\omega_X)$ and $\mathcal{N}_Y=(Y,\mu_Y,\omega_Y)$ there exists an optimal semi-coupling $(\gamma_0^*,\gamma_1^*) \in \Gamma(\mu_X,\mu_Y)$ that attains $\CGW_\delta(\mathcal{N}_X,\mathcal{N}_Y)$.    
\end{lemma}

\begin{proof}
    This proof follows exactly as in \cite{chizat2018unbalanced}. In particular, \cite[Lemma 2.9]{chizat2018unbalanced} gives general conditions for weak* lower semi-continuity for functionals on spaces of measures, and it applies in our case to give weak* lower semi-continuity of $\mathbf{H}_\delta$. Furthermore, since $X,Y$ are compact and $\mu_X,\mu_Y$ have finite measure, we conclude $\Gamma(\mu_X,\mu_Y)$ is weak* compact (see \cite[Proposition 3.4]{chizat2018unbalanced} for more details). Thus, we can conclude the existence of minimizers.
\end{proof}

\begin{lemma}\label{lem:CGW_loss_formula}
    Let $\mathcal{N}_X$ and $\mathcal{N}_Y$ be measure networks, let $(\gamma_0,\gamma_1) \in \Gamma(\mu_X,\mu_Y)$ be a semi-coupling, and let $\gamma \in \mathbb{M}(X \times Y)$ be a reference measure such that $\gamma_0,\gamma_1 \ll \gamma$. The $\CGW$ loss $\mathbf{H}_\delta(\gamma_0,\gamma_1)$ is equal to 
    \begin{align*}
        &4\delta^2\big(\mu_X(X)^2 + \mu_Y(Y)^2\big)\\
        &\qquad -8\delta^2\int_{(X\times Y)^2} \Omega\left(\frac{|\omega_X(x,x') -\omega_Y(y,y') |}{2\delta}\right) \\
        &\hspace{1in} \cdot \sqrt{\frac{\textnormal{d}\gamma_0}{\textnormal{d}\gamma}(x,y)\frac{\textnormal{d}\gamma_1}{\textnormal{d}\gamma}(x,y)\frac{\textnormal{d}\gamma_0}{\textnormal{d}\gamma}(x',y')\frac{\textnormal{d}\gamma_1}{\textnormal{d}\gamma}(x',y')} \, \textnormal{d}\gamma(x,y)  \textnormal{d}\gamma(x',y'),
    \end{align*}
    where $\Omega$ is the function defining $\mathsf{d}_{\Co_\delta[\R]}$ (Definition \ref{def:cone_dist}).
\end{lemma}

\begin{proof}
    Using the structure of $\mathsf{d}_{\Co_\delta[\R]}$ (see Definition \ref{def:cone_dist}), and recalling the definitions of the functions $p_X^{\gamma_0,\gamma}$ and $p_Y^{\gamma_1,\gamma}$ (see Definition \ref{CGW_ours}), we have that the integrand in \eqref{eqn:H_delta} can be expressed as 
    \begin{align}
        &\mathsf{d}_{\Co_\delta[\R]}(p_X^{\gamma_0,\gamma}(x,y,x',y'),p_Y^{\gamma_1,\gamma}(x,y,x',y'))^2 \nonumber \\
        &\quad = 4\delta^2 \left(\frac{\textnormal{d}\gamma_0}{\textnormal{d}\gamma}(x,y)\frac{\textnormal{d}\gamma_0}{\textnormal{d}\gamma}(x',y') +\frac{\textnormal{d}\gamma_1}{\textnormal{d}\gamma}(x,y)\frac{\textnormal{d}\gamma_1}{\textnormal{d}\gamma}(x',y') \right) \label{eqn:CGW_loss_formula1} \\
        &\qquad -8\delta^2 \Omega\left(\frac{|\omega_X(x,x') -\omega_Y(y,y') |}{2\delta}\right)\sqrt{\frac{\textnormal{d}\gamma_0}{\textnormal{d}\gamma}(x,y)\frac{\textnormal{d}\gamma_1}{\textnormal{d}\gamma}(x,y)\frac{\textnormal{d}\gamma_0}{\textnormal{d}\gamma}(x',y')\frac{\textnormal{d}\gamma_1}{\textnormal{d}\gamma}(x',y')} \label{eqn:CGW_loss_formula2}.
    \end{align}
    Integrating the term \eqref{eqn:CGW_loss_formula2} against $\gamma \otimes \gamma$ obviously yields the second term in the proposed formula, so consider the integral of the term \eqref{eqn:CGW_loss_formula1}:
    \begin{align*}
        &4\delta^2 \int_{(X \times Y)^2} \frac{\textnormal{d}\gamma_0}{\textnormal{d}\gamma}(x,y)\frac{\textnormal{d}\gamma_0}{\textnormal{d}\gamma}(x',y') +\frac{\textnormal{d}\gamma_1}{\textnormal{d}\gamma}(x,y)\frac{\textnormal{d}\gamma_1}{\textnormal{d}\gamma}(x',y') \; \textnormal{d}\gamma(x,y)  \textnormal{d}\gamma(x',y') \\
        &\qquad \qquad = 4 \delta^2 \int_{(X \times Y)^2} \textnormal{d}\gamma_0(x,y)  \textnormal{d}\gamma_0(x',y') + 4 \delta^2 \int_{(X \times Y)^2} \textnormal{d}\gamma_1(x,y)  \textnormal{d}\gamma_1(x',y') \\
        &\qquad \qquad = 4\delta^2 \mu_X(X)^2 + 4 \delta^2 \mu_Y(Y)^2,
    \end{align*}
    where we have applied the marginal conditions of the semi-coupling $(\gamma_0,\gamma_1)$. This gives the first term of the claimed formula, so the proof is complete.
\end{proof}

Next, we prove lemmas which straightforwardly imply the triangle inequality for the conic Gromov-Wasserstein distance. We begin by stating a version of the gluing lemma for semicouplings. This is similar to the construction found in the proof of the triangle inequality for Wasserstein-Fisher-Rao distance~\cite[Theorem 3.7]{chizat2018unbalanced}, but requires some extra conditions to make the argument go through in our setting.

\begin{lemma}[Gluing Lemma for Semi-Couplings]\label{lem:gluing_lemma}
    Fix $\mu_X \in \mathbb{M}(X)$, $\mu_Y \in \mathbb{M}(Y)$, and $\mu_Z \in \mathbb{M}(Z)$. 
Let $(\gamma^{XY}_0,\gamma^{XY}_1) \in \Gamma(\mu_X,\mu_Y)$ and $(\gamma^{YZ}_0,\gamma^{YZ}_1) \in \Gamma(\mu_Y,\mu_Z)$ be semicouplings with reference measures $\gamma^{XY} \in \mathbb{M}(X \times Y)$ and $\gamma^{YZ} \in \mathbb{M}(Y \times Z)$ such that $\gamma^{XY}_0,\gamma^{XY}_1 \ll \gamma^{XY}$ and $\gamma^{YZ}_0,\gamma^{YZ}_1 \ll \gamma^{YZ}$. Denote the associated densities as $g_i^{XY} = \mathrm{d}\gamma_i^{XY}/\mathrm{d}\gamma^{XY}$ and $g_i^{YZ} = \mathrm{d}\gamma_i^{YZ}/\mathrm{d}\gamma^{YZ}$, for $i \in \{0,1\}$. There exists
\begin{enumerate}
    \item a Polish space $Q$,
    \item a measure $\lambda \in \mathbb{M}(Q)$,
    \item measurable maps $p^X:Q \to X$, $p^Y:Q \to Y$ and $p^Z:Q \to Z$, and
    \item densities $a,b,c:Q \to [0,\infty)$
\end{enumerate}
satisfying the following properties:
\begin{enumerate}
    \item $(p^X,p^Y)_\#(a^2 \lambda) = \gamma_0^{XY}$,
    \item $(p^X,p^Y)_\# (b^2 \lambda) = \gamma_1^{XY}$,
    \item $(p^X,p^Y)_\# (ab \cdot \lambda) = \sqrt{g_0^{XY} g_1^{XY}} \gamma^{XY}$,
    \item $(p^Y,p^Z)_\# (b^2 \lambda) = \gamma_0^{YZ}$,
    \item $(p^Y,p^Z)_\# (c^2 \lambda) = \gamma_1^{YZ}$, and
    \item $(p^Y,p^Z)_\# (bc \cdot \lambda) = \sqrt{g_0^{YZ} g_1^{YZ}}\gamma^{YZ}$.
\end{enumerate}
\end{lemma}

We note that, if we only required ``lifted'' measures satisfying Properties 1,2, 4, and 5 above, then this would read as a standard extension of the gluing lemma. The additional points 3 and 6 are used in the proof of the triangle inequality for CGW distance; as the objective function for CGW explicitly involves densities,  the requirement that the gluing respect some properties of densities should not be entirely unexpected.

\begin{proof}
    The proof is fairly intuitive if we assume that the densities $g_i^{XY}$ and $g_i^{YZ}$ are strictly positive. We begin by working under this assumption, then explain how to adapt the proof once this is dropped. 

    Working under the positive density assumption, take $Q = X \times Y \times Z$ and let $p^X,p^Y,p^Z$ be coordinate projections. We set $\lambda \in \mathbb{M}(Q)$ to be the usual gluing of the measures $\gamma_1^{XY}$ and $\gamma_0^{YZ}$ (valid, since the $Y$-marginals of both measures are equal to $\mu_Y$). For a given density $a:X \times Y \times Z \to [0,\infty)$, Property $1$ reads as 
    \[
    \int_{Z} a(x,y,z)^2 \mathrm{d} \lambda(x,y,z) = \mathrm{d}\gamma_0^{XY}(x,y) = g_0^{XY}(x,y) \mathrm{d}\gamma^{XY}(x,y).
    \]
    If we take $a$ to be independent of $z \in Z$, writing $a(x,y)$, we have 
    \begin{align*}
    \int_{Z} a(x,y)^2 \mathrm{d} \lambda(x,y,z) &= a(x,y)^2 \int_{Z} \mathrm{d} \lambda(x,y,z) \\
    &= a(x,y)^2 \mathrm{d}\gamma_1^{XY} = a(x,y)^2 g_1^{XY}(x,y) \mathrm{d}\gamma^{XY}(x,y),
    \end{align*}
    where we have used the fact that $\lambda$ is a gluing. Solving for $a(x,y)$, we have 
    \begin{equation}\label{eqn:gluing_A}
    a(x,y) = \sqrt{\frac{g_0^{XY}(x,y)}{g_1^{XY}(x,y)}}.
    \end{equation}
    Similarly, we take $c$ to be independent of $x$ and set 
    \begin{equation}\label{eqn:gluing_C}
    c(y,z) = \sqrt{\frac{g_1^{YZ}(y,z)}{g_0^{YZ}(y,z)}}.
    \end{equation}
    Finally, take $b(x,y,z) = 1$. By the definitions of $\lambda$, $a$, $b$ and $c$, the desired Properties 1, 2, 4, and 5 are satisfied. Now consider Property 3: 
    \begin{align*}
        \mathrm{d}(p^X,p^Y)_\#(ab \cdot \lambda)(x,y) &= \int_Z a(x,y) \mathrm{d}\lambda(x,y,z) = a(x,y) \mathrm{d}\gamma_1^{XY}(x,y) \\
        &= \sqrt{\frac{g_0^{XY}(x,y)}{g_1^{XY}(x,y)}} g_1^{XY}(x,y) \mathrm{d}\gamma^{XY}(x,y) \\
        &= \sqrt{g_0^{XY}(x,y)g_1^{XY}(x,y)} \mathrm{d}\gamma^{XY}(x,y),
    \end{align*}
    so that the desired condition holds. Similarly, Property 6 is satisfied. 

    Now we extend the proof to the setting where the strict positivity condition on the densities is dropped. In this case, the expressions \eqref{eqn:gluing_A} and \eqref{eqn:gluing_C} are not well-defined, and the construction becomes more involved, as we have to deal with singular parts of the measures. To this end, let $\gamma_{0 \perp 1}^{XY}$ denote the singular part of $\gamma_0^{XY}$ with respect to $\gamma_1^{XY}$; that is,
    \[
    \gamma_{0 \perp 1}^{XY} \coloneqq \mathbf{1}_{\{g^{XY}_1 = 0\}} \gamma_0^{XY},
    \]
    where we use $\mathbf{1}_\bullet$ to denote an indicator function. Similarly, define
    \[
    \gamma_{1 \perp 0}^{YZ} \coloneqq \mathbf{1}_{\{g^{YZ}_0 = 0\}} \gamma_1^{YZ}.
    \]
    For the rest of the proof, we let $\lambda, p^X,p^Y,p^Z, a ,b ,c$ and $Q$ be defined as above and use overlines to indicate the various pieces of the construction in this setting, e.g., $\overline{\lambda}$. 
    Define the measure space $\overline{Q}$ as a disjoint union of three copies of $Q$:
    \[
    \overline{Q} = Q_1 \sqcup Q_2 \sqcup Q_3 = (X \times Y \times Z) \sqcup (X \times Y \times Z) \sqcup (X \times Y \times Z).
    \] 
    The measure $\overline{\lambda}$ on $\overline{Q}$ is given by 
    \[
    \overline{\lambda}|_{Q_1} = \lambda, \quad   \overline{\lambda}|_{Q_2} = \gamma_{0 \perp 1}^{XY} \otimes \hat{\mu}_Z, \quad \overline{\lambda}|_{Q_3} = \hat{\mu}_X \otimes \gamma_{1 \perp 0}^{YZ}.
    \]
    where $\hat{\mu}_X = \frac{1}{\mu_X(X)} \mu_X$ and $\hat{\mu}_Z = \frac{1}{\mu_Z(Z)} \mu_Z$. The maps $\overline{p}^X:Q \to X$, $\overline{p}^Y:Q \to Y$, and $\overline{p}^Z:Q \to Z$ are defined to be coordinate projections on each block $Q_i$. 
    
    We now define the densities, beginning by defining $\overline{a}$ by
    \[
    \overline{a}|_{Q_1}(x,y,z) = \mathbf{1}_{\{g_1^{XY}(x,y) > 0\}} a(x,y), \quad \overline{a}|_{Q_2}(x,y,z) = 1, \quad \overline{a}|_{Q_3}(x,y,z) = 0.
    \]
    Let us immediately check that the desired Property 1 holds:
    \begin{align*}
        \mathrm{d}(\overline{p}^X,\overline{p}^Y)_\#(\overline{a}^2 \overline{\lambda})(x,y) &= \int_Z \mathbf{1}_{\{g_1^{XY}(x,y) > 0\}} a(x,y)^2 \mathrm{d}\lambda(x,y,z) \\
        & \quad + \int_Z 1 \cdot \mathrm{d}(\gamma_{0 \perp 1}^{XY} \otimes \hat{\mu}_Z)(x,y,z) + \int_Z 0 \cdot \mathrm{d}(\hat{\mu}_X \otimes \gamma_{1 \perp 0}^{YZ})(x,y,z) \\
        &= \mathbf{1}_{\{g_1^{XY}(x,y) > 0\}} a(x,y)^2 \mathrm{d}\gamma_1^{XY}(x,y) + \mathrm{d}\gamma_{0 \perp 1}^{XY}(x,y) \\
        &= \mathbf{1}_{\{g_1^{XY}(x,y) > 0\}} \mathrm{d}\gamma_0^{XY}(x,y) + \mathbf{1}_{\{g_1^{XY}(x,y) = 0\}} \mathrm{d}\gamma_0^{XY}(x,y) \\
        &= \mathrm{d}\gamma_0^{XY}(x,y).
    \end{align*}
    Next, define 
    \[
    \overline{b}|_{Q_1}(x,y,z) = 1, \quad \overline{b}|_{Q_2}(x,y,z) = 0, \quad \overline{b}|_{Q_3}(x,y,z) = 0. 
    \]
    It is then straightforward to see that Properties 2 and 4 hold. Finally, we define 
    \[
    \overline{c}|_{Q_1}(x,y,z) = \mathbf{1}_{\{g_0^{YZ}(y,z) > 0\}} c(y,z), \quad \overline{c}|_{Q_2}(x,y,z) = 0, \quad \overline{c}|_{Q_3}(x,y,z) = 1.
    \]
    By a calculation similar to the above, one can check that Property 5 is satisfied. It then only remains to show that Properties 3 and 6 hold under these definitions. For Property 3, we have 
    \begin{align*}
        \mathrm{d}(\overline{p}^X,\overline{p}^Y)_\#(\overline{a} \overline{b} \cdot \overline{\lambda})(x,y) &= \int_Z \mathbf{1}_{\{g_1^{XY}(x,y) > 0\}} a(x,y) \cdot 1 \cdot \mathrm{d}\lambda(x,y,z) \\
        & \quad + \int_Z 1 \cdot 0 \cdot \mathrm{d}(\gamma_{0 \perp 1}^{XY} \otimes \hat{\mu}_Z)(x,y,z) \\
        &\quad \quad + \int_Z 0 \cdot 0   \cdot \mathrm{d}(\hat{\mu}_X \otimes \gamma_{1 \perp 0}^{YZ})(x,y,z) \\ 
        &= \mathbf{1}_{\{g_1^{XY}(x,y) > 0\}} \sqrt{g_0^{XY}(x,y)g_1^{XY}(x,y)} \mathrm{d}\gamma^{XY}(x,y) \\
        &= \sqrt{g_0^{XY}(x,y)g_1^{XY}(x,y)} \mathrm{d}\gamma^{XY}(x,y),
    \end{align*}
    and Property 6 is verified via a similar calculation.
\end{proof}

The next lemma uses the Gluing Lemma to further prepare for the triangle inequality for CGW distance.

\begin{lemma}\label{lem:gluing_lemma_estimates}
Let $\mathcal{N}_X$, $\mathcal{N}_Y$, and $\mathcal{N}_Z$ be measure networks, and choose semicouplings as in Lemma \ref{lem:gluing_lemma}. Using the notation of Lemma \ref{lem:gluing_lemma}, define  $(\gamma_0^{XZ},\gamma_1^{XZ}) \in\Gamma(\mu_X,\mu_Z)$ by 
\[
\gamma_0^{XZ} = (p^X,p^Z)_\# (a^2 \lambda), \quad \gamma_1^{XZ} = (p^X,p^Z)_\# (c^2 \lambda).
\]
Then we have:
\begin{enumerate}
    \item $\mathbf{H}_\delta(\gamma_0^{XY},\gamma_1^{XY}) = \int_{Q^2} \mathsf{d}_{\mathcal{C}_\delta[\R]}\left(\left[\omega_X(p^X,p^X),a^2\right],\left[\omega_Y(p^Y,p^Y),b^2\right]\right)^2 \mathrm{d}\lambda \mathrm{d}\lambda$,
    \item $\mathbf{H}_\delta(\gamma_0^{YZ},\gamma_1^{YZ}) = \int_{Q^2} \mathsf{d}_{\mathcal{C}_\delta[\R]}\left(\left[\omega_Y(p^Y,p^Y),b^2\right],\left[\omega_Z(p^Z,p^Z),c^2\right]\right)^2 \mathrm{d}\lambda \mathrm{d}\lambda$, and 
    \item $\mathbf{H}_\delta(\gamma_0^{XZ},\gamma_1^{XZ}) \leq \int_{Q^2} \mathsf{d}_{\mathcal{C}_\delta[\R]}\left(\left[\omega_X(p^X,p^X),a^2\right],\left[\omega_Z(p^Z,p^Z),c^2\right]\right)^2 \mathrm{d}\lambda \mathrm{d}\lambda$,
\end{enumerate}
 where all integrands suppress function arguments for the sake of conciseness (e.g., $a^2$ is a stand-in for $a(q)a(q')$, with $q,q' \in Q$). 
\end{lemma}

\begin{proof}
    We first consider the integral term in the formula for $\mathbf{H}_\delta(\gamma_0^{XY},\gamma_1^{XY})$ from Lemma \ref{lem:CGW_loss_formula}. Using notation from Lemma \ref{lem:gluing_lemma}, and slightly rearranging terms, this is given by 
    \begin{align*}
        &\int_{(X \times Y)^2} \Omega\left(\frac{|\omega_X(x,x')-\omega_Y(y,y')}{2\delta} \right) \\
        &\qquad \qquad \cdot \sqrt{g_0^{XY}(x,y) g_1^{XY}(x,y)} \mathrm{d}\gamma^{XY}(x,y) \sqrt{g_0^{XY}(x',y') g_1^{XY}(x',y')} \mathrm{d}\gamma^{XY}(x',y').
    \end{align*}
    By the properties of the gluing construction, a change of variables shows that this is equal to 
\begin{align*}
        &\int_{Q^2} \Omega\left(\frac{|\omega_X(p^X(q),p^X(q'))-\omega_Y(p^Y(q), p^Y(q'))}{2\delta} \right)a(q)b(q) \mathrm{d}\lambda(q) a(q')b(q')\mathrm{d}\lambda(q').
    \end{align*}
    Running the derivation in the proof of Lemma \ref{lem:CGW_loss_formula} in reverse proves that the expression in Claim 1 is valid. The proof for Claim 2 is similar. 

    It remains to prove Claim 3. Let $\gamma^{XZ} = (p^X,p^Z)_\# \lambda$, which we consider as a reference measure for $\gamma_0^{XZ}$ and $\gamma_1^{XZ}$. Let $g_0^{XZ},g_1^{XZ}$ denote their respective densities (these exist due to the basic fact that pushforwards preserve absolute continuity). Following a similar strategy as the above, once again invoking Lemma \ref{lem:CGW_loss_formula}, the proof of Claim 3 follows if we can show that 
    \begin{equation*}
    (p^X,p^Z)_\# (ac \cdot \lambda) \leq \sqrt{g_0^{XZ}g_1^{XZ}} \cdot \lambda^{XZ}.
    \end{equation*}
    Note, in particular, that equality of these quantities is not guaranteed by the Gluing Lemma. Further, it is sufficient to show that the densities satisfy
    \begin{equation}\label{eqn:triangle_inequality_2}
    \frac{\mathrm{d}(p^X,p^Z)_\# (ac \cdot \lambda)}{\mathrm{d}\gamma^{XZ}}(x,z) \leq \sqrt{g_0^{XZ}(x,z)g_1^{XZ}(x,z)} 
    \end{equation}
    almost everywhere. To do so, we will apply the general measure-theoretic  result~\cite[Proposition A.13]{leonard2014some}: for a measurable map of Polish spaces $\phi:U \to V$, and any $\mu \in \mathbb{M}(U)$ and density $h:U \to [0,\infty)$, we have 
    \begin{align}\label{eqn:disintegration_measures}
    \frac{\mathrm{d}\phi_\#(h \cdot \mu)}{\mathrm{d}\phi_\# \mu}(v) = \int_{\phi^{-1}(v)} h(u) \mathrm{d}\mu_v(u),
    \end{align}
    where $\mu_v$ is the disintegration of $\mu$ with respect to $\phi$ at $v$. Applying this formula to the densities in question gives
    \begin{align*}
        &\frac{\mathrm{d}(p^X,p^Z)_\# (ac \cdot \lambda)}{\mathrm{d}\gamma^{XZ}}(x,z) \\
        &\qquad = \int_{(p^X,p^Z)^{-1}(x,z)} a(q) c(q) \mathrm{d}\lambda_{(x,z)}(q) \\
        &\qquad \leq \left(\int_{(p^X,p^Z)^{-1}(x,z)} a(q)^2 \mathrm{d}\lambda_{(x,z)}(q) \right)^{1/2} \left(\int_{(p^X,p^Z)^{-1}(x,z)} c(q)^2 \mathrm{d}\lambda_{(x,z)}(q) \right)^{1/2} \\
        &\qquad = \sqrt{g_0^{XZ}(x,z)g_1^{XZ}(x,z)},
    \end{align*}
    where the inequality is Cauchy-Schwarz. This proves that \eqref{eqn:triangle_inequality_2} holds, and Claim 3 follows.
\end{proof}

We now proceed with the proof of Theorem \ref{thm:psuedometric}.

\begin{proof}[Proof of Theorem \ref{thm:psuedometric}]
Let $\mathcal{N}_X$ and $\mathcal{N}_Y$ be measure networks. Non-negativity of $\CGW_\delta(\mathcal{N}_X,\mathcal{N}_Y)$ is trivial. 
Symmetry of $\CGW_\delta$ follows from the fact that, for any $(\gamma_0,\gamma_1) \in \Gamma(\mu_X,\mu_Y)$, the pair 
\[
(\gamma_0^\top, \gamma_1^\top) := ((x,y) \mapsto (y,x))_\# \gamma_0, \, ((x,y) \mapsto (y,x))_\# \gamma_1
\] 
belongs to $\Gamma(\mu_Y,\mu_X)$, and that $\mathbf{H}_\delta$ is symmetric.

 To prove the triangle inequality, let $\mathcal{N}_X$, $\mathcal{N}_Y$ and $\mathcal{N}_Z$ be measure networks, and choose semicouplings $(\gamma_0^{XY},\gamma_1^{XY}) \in \Gamma(\mu_X,\mu_Y)$ and $(\gamma_0^{YZ},\gamma_1^{YZ}) \in \Gamma(\mu_Y,\mu_Z)$. Let $(\gamma_0^{XZ},\gamma_1^{XZ}) \in \Gamma(\mu_X,\mu_Z)$, be the semicoupling defined as in Lemma \ref{lem:gluing_lemma_estimates}. Using notation as in the Lemmas \ref{lem:gluing_lemma} and \ref{lem:gluing_lemma_estimates} (including some suppressed integrand arguments), we have 
\begin{align}
    \mathbf{H}_\delta(\gamma_0^{XZ},\gamma_1^{XZ})^{1/2} &\leq \left(\int_{Q^2} \mathsf{d}_{\mathcal{C}_\delta[\R]}\left(\left[\omega_X(p^X,p^X),a^2\right],\left[\omega_Z(p^Z,p^Z),c^2\right]\right)^2 \mathrm{d}\lambda \mathrm{d}\lambda\right)^{1/2} \label{eqn:tri_ineq_1} \\
    &\leq \Bigg(\int_{Q^2} \Big(\mathsf{d}_{\mathcal{C}_\delta[\R]}\left(\left[\omega_X(p^X,p^X),a^2\right],\left[\omega_Y(p^Y,p^Y),b^2\right]\right)\label{eqn:tri_ineq_2}\\
        &\qquad \qquad + \mathsf{d}_{\mathcal{C}_\delta[\R]}\left(\left[\omega_Y(p^Y,p^Y),b^2\right],\left[\omega_Z(p^Z,p^Z),c^2\right]\right)\Big)^2 \mathrm{d}\lambda \mathrm{d}\lambda\Bigg)^{1/2} \nonumber \\
    &\leq \Bigg(\int_{Q^2} \mathsf{d}_{\mathcal{C}_\delta[\R]}\left(\left[\omega_X(p^X,p^X),a^2\right],\left[\omega_Y(p^Y,p^Y),b^2\right]\right)^2 \mathrm{d}\lambda \mathrm{d}\lambda\Bigg)^{1/2} \label{eqn:tri_ineq_3}\\
    &\qquad  + \Bigg(\int_{Q^2} \mathsf{d}_{\mathcal{C}_\delta[\R]}\left(\left[\omega_Y(p^Y,p^Y),b^2\right],\left[\omega_Z(p^Z,p^Z),c^2\right]\right)^2 \mathrm{d}\lambda \mathrm{d}\lambda\Bigg)^{1/2} \nonumber\\
    &= \mathbf{H}_\delta(\gamma_0^{XY},\gamma_1^{XY})^{1/2} + \mathbf{H}_\delta(\gamma_0^{YZ},\gamma_1^{YZ})^{1/2} \label{eqn:tri_ineq_4},
\end{align}
where \eqref{eqn:tri_ineq_1} is Claim 3 of Lemma \ref{lem:gluing_lemma_estimates}, \eqref{eqn:tri_ineq_2} is triangle inequality for $d_{\mathcal{C}_\delta[\R]}$, \eqref{eqn:tri_ineq_3} is Minkowski's inequality, and \eqref{eqn:tri_ineq_4} is another application of Lemma \ref{lem:gluing_lemma_estimates}. Since the initial semicouplings were arbitrary, this inequality passes through to the infima in the definition of CGW distance, and triangle inequality thus follows.

It only remains to establish the vanishing condition for $\CGW_\delta$. Suppose that $\mathcal{N}_X$ and $\mathcal{N}_Y$ are weakly isomorphic and let $(Z,\mu_Z)$ be a metric space with $f:Z\to X$ and $g:Z\to Y$ satisfying the conditions of Definition \ref{def:wisomorphism}. Let $(f,g):Z \to X \times Y$ denote the map $z \mapsto (f(z),g(z))$ and set $\gamma\coloneqq(f,g)_\#\mu_Z$. Then $(\gamma,\gamma) \in \Gamma(\mu_X,\mu_Y)$, whence we obtain (using notational Convention \ref{conv:simplified_integrand})
\begin{align*}
\CGW^2_\delta(\mathcal{N}_X,\mathcal{N}_Y)&\leq\int_{(X\times Y)^2}\mathsf{d}_{\Co_\delta[\R]}([\omega_X,1],[\omega_Y,1])^2\textnormal{d}\gamma\textnormal{d}\gamma\\
&=\int_{Z^2}\mathsf{d}_{\Co_\delta[\R]}([f^\#\omega_X,1],[g^\#\omega_Y,1])^2\textnormal{d}\mu_Z\textnormal{d}\mu_Z = 0.
\end{align*}
Next, assume $\CGW_\delta(\mathcal{N}_X,\mathcal{N}_Y)=0$. Let $(\gamma_0,\gamma_1)$ be an optimal semi-coupling (guaranteed to exist, by Lemma \ref{lemma:lsc_functional}), so that, by Lemma \ref{lem:CGW_loss_formula} and the calculations in its proof, we have
    \begin{align*}
        0&=4 \delta^2 \int_{X\times Y}\int_{X\times Y} \frac{\textnormal{d}\gamma_0}{\textnormal{d}\gamma}(x,y)\frac{\textnormal{d}\gamma_0}{\textnormal{d}\gamma}(x',y') +\frac{\textnormal{d}\gamma_1}{\textnormal{d}\gamma}(x,y)\frac{\textnormal{d}\gamma_1}{\textnormal{d}\gamma}(x',y')\\
        &\hspace{1in} -2\Omega\left(\frac{|\omega_X(x,x')-\omega_Y(y,y')|}{2\delta}\right) \\
        &\hspace{1.25in} \cdot \sqrt{\frac{\textnormal{d}\gamma_0}{\textnormal{d}\gamma}(x,y)\frac{\textnormal{d}\gamma_1}{\textnormal{d}\gamma}(x,y)\frac{\textnormal{d}\gamma_0}{\textnormal{d}\gamma}(x',y')\frac{\textnormal{d}\gamma_1}{\textnormal{d}\gamma}(x',y')} \, \textnormal{d}\gamma(x,y)  \textnormal{d}\gamma(x',y').
    \end{align*}
    
    Since the integrand is always non-negative, this implies that, $\gamma\otimes\gamma$-almost everywhere,
    \begin{multline*}
        \frac{\textnormal{d}\gamma_0}{\textnormal{d}\gamma}(x,y)\frac{\textnormal{d}\gamma_0}{\textnormal{d}\gamma}(x',y') +\frac{\textnormal{d}\gamma_1}{\textnormal{d}\gamma}(x,y)\frac{\textnormal{d}\gamma_1}{\textnormal{d}\gamma}(x',y')=\\ 2\Omega\left(\frac{|\omega_X(x,x')-\omega_Y(y,y')|}{2\delta}\right)\sqrt{\frac{\textnormal{d}\gamma_0}{\textnormal{d}\gamma}\frac{\textnormal{d}\gamma_1}{\textnormal{d}\gamma}}(x,y)\sqrt{\frac{\textnormal{d}\gamma_0}{\textnormal{d}\gamma}\frac{\textnormal{d}\gamma_1}{\textnormal{d}\gamma}}(x',y').
    \end{multline*}
    Since $\Omega\leq1$, then $\gamma\otimes\gamma$-almost everywhere,
    \begin{equation*}
        \frac{\textnormal{d}\gamma_0}{\textnormal{d}\gamma}(x,y)\frac{\textnormal{d}\gamma_0}{\textnormal{d}\gamma}(x',y') +\frac{\textnormal{d}\gamma_1}{\textnormal{d}\gamma}(x,y)\frac{\textnormal{d}\gamma_1}{\textnormal{d}\gamma}(x',y')\leq2\sqrt{\frac{\textnormal{d}\gamma_0}{\textnormal{d}\gamma}\frac{\textnormal{d}\gamma_1}{\textnormal{d}\gamma}}(x,y)\sqrt{\frac{\textnormal{d}\gamma_0}{\textnormal{d}\gamma}\frac{\textnormal{d}\gamma_1}{\textnormal{d}\gamma}}(x',y').
    \end{equation*}    
    
    By the AM-GM inequality, $\gamma\otimes\gamma$-almost everywhere,
    \begin{equation}\label{eqn:AM_GM_conclusion}
        \frac{\textnormal{d}\gamma_0}{\textnormal{d}\gamma}(x,y)\frac{\textnormal{d}\gamma_0}{\textnormal{d}\gamma}(x',y') = \frac{\textnormal{d}\gamma_1}{\textnormal{d}\gamma}(x,y)\frac{\textnormal{d}\gamma_1}{\textnormal{d}\gamma}(x',y').
    \end{equation}
   So, $\gamma$-almost everywhere, we have 
    \begin{equation*}
         \frac{\textnormal{d}\gamma_0}{\textnormal{d}\gamma}(x,y)= \frac{\textnormal{d}\gamma_1}{\textnormal{d}\gamma}(x,y).
    \end{equation*}
    Indeed, integrating both sides of \eqref{eqn:AM_GM_conclusion} over $(x',y')$ yields
    \begin{align*}
    \frac{\textnormal{d}\gamma_0}{\textnormal{d}\gamma}(x,y)\mu_X(X) &= \int_{X \times Y} \frac{\textnormal{d}\gamma_0}{\textnormal{d}\gamma}(x,y)\frac{\textnormal{d}\gamma_0}{\textnormal{d}\gamma}(x',y') \mathrm{d}\gamma(x',y') \\
    &= \int_{X \times Y} \frac{\textnormal{d}\gamma_1}{\textnormal{d}\gamma}(x,y)\frac{\textnormal{d}\gamma_1}{\textnormal{d}\gamma}(x',y') \mathrm{d}\gamma(x',y') = \frac{\textnormal{d}\gamma_1}{\textnormal{d}\gamma}(x,y)\mu_Y(Y).
    \end{align*}
    As $\mathsf{CGW}_\delta(\mathcal{N}_X,\mathcal{N}_Y)=0$ implies $\mu_X(X) = \mu_Y(Y)$, the claim follows.
    Therefore, $\gamma_0=\gamma_1$. Let $Z=X\times Y$, $\mu_Z=\gamma_0=\gamma_1$ and $f:X \times Y \to X$ and $g:X \times Y \to Y$ the obvious projection maps. Thus, $f_\#\mu_Z=\mu_X$, $g_\#\mu_Z=\mu_Y$. Moreover, Lemma \ref{lem:CGW_loss_formula} gives
    \begin{align*}
        0 &= \CGW_\delta(\mathcal{N}_X,\mathcal{N}_Y)  = 4\delta^2 \left(2\mu_Z(Z)^2 - 2 \int_{Z^2} \Omega\left(\frac{|f^\#\omega_X-g^\#\omega_Y|}{2\delta}\right) \, \mathrm{d}\mu_Z \mathrm{d}\mu_Z \right).
    \end{align*}
    Since $\Omega \leq 1$, this is only possible if the integrand is equal to one $\mu_Z \otimes \mu_Z$-almost everywhere. This implies that the argument of $\Omega$ is equal to zero $\mu_Z \otimes \mu_Z$-almost everywhere; that is, $\|f^\#\omega_X-g^\#\omega_Y\|_{\mathrm{L}^\infty(\mu_Z \otimes \mu_Z)} = 0$. It follows that $\mathcal{N}_X$ and $\mathcal{N}_Y$ are weakly isomorphic. 
\end{proof}

\subsection{Equivalence to the Formulation of S\'ejourn\'e, Vialard and Peyr\'e}

A short argument shows that the optimization problems defining $\CGW_{1/2}$ (as in Definition \ref{CGW_ours}) and the original formulation of conic Gromov-Wasserstein distance $\CGW_{\mathrm{SVP}}$ (as in Definition \ref{def:CGW_Sejourne})  are, in fact, equivalent. This mirrors an equivalence of cone and semicoupling formulations of the Wasserstein-Fisher-Rao distance~\cite[Remark 3.10]{friesecke2021barycenters}.

\begin{proposition}[Equivalence to the Formulation of \cite{NEURIPS2021}]\label{prop:CGW_semi_coupling}
For mm-spaces $\mathcal{M}_X$ and $\mathcal{M}_Y$,
\begin{equation}
    \CGW_{1/2}(\mathcal{M}_X, \mathcal{M}_Y) = \CGW_{\mathrm{SVP}}(\mathcal{M}_X, \mathcal{M}_Y).
\end{equation}
\end{proposition}
\begin{proof}
Let $(\gamma_0, \gamma_1) \in \Gamma(\mu_X, \mu_Y)$ and let $\gamma\in\Mm(X\times Y)$ be such that $\gamma_0, \gamma_1 \ll \gamma$. 
Consider the map 
\begin{align*}
    T:X\times Y &\to \Co[X]\times \Co[Y] \\
    (x,y)&\mapsto\left(\left[x,\sqrt{\frac{\textnormal{d}\gamma_0}{\textnormal{d}\gamma}(x,y)}\right], \left[y, \sqrt{\frac{\textnormal{d}\gamma_1}{\textnormal{d}\gamma}(x,y)}\right]\right),    
\end{align*}
then define $\alpha := T_\# \gamma \in \Mm(\Co[X]\times\Co[Y])$. Note that
\[
\int_{\R^+}r^2 \textnormal{d}(\operatorname{Pr}_0)_\#\alpha(\cdot,r)= \int_{X\times Y} \frac{\textnormal{d}\gamma_0}{\textnormal{d}\gamma}(\cdot,y) \textnormal{d}\gamma(\cdot,y) =\mu_X
\]
and similarly,
\[
\int_{\R^+}s^2 \textnormal{d}(\operatorname{Pr}_1)_\#\alpha(\cdot,s)=\mu_Y.
\]
So, $\alpha\in \mathrm{A}(\mu_X,\mu_Y).$
By definition,
\begin{align*}
 \mathbf{H}_{1/2}(\gamma_0, \gamma_1)
&= \int_{(X \times Y)^2} \bigg[ 
    \frac{\mathrm{d}\gamma_0}{\mathrm{d}\gamma}(x, y) \cdot \frac{\mathrm{d}\gamma_0}{\mathrm{d}\gamma}(x', y') 
    + \frac{\mathrm{d}\gamma_1}{\mathrm{d}\gamma}(x, y) \cdot \frac{\mathrm{d}\gamma_1}{\mathrm{d}\gamma}(x', y') \\
&\hspace{1in}  - 2 \, \Omega\big( |\mathsf{d}_X(x, x') - \mathsf{d}_Y(y, y')| \big) \\
&\hspace{1.15in} \cdot \left. \sqrt{ 
    \frac{\mathrm{d}\gamma_0}{\mathrm{d}\gamma}(x, y) \cdot \frac{\mathrm{d}\gamma_1}{\mathrm{d}\gamma}(x, y) } 
    \cdot \sqrt{ 
    \frac{\mathrm{d}\gamma_0}{\mathrm{d}\gamma}(x', y') \cdot \frac{\mathrm{d}\gamma_1}{\mathrm{d}\gamma}(x', y') 
} \right] \\
&\hspace{3in} \mathrm{d}\gamma(x, y) \, \mathrm{d}\gamma(x', y') \\
&= \int_{(\Co(X) \times \Co(Y))^2} \bigg[ (rr')^2 + (ss')^2  - 2 \, rr'ss' \, \Omega\big( |\mathsf{d}_X(x, x') - \mathsf{d}_Y(y, y')| \big) \bigg] \\
&\hspace{2.15in} \mathrm{d}\alpha([x, r], [y, s]) \, \mathrm{d}\alpha([x', r'], [y', s']) \\
& =\mathbf{K}(\alpha),
\end{align*}
where the second line follows by the change of variables, hence
\begin{equation}
    \CGW_{1/2}(\mathcal{M}_X, \mathcal{M}_Y) \geq \CGW_{\mathrm{SVP}}(\mathcal{M}_X, \mathcal{M}_Y).
\end{equation}

To prove the other inequality, let $\alpha\in\mathrm{A}(\mu_X,\mu_Y)$. 
Define 
\begin{equation*}
\gamma_0 = \int_{\R^+\times\R^+} r^2 \textrm{d}\alpha([\cdot, r],[\cdot, s])\in \Mm(X\times Y) 
\end{equation*}
and
\begin{equation*} 
\gamma_1 = \int_{\R^+\times\R^+} s^2 \textrm{d}\alpha([\cdot, r],[\cdot, s])\in \Mm(X\times Y).
\end{equation*}
Note that 
\begin{equation*}    (\operatorname{Pr}_0)_\#\gamma_0=\int_Y\textnormal{d}\gamma_0(\cdot,y)=\int_{\R^+\times Y\times \R^+} r^2 \textnormal{d}\alpha([\cdot,r],[y,s]) = \int_{\R^+} r^2\textnormal{d}(\operatorname{Pr}_0)_\#\alpha(\cdot,r)=\mu_X
\end{equation*}
and 
\begin{equation*}    (\operatorname{Pr}_1)_\#\gamma_1=\int_X\textnormal{d}\gamma_1(x,\cdot)=\int_{X\times \R^+\times \R^+} s^2 \textnormal{d}\alpha([x,r],[\cdot,s]) = \int_{\R^+}s^2\textnormal{d}(\operatorname{Pr}_1)_\#\alpha(\cdot,s)=\mu_Y.
\end{equation*}
Therefore, $(\gamma_0,\gamma_1)\in\Gamma(\mu_X,\mu_Y)$.
A similar argument to the one above shows that $\mathbf{H}_{1/2}(\gamma_0,\gamma_1)= \mathbf{K}(\alpha)$, so that
\begin{equation}
    \CGW_{1/2}(\mathcal{M}_X, \mathcal{M}_Y) \leq \CGW_{\mathrm{SVP}}(\mathcal{M}_X, \mathcal{M}_Y).
\end{equation}
\end{proof}

\section{Properties of the conic Gromov-Wasserstein Distance}\label{section:properties_CGW}

This section establishes various theoretical properties of the conic Gromov-Wasserstein distance.

\subsection{Scaling Properties}\label{subsec:scaling}

When generalizing transport metrics to the unbalanced case, it is crucial to understand how the metric behaves with respect to scaling in the total mass of the networks. To do so, we leverage the following result from \cite{Lashos_2019}.

\begin{lemma}[{\cite[Lemma 2.1]{Lashos_2019}}]\label{lemma:scaling}
Let $(X,d)$ be a metric space and consider the cone over $X$, $\Co[X]$ equipped with a distance $\mathsf{d}_{\Co_\delta[X]}$ from Definition \ref{def:cone_dist}. Then

\begin{equation*}
    \mathsf{d}_{\Co_\delta[X]}^2([x,tt'],[y,ss'])=t's'\mathsf{d}_{\Co_\delta[X]}^2([x,t],[y,s]) + 4\delta^2\left[(t'^2-t's')t^2+(s'^2-t's')s^2\right].
\end{equation*}
\end{lemma}
Directly following from this, we have
\begin{equation}\label{eq:scaling2}
\begin{aligned}
\mathsf{d}_{\Co_\delta[X]}^2([x,tr],[y,sr])&=r^2\mathsf{d}_{\Co_\delta[X]}^2([x,t],[y,s]) + 4\delta^2\bigg[ (r^2-rr)s^2+(r^2-rr)t^2 \bigg] \\
&= r^2\mathsf{d}_{\Co_\delta[X]}^2([x,t],[y,s])\\
\end{aligned}
\end{equation}
and
\begin{equation}\label{eq:scaling1}
\begin{aligned}
    \mathsf{d}_{\Co_\delta[X]}^2([x,t'],[x,s'])&=t's'\mathsf{d}_{\Co_\delta[X]}^2([x,1],[x,1]) 
    + 4\delta^2\bigg[ (t'^2-t's')+(s'^2-t's')\bigg]\\
    &= 4\delta^2(t'-s')^2.\\
\end{aligned}
\end{equation}

We now show how the CGW distance behaves under measure scaling.

\begin{proposition}\label{prop:CGWscale} Given a measure network $\mathcal{N}_X=(X,\mu_X,\omega_X)$. Let $\mathcal{N}_X^s=(X,s\cdot\mu_X,\omega_X)$, and $\mathcal{N}_X^r=(X,r\cdot\mu_X,\omega_X)$ where $r,s\in\R_{\geq0}$. Then we have
\begin{equation}
\CGW_\delta(\mathcal{N}_X^s,\mathcal{N}_X^r)\leq2\delta|r-s|\mu_X(X).
\end{equation}
\end{proposition}
\begin{proof}
Consider the diagonal map $\bigtriangleup:X\to X\times X$. Define $\gamma_0 = r \cdot\bigtriangleup_{\#}\mu_X$, $\gamma_1 = s \cdot\bigtriangleup_{\#}\mu_X$, and $\gamma = \bigtriangleup_{\#}\mu_X$. 
Thus, $(\gamma_0,\gamma_1)\in\Gamma(r\cdot\mu_X,s\cdot\mu_X)$ and $\gamma_0,\gamma_1\ll\gamma$ with $\frac{\textnormal{d}\gamma_0}{\textnormal{d}\gamma}=r$ and $\frac{\textnormal{d}\gamma_1}{\textnormal{d}\gamma}=s$. 
We then have (using notational Convention \ref{conv:simplified_integrand})
\begin{align*}
     \CGW_\delta(\mathcal{N}_X^s,\mathcal{N}_X^r)^2&\leq\int_{(X\times X)^2}\mathsf{d}_{\Co_\delta[\R]}\left(\omega_X,\frac{\textnormal{d}\gamma_0}{\textnormal{d}\gamma},\omega_Y,\frac{\textnormal{d}\gamma_1}{\textnormal{d}\gamma}\right)^2\textnormal{d}\gamma\textnormal{d}\gamma\\
     &=\int_{(X\times X)^2}\mathsf{d}_{\Co_\delta[\R]}\left(\omega_X,r,\omega_X,s\right)^2\textnormal{d}\gamma\textnormal{d}\gamma\\
     &=\int_{(X\times X)}\mathsf{d}_{\Co_\delta[\R]}\left(\left[\omega_X(x,x'),r\right],\left[\omega_X(x,x'),s\right]\right)^2\textnormal{d}\mu(x)\textnormal{d}\mu(x')\\
     &= 4\delta^2(r-s)^2\mu_X(X)^2. 
\end{align*}
Note that the final equality follows directly from ~\eqref{eq:scaling1}.
\end{proof}

\begin{proposition}\label{prop:CGWscale2}
Given measure networks $\mathcal{N}_X$ and $\mathcal{N}_Y$, and $r \geq 0$, we have
\begin{equation}\label{eqn:scaling_X_and_Y}
\CGW_\delta(\mathcal{N}_X^r,\mathcal{N}_Y^r)=  r\cdot\CGW_\delta(\mathcal{N}_X,\mathcal{N}_Y).
\end{equation}
\end{proposition}
\begin{proof}
Let $(\gamma_0,\gamma_1)\in\Gamma(\mu_X,\mu_Y)$ be a semi-coupling attaining $\CGW_\delta(\mathcal{N}_X,\mathcal{N}_Y)$, and let $\gamma_0,\gamma_1 \ll \gamma$. Define $\pi_0 = r\cdot\gamma_0$ and $\pi_1 = r\cdot\gamma_1$. Thus, $(\pi_0,\pi_1)\in \Gamma(r\cdot\mu_X,r\cdot\mu_Y)$ and $\pi_0,\pi_1 \ll\gamma$ with $\frac{\textnormal{d}\pi_0}{\textnormal{d}\gamma}=r\cdot\frac{\textnormal{d}\gamma_0}{\textnormal{d}\gamma}$ and $\frac{\textnormal{d}\pi_1}{\textnormal{d}\gamma}=r\cdot\frac{\textnormal{d}\gamma_1}{\textnormal{d}\gamma}$.

Therefore,
\begin{align}\label{eqn:forward_inequality}
\begin{aligned}
\CGW_\delta(\mathcal{N}_X^r,\mathcal{N}_Y^r)^2&\leq\int_{(X\times Y)^2}\mathsf{d}_{\Co_\delta[\R]}\left(\omega_X,\frac{\textnormal{d}\pi_0}{\textnormal{d}\gamma},\omega_Y,\frac{\textnormal{d}\pi_1}{\textnormal{d}\gamma}\right)^2\textnormal{d}\gamma\textnormal{d}\gamma\\
&=\int_{(X\times Y)^2}\mathsf{d}_{\Co_\delta[\R]}\left(\omega_X,r\frac{\textnormal{d}\gamma_0}{\textnormal{d}\gamma},\omega_Y,r\frac{\textnormal{d}\gamma_1}{\textnormal{d}\gamma}\right)^2\textnormal{d}\gamma\textnormal{d}\gamma\\
&= r^2\CGW_\delta(\mathcal{N}_{X},\mathcal{N}_{Y})^2 \\
\end{aligned}
\end{align}

Note that if $r=0$ then the equality \eqref{eqn:scaling_X_and_Y} holds trivially. For $r > 0$, the same argument as above, starting from $\mathcal{N}_X^r$ and $\mathcal{N}_Y^r$, and rescaling by $\frac{1}{r}$, yields the opposite inequality. These inequalities together imply the result. Moreover, this argument shows that the re-scaled semicouplings are optimal for the rescaled problem.
\end{proof}

% similarly we have
% \begin{align}\label{eqn:reverse_inequality}
% \begin{aligned}
% \CGW_\delta(\mathcal{N}_X,\mathcal{N}_Y)^2
% &\leq
% \int_{(X\times Y)^2}
% \mathsf{d}_{\Co_\delta[\mathbb{R}]}
% \left(
% \omega_X,\frac{\mathrm{d}\gamma_0}{\mathrm{d}\pi},
% \omega_Y,\frac{\mathrm{d}\gamma_1}{\mathrm{d}\pi}
% \right)^2
% \,\mathrm{d}\pi\,\mathrm{d}\pi \\
% &=
% \int_{(X\times Y)^2}
% \mathsf{d}_{\Co_\delta[\mathbb{R}]}
% \left(
% \omega_X,\frac{1}{r}\frac{\mathrm{d}\pi_0}{\mathrm{d}\pi},
% \omega_Y,\frac{1}{r}\frac{\mathrm{d}\pi_1}{\mathrm{d}\pi}
% \right)^2
% \,\mathrm{d}\pi\,\mathrm{d}\pi \\
% &=
% \frac{1}{r^2}
% \int_{(X\times Y)^2}
% \mathsf{d}_{\Co_\delta[\mathbb{R}]}
% \left(
% \omega_X,\frac{\mathrm{d}\pi_0}{\mathrm{d}\pi},
% \omega_Y,\frac{\mathrm{d}\pi_1}{\mathrm{d}\pi}
% \right)^2
% \,\mathrm{d}\pi\,\mathrm{d}\pi \\
% &=
% \frac{1}{r^2}\,\CGW_\delta(\mathcal{N}_X^r,\mathcal{N}_Y^r)^2.
% \end{aligned}
% \end{align}
% }
% Combining \eqref{eqn:forward_inequality} and \eqref{eqn:reverse_inequality} proves the equality in \eqref{eqn:scaling_X_and_Y}. \textcolor{red}{Consequently, we note that the rescaled semi-couplings are in fact optimal for the corresponding rescaled problem.}

Putting the results of this subsection together and applying the triangle inequality for CGW distance, the following corollary is immediate.

\begin{corollary}\label{corr:scaling_cgw}
    Let $\mathcal{N}_X,\mathcal{N}_Y$ be measure networks and let $r,s \geq 0$. Then 
    \[
    \CGW_\delta(\mathcal{N}_X^r, \mathcal{N}_Y^s) \leq \CGW_\delta(\mathcal{N}_X,\mathcal{N}_Y) + 2\delta \bigg( |1-r|\mu_X(X) + |1-s|\mu_Y(Y)\bigg).
    \]
\end{corollary}

Furthermore, we derive an additional scaling property of CGW distance, which is analogous to that of Hellinger-Kantorovich distance~\cite[Theorem 3.3]{Lashos_2019}.
\begin{proposition}\label{prop:CGWscale3}
Given measure networks $\mathcal{N}_X$ and $\mathcal{N}_Y$, and $r,s \geq 0$, we have
\begin{equation}\label{eqn:scaling_rX_and_sY}
\CGW_\delta(\mathcal{N}_X^r,\mathcal{N}_Y^s)^2=  rs\cdot\CGW_\delta(\mathcal{N}_X,\mathcal{N}_Y)^2 + 4\delta^2((r^2-rs)\mu_X(X)^2 + (s^2-rs)\mu_Y(Y)^2).
\end{equation}
\end{proposition}
\begin{proof}
Let $(\gamma_0,\gamma_1)\in\Gamma(\mu_X,\mu_Y)$ be a semi-coupling attaining $\CGW_\delta(\mathcal{N}_X,\mathcal{N}_Y)$, and let $\gamma_0,\gamma_1 \ll \gamma$. Define $\pi_0 = r\cdot\gamma_0$ and $\pi_1 = s\cdot\gamma_1$. Thus, $(\pi_0,\pi_1)\in \Gamma(r\cdot\mu_X,s\cdot\mu_Y)$ and $\pi_0,\pi_1 \ll\gamma$ with $\frac{\textnormal{d}\pi_0}{\textnormal{d}\gamma}=r\cdot\frac{\textnormal{d}\gamma_0}{\textnormal{d}\gamma}$ and $\frac{\textnormal{d}\pi_1}{\textnormal{d}\gamma}=s\cdot\frac{\textnormal{d}\gamma_1}{\textnormal{d}\gamma}$. 
Therefore,
\begin{align}\label{eqn:forward_inequality2}
\begin{aligned}
\CGW_\delta(\mathcal{N}_X^r,\mathcal{N}_Y^s)^2&\leq\int_{(X\times Y)^2}\mathsf{d}_{\Co_\delta[\R]}\left(\omega_X,\frac{\textnormal{d}\pi_0}{\textnormal{d}\gamma},\omega_Y,\frac{\textnormal{d}\pi_1}{\textnormal{d}\gamma}\right)^2\textnormal{d}\gamma\textnormal{d}\gamma\\
&=\int_{(X\times Y)^2}\mathsf{d}_{\Co_\delta[\R]}\left(\omega_X,r\frac{\textnormal{d}\gamma_0}{\textnormal{d}\gamma},\omega_Y,s\frac{\textnormal{d}\gamma_1}{\textnormal{d}\gamma}\right)^2\textnormal{d}\gamma\textnormal{d}\gamma\\
&=\int_{(X\times Y)^2}rs \cdot  \mathsf{d}_{\Co_\delta[\R]}\left(\omega_X,\frac{\textnormal{d}\gamma_0}{\textnormal{d}\gamma},\omega_Y,\frac{\textnormal{d}\gamma_1}{\textnormal{d}\gamma}\right)^2\\
&\qquad+4\delta^2\left((r^2-rs)\frac{\textnormal{d}\gamma_0}{\textnormal{d}\gamma}(x,y)\frac{\textnormal{d}\gamma_0}{\textnormal{d}\gamma}(x',y')\right.\\
&\qquad\qquad\left.+(s^2-rs)\frac{\textnormal{d}\gamma_1}{\textnormal{d}\gamma}(x,y)\frac{\textnormal{d}\gamma_1}{\textnormal{d}\gamma}(x',y')\right) \textnormal{d}\gamma\textnormal{d}\gamma\\
&= rs\CGW_\delta(\mathcal{N}_{X},\mathcal{N}_{Y})^2+4\delta^2\left((r^2-rs)\mu_X(X)^2+(s^2-rs)\mu_Y(Y)^2\right)  \\
\end{aligned}
\end{align}
where the second equality follows from \cite[Lemma 2.1]{Lashos_2019}. Note that if $r=0$ or $s=0$ the equality \eqref{eqn:scaling_rX_and_sY} holds trivially. For $r,s > 0$, we have
\begin{align*}
\begin{aligned}
&\CGW_\delta(\mathcal{N}_X,\mathcal{N}_Y)^2\\
&\quad\leq \frac{1}{rs}\CGW_\delta(\mathcal{N}_X^r,\mathcal{N}_Y^s)^2+4\delta^2\left(\left(\frac{s-r}{sr^2}\right) (r\cdot\mu_X(X))^2+\left(\frac{s-r}{rs^2}\right) (s\cdot\mu_Y(Y))^2\right)\\
&\quad=\frac{1}{rs}\left(\CGW_\delta(\mathcal{N}_X^r,\mathcal{N}_Y^s)^2+4\delta^2\left(\left(rs-r^2\right)\mu_X(X)^2 +\left(rs-s^2\right) \mu_Y(Y)^2\right)\right).
\end{aligned}
\end{align*}
Thus, 
\begin{equation}\label{eqn:reverse_inequality2}    
  rs\cdot\CGW_\delta(\mathcal{N}_X,\mathcal{N}_Y)^2 + 4\delta^2\left((r^2-rs)\mu_X(X)^2 + (s^2-rs)\mu_Y(Y)^2\right)\leq \CGW_\delta(\mathcal{N}_X^r,\mathcal{N}_Y^s)^2.
\end{equation}
Combining \eqref{eqn:forward_inequality2} and \eqref{eqn:reverse_inequality2} proves the equality in \eqref{eqn:scaling_rX_and_sY}. 
\end{proof}

\subsection{Variational Convergence}\label{sec:Variational_Convergence}
We now recall the notion of $\Gamma$-convergence for functions on a metric space $(X,d)$.
 
\begin{mydef}[$\Gamma$-Convergence]\label{def:gamma_convergence}
Let $(X,d)$ be a metric space and let $\{\mathcal{E}_n\}_{n=1}^\infty$ be a sequence of functionals $\mathcal{E}_n:X \to \mathbb{R}\cup\{\pm \infty \}$. We say that $\mathcal{E}_n$ \define{$\Gamma$-converges} to a functional $\mathcal{E}_\infty$ as $n \to \infty$, denoted $\mathcal{E}_n \stackrel{\Gamma}{\to} \mathcal{E}_\infty$, if for all $u \in X$ we have:
\begin{enumerate}
 \item{(\define{Lim-Inf Inequality})} for every sequence $\{u_n\}_{n=1}^\infty$ in $X$ such that $u_n \to u$, it holds that $\mathcal{E}_\infty(u) \leq \liminf\limits_{n \to\infty} \mathcal{E}_n(u_n)$;
\item{(\define{Lim-Sup Inequality})} there exists a sequence $\{u_n\}_{n=1}^\infty$ in $X$ with $u_n \to u$ such that $\mathcal{E}_\infty(u) \geq \limsup\limits_{n \to \infty}\mathcal{E}_n(u_n)$.
\end{enumerate}
\end{mydef}

This concept was introduced by De Giorgi and Franzoni in 1970’s to study limits of variational problems \cite{De_Giorgi}; we refer to \cite{Braides_2002,Dal_Maso} for comprehensive introductions to $\Gamma$-convergence. Rather than requiring pointwise or uniform convergence, $\Gamma$-convergence ensures that the limiting behavior of minimizers is preserved: any sequence of approximate minimizers of the original functionals will converge (under suitable compactness conditions) to a minimizer of the limiting functional. 

The main theorem for this subsection treats variational convergence of the CGW cost functional (see Figure \ref{fig:matching}). Its statement requires some additional definitions and notation. For measure networks $\mathcal{N}_X$ and $\mathcal{N}_Y$, consider the augmented functional
    \begin{equation*}
        \begin{split}
            \widehat{\mathbf{H}}_\delta: \Gamma(\mu_X,\mu_Y) &\to \R \\
            (\gamma_0,\gamma_1) &\mapsto \mathbf{H}_\delta(\gamma_0,\gamma_1)- 4\delta^2 \left(\mu_X(X) - \mu_Y(Y)\right)^2,
        \end{split}
    \end{equation*}
    where $\mathbf{H}_\delta$ is as in \eqref{eqn:H_delta}. We note that, applying Lemma \ref{lem:CGW_loss_formula}, this functional can be conveniently rewritten as 
    \begin{equation}\label{eqn:augmented_CGW}
    \begin{split}
    \widehat{\mathbf{H}}_\delta(\gamma_0,\gamma_1) &= 8\delta^2\Bigg(\mu_X(X)\mu_Y(Y)  \\
    &\qquad \qquad \left. - \int_{(X\times Y)^2} \Omega\left(\frac{|\omega_X - \omega_Y|}{2\delta} \right) \sqrt{\frac{\mathrm{d}\gamma_0}{\mathrm{d}\gamma}\frac{\mathrm{d}\gamma_1}{\mathrm{d}\gamma}\frac{\mathrm{d}\gamma_0}{\mathrm{d}\gamma}\frac{\mathrm{d}\gamma_1}{\mathrm{d}\gamma}} \mathrm{d}\gamma \mathrm{d}\gamma\right),
    \end{split}
    \end{equation}
    where the function arguments in the integrand are suppressed to save space, but should agree with those in Lemma \ref{lem:CGW_loss_formula}.

We assume throughout the rest of this subsection that the function 
$\Omega: [0,\infty) \to \mathbb{R}$ satisfies a global lower bound and a local expansion at the origin. 
Specifically, there exists a constant $C \ge 0$ such that
\begin{equation}\label{eqn:Omega_global_lower}
    \Omega(z) \ge 1 - C z^2 \qquad \text{for all } z \ge 0,
\end{equation}
and there exist constants $C' \ge 0$ and $r > 0$ such that
\begin{equation}\label{eqn:Omega_local_bounds}
    \Omega(z) \le 1 - C z^2 + C' z^4 \qquad \text{for all } z \in [0,r].
\end{equation}
The constant $C$ is uniquely determined by the choice $\Omega$, and its value is crucial for the definition of the limiting functional. In applications, the argument  $z = \frac{|\omega_X(x,x') - \omega_Y(y,y')|}{2\delta} \leq \frac{M}{2\delta}\leq r$
is uniformly bounded for some $M > 0$ such that $\delta \ge \frac{M}{2r}$. 
Hence, for sufficiently large $\delta$, one has $z \in [0,r]$, so that the local bound \eqref{eqn:Omega_local_bounds} applies.
These bounds are satisfied, for instance, by the following choices of \(\Omega(z)\):

\begin{equation}\label{eq:ineq_cos}
    1 - \frac{z^2}{2} \leq \Omega(z) = \overline{\cos}(z) \leq 1 - \frac{z^2}{2} + \frac{z^4}{24},
\end{equation}

\begin{equation}\label{eq:ineq_exp}
    1 - z^2 \leq \Omega(z) = \exp(-z^2) \leq 1 - z^2 + \frac{z^4}{2}.
\end{equation}

    Finally, we introduce the functional 
    \[
    \mathbf{H}_\infty:\Gamma(\mu_X, \mu_Y) \to \R \cup \{\infty\},
    \]
    which is defined according to the following cases:
    \begin{itemize}[leftmargin=*]
    \item If $\gamma_0 = 0$ or $\gamma_1 = 0$ (equivalently, $\mu_X = 0$ or $\mu_Y = 0$), then $\mathbf{H}_\infty(\gamma_0,\gamma_1) \coloneqq 0$. We note that $\mathbf{H}_\infty = 0$ whenever one of the semi-couplings has zero mass. This is a direct consequence of the normalization in $\widehat{\mathbf{H}}_\delta$, where the diverging terms are subtracted to obtain a finite limit. 
    \item If $\gamma_1 = \alpha \gamma_0$ for $\alpha \coloneqq \mu_Y(Y)/\mu_X(X)$ (assuming $\mu_X(X) > 0$), then 
    \[
    \mathbf{H}_\infty(\gamma_0,\gamma_1) \coloneqq 2 C \alpha \int_{(X \times Y)^2}|\omega_X(x,x')-\omega_Y(y,y')|^2 \textnormal{d}\gamma_0(x,y) \textnormal{d}\gamma_0(x',y'),
    \]
    where $C$ is as in \eqref{eqn:Omega_global_lower} and \eqref{eqn:Omega_local_bounds}.
    \item If neither of the previous conditions apply, then $\mathbf{H}_\infty(\gamma_0,\gamma_1) \coloneqq \infty$. 
    \end{itemize} 
    
    With the definitions and assumptions above, the main result of this subsection is stated as follows.

\begin{theorem}[$\Gamma$-Convergence for CGW]\label{thm:Gamma_convergence}
    Let $\mathcal{N}_X$ and $\mathcal{N}_Y$ be measure networks, and let $\widehat{\mathbf{H}}_\delta$ and $\mathbf{H}_\infty$ be the functionals defined above.  Then we have $\widehat{\mathbf{H}}_{\delta}$ $\stackrel{\Gamma}{\to}$ $\mathbf{H}_{\infty}$.
\end{theorem}

The strategy of the proof is inspired by that of \cite[Theorem 5.10]{chizat2018unbalanced}, which considers variational convergence in the WFR setting. We will use the following preliminary results, which both assume fixed measure networks $\mathcal{N}_X$ and $\mathcal{N}_Y$ and the polynomial estimates on $\Omega$ described above. 

\begin{lemma}\label{lem:H_hat_decomposition}
    For sufficiently large $\delta$, the adapted $\CGW$ loss $\widehat{\mathbf{H}}_\delta$ in \eqref{eqn:augmented_CGW} can be expressed as 
    \begin{equation}\label{eqn:H_hat_decomposition}
    \widehat{\mathbf{H}}_\delta(\gamma_0,\gamma_1) = 8\delta^2 U(\gamma_0,\gamma_1) + 2 C V(\gamma_0,\gamma_1) -\frac{C'}{2\delta^2} W(\gamma_0,\gamma_1),
    \end{equation}
    where 
    \begin{itemize}[leftmargin=*]
        \item $C$ and $C'$ are as in \eqref{eqn:Omega_global_lower} and \eqref{eqn:Omega_local_bounds}, respectively;
        \item $U$ is a lower semi-continuous function such that $U(\gamma_0,\gamma_1) \geq 0$, with equality if and only if $\gamma_0 = 0$, $\gamma_1 = 0$, or $\gamma_0 = \frac{\mu_Y(Y)}{\mu_X(X)} \gamma_1$;
        \item $V$ is a function defined by 
        \[
        V(\gamma_0,\gamma_1) \coloneqq  \int_{(X \times Y)^2} (\omega_X - \omega_Y)^2 \sqrt{\frac{\mathrm{d}\gamma_0}{\mathrm{d}\gamma}\frac{\mathrm{d}\gamma_1}{\mathrm{d}\gamma}\frac{\mathrm{d}\gamma_0}{\mathrm{d}\gamma}\frac{\mathrm{d}\gamma_1}{\mathrm{d}\gamma}} \mathrm{d}\gamma \mathrm{d}\gamma
        \]
        (where function arguments in the integrands are suppressed for space);
        \item $W$ is a non-negative, bounded function. 
    \end{itemize}
\end{lemma}

\begin{proof}
    Using the bounds \eqref{eqn:Omega_global_lower} and \eqref{eqn:Omega_local_bounds}, and the expression \eqref{eqn:augmented_CGW}, there must be a number $I(\gamma_0,\gamma_1) \in [0,1]$ such that 
    \begin{align*}
        \widehat{\mathbf{H}}_\delta(\gamma_0,\gamma_1) = &8\delta^2 \underbrace{\left[\mu_X(X) \mu_Y(Y) - \int_{(X \times Y)^2} \sqrt{\frac{\mathrm{d}\gamma_0}{\mathrm{d}\gamma}\frac{\mathrm{d}\gamma_1}{\mathrm{d}\gamma}\frac{\mathrm{d}\gamma_0}{\mathrm{d}\gamma}\frac{\mathrm{d}\gamma_1}{\mathrm{d}\gamma}} \mathrm{d}\gamma \mathrm{d}\gamma  \right]}_{:=U} \\
        &+ 2 C \underbrace{\left[\int_{(X \times Y)^2} (\omega_X - \omega_Y)^2 \sqrt{\frac{\mathrm{d}\gamma_0}{\mathrm{d}\gamma}\frac{\mathrm{d}\gamma_1}{\mathrm{d}\gamma}\frac{\mathrm{d}\gamma_0}{\mathrm{d}\gamma}\frac{\mathrm{d}\gamma_1}{\mathrm{d}\gamma}} \mathrm{d}\gamma \mathrm{d}\gamma \right]}_{:=V} \\
        & - \frac{C'}{2\delta^2} \underbrace{\left[ I(\gamma_0,\gamma_1) \int_{(X \times Y)^2} (\omega_X - \omega_Y)^4 \sqrt{\frac{\mathrm{d}\gamma_0}{\mathrm{d}\gamma}\frac{\mathrm{d}\gamma_1}{\mathrm{d}\gamma}\frac{\mathrm{d}\gamma_0}{\mathrm{d}\gamma}\frac{\mathrm{d}\gamma_1}{\mathrm{d}\gamma}} \mathrm{d}\gamma \mathrm{d}\gamma\right]}_{:=W}.
    \end{align*}
    The functions $U$, $V$ and $W$ are then defined accordingly by the bracketed formulas in the respective lines of this expression. It remains to show that these functions satisfy the claimed properties. 

    The function $U$ is lower semi-continuous by \cite[Lemma 5.12]{chizat2018unbalanced}. Moreover, Cauchy-Schwarz applied to the functions 
    \begin{equation}\label{eqn:CS_functions}
        (x,y,x',y') \mapsto \sqrt{\frac{\mathrm{d}\gamma_i}{\mathrm{d}\gamma}(x,y)\frac{\mathrm{d}\gamma_i}{\mathrm{d}\gamma}(x',y')}, \qquad i \in \{0,1\},
    \end{equation}
    implies that 
    \begin{align*}
    &\int_{(X \times Y)^2} \sqrt{\frac{\mathrm{d}\gamma_0}{\mathrm{d}\gamma}\frac{\mathrm{d}\gamma_1}{\mathrm{d}\gamma}\frac{\mathrm{d}\gamma_0}{\mathrm{d}\gamma}\frac{\mathrm{d}\gamma_1}{\mathrm{d}\gamma}} \mathrm{d}\gamma \mathrm{d}\gamma \\
    &\qquad \leq \left(\int_{(X \times Y)^2} \frac{\mathrm{d}\gamma_0}{\mathrm{d}\gamma} \frac{\mathrm{d}\gamma_0}{\mathrm{d} \gamma} \mathrm{d}\gamma \mathrm{d}\gamma \right)^{1/2} \left(\int_{(X \times Y)^2} \frac{\mathrm{d}\gamma_1}{\mathrm{d}\gamma} \frac{\mathrm{d}\gamma_1}{\mathrm{d} \gamma} \mathrm{d}\gamma \mathrm{d}\gamma \right)^{1/2} = \mu_X(X)\mu_Y(Y),
    \end{align*}
    with equality if and only if the $i=0$ and $i=1$ instances of the function \eqref{eqn:CS_functions} are linearly dependent. An argument similar to the one used in the proof of Theorem \ref{thm:psuedometric} shows that this implies that $\gamma_0 = 0$ or $\gamma_1 = 0$ or $\gamma_1 = \alpha \gamma_0$ for some constant $\alpha$, which is then forced to be $\alpha = \mu_Y(Y)/ \mu_X(X)$.
 
    Finally, $W$ is bounded due to the fact that $I(\gamma_0,\gamma_1) \in [0,1]$ and $\omega_X$ and $\omega_Y$ are assumed to be bounded kernels. 
\end{proof}

\begin{corollary}\label{cor:liminf_bound}
    Fix natural numbers $\delta_1 \leq \delta_2 \leq \delta$ with $\delta_1$  sufficiently large, in the sense of Lemma \ref{lem:H_hat_decomposition}. Let $(\gamma_0,\gamma_1) \in \Gamma(\mu_X,\mu_Y)$ and let $\{(\gamma_{0,\delta},\gamma_{1,\delta})\}_{\delta \in \mathbb{N}}$ be a sequence in $\Gamma(\mu_X,\mu_Y)$ which converges weakly to $(\gamma_0,\gamma_1)$. Then there exists a constant $K$ (depending on the measure networks $\mathcal{N}_X$ and $\mathcal{N}_Y$, and the polynomial constants $C$ and $C'$) such that 
    \[
    \liminf_{\delta \to \infty} \widehat{\mathbf{H}}_\delta(\gamma_{0,\delta},\gamma_{1,\delta}) \geq \widehat{\mathbf{H}}_{\delta_1}(\gamma_0,\gamma_1) + 8 (\delta_2^2 - \delta_1^2) U(\gamma_0,\gamma_1) - \frac{K}{\delta_1^2},
    \]
    with $U$ as in Lemma \ref{lem:H_hat_decomposition}.
\end{corollary}

\begin{proof}
    For $\delta_1 \leq \delta_2 \leq \delta$, with $\delta_1$ sufficiently large, the expression from Lemma \ref{lem:H_hat_decomposition} yields
    \begin{align*}
    &\widehat{\mathbf{H}}_{\delta}(\gamma_{0,\delta},\gamma_{1,\delta}) - \widehat{\mathbf{H}}_{\delta_1}(\gamma_{0,\delta},\gamma_{1,\delta}) \\
    &\qquad \qquad = 8(\delta^2 - \delta_1^2) U(\gamma_{0,\delta},\gamma_{1,\delta}) - \frac{C'}{2}\left(\frac{1}{\delta^2}-\frac{1}{\delta^2_1} \right)W(\gamma_{0,\delta},\gamma_{1,\delta}),
    \end{align*}
    for some $I(\gamma_0,\gamma_1)\in [-1,1]$ in the definition of $W(\gamma_{0,\delta},\gamma_{1,\delta})$ accounting for the difference in the sign. We now assume without loss of generality that $\delta \geq \delta_1$, so that, also invoking the boundedness of $W$, we may choose $K$ satisfying
    \begin{equation}\label{eqn:bound_W_C}
    -\frac{C'}{2}\left(\frac{1}{\delta^2}-\frac{1}{\delta^2_1}\right)W(\gamma_{0,\delta},\gamma_{1,\delta}) \geq -\frac{K}{\delta_1^2}.
    \end{equation}
    Using \eqref{eqn:bound_W_C} and the fact $\delta \geq \delta_2$, we can then bound 
    \begin{align*} 
    \widehat{\mathbf{H}}_{\delta}(\gamma_{0,\delta},\gamma_{1,\delta}) \geq \widehat{\mathbf{H}}_{\delta_1}(\gamma_{0,\delta},\gamma_{1,\delta}) + 8(\delta_2^2 - \delta_1^2) U(\gamma_{0,\delta},\gamma_{1,\delta}) - \frac{K}{\delta_1^2}\\
    \end{align*}
    The desired inequality follows by taking the lim-inf and applying lower semi-continuity of the functions $\widehat{\mathbf{H}}_{\delta_1}$ and $U$.
\end{proof}

We proceed with the proof of the main result of this subsection.

\begin{proof}[Proof of Theorem \ref{thm:Gamma_convergence}]
We begin with the \textbf{lim-sup inequality}. Let $(\gamma_0,\gamma_1) \in \Gamma(\mu_X,\mu_Y)$ and consider the constant sequence $\{(\gamma_{0,\delta},\gamma_{1,\delta})\}_{\delta \in \mathbb{N}}$ with $\gamma_{0,\delta} = \gamma_0$, $\gamma_{1,\delta} = \gamma_1$ for all $\delta$. It is obvious that $(\gamma_{0,\delta},\gamma_{1,\delta}) \to (\gamma_0,\gamma_1)$, so we wish to show that $\limsup \widehat{\mathbf{H}}_\delta(\gamma_{0,\delta},\gamma_{1,\delta}) \leq  \mathbf{H}_\infty(\gamma_0,\gamma_1)$. We do so by checking the cases in the definition of $\mathbf{H}_\infty$.  

First, consider the case that $\gamma_0 = 0$ or $\gamma_1 = 0$. Then it follows easily from the expression \eqref{eqn:augmented_CGW} that 
\[
\widehat{\mathbf{H}}_\delta(\gamma_{0,\delta},\gamma_{1,\delta}) = 0 = \mathbf{H}_\infty(\gamma_0,\gamma_1),
\]
and the lim-sup inequality follows. 

Next, consider the case where $\gamma_1 = \alpha \gamma_0$, for $\alpha = \mu_Y(Y)/\mu_X(X)$. Using the expression \eqref{eqn:augmented_CGW} and the bound \eqref{eqn:Omega_global_lower}, we have 
\begin{align*}
    \widehat{\mathbf{H}}_\delta(\gamma_{0,\delta},\gamma_{1,\delta}) &= \widehat{\mathbf{H}}_\delta(\gamma_{0},\alpha \gamma_{0}) \\
    &= 8\delta^2 \left(\mu_X(X)\mu_Y(Y) - \alpha \int_{(X \times Y)^2} \Omega\left(\frac{|\omega_X - \omega_Y|}{2\delta}\right) \mathrm{d}\gamma_0 \mathrm{d}\gamma_0  \right) \\
    &\leq 8\delta^2 \left(\mu_X(X)\mu_Y(Y) - \alpha \int_{(X \times Y)^2} 1- C\left(\frac{|\omega_X - \omega_Y|}{2\delta}\right)^2 \mathrm{d}\gamma_0 \mathrm{d}\gamma_0  \right) \\
    &\leq 8\delta^2 \left(\mu_X(X)\mu_Y(Y) - \alpha \mu_X(X)^2 + \frac{C \alpha}{4 \delta^2}  \int_{(X \times Y)^2} \left(\omega_X - \omega_Y\right)^2 \mathrm{d}\gamma_0 \mathrm{d}\gamma_0  \right) \\
    &= 2C\alpha \int_{(X \times Y)^2} \left(\omega_X - \omega_Y\right)^2 \mathrm{d}\gamma_0 \mathrm{d}\gamma_0 = \mathbf{H}_\infty(\gamma_0,\gamma_1).
\end{align*}
The lim-sup inequality in this case then follows. 

In the last case of the definition of $\mathbf{H}_\infty$ (where $\mathbf{H}_\infty = \infty$), the lim-sup inequality is obvious. The proof of the lim-sup inequality is therefore complete.

Next, we prove the \textbf{lim-inf inequality}. Let $(\gamma_0,\gamma_1) \in \Gamma(\mu_X,\mu_Y)$ be the limit of a sequence of semi-couplings $\{(\gamma_{0,\delta},\gamma_{1,\delta})\}_{\delta \in \mathbb{N}}$. We establish the inequality by checking cases. 

To check the first case of $\mathbf{H}_\infty$, assume without loss of generality that $\gamma_0 = 0$. Then $\mu_X = 0$, so that $\gamma_{0,\delta} = 0$ for all $\delta$. In turn, it follows easily that $\widehat{\mathbf{H}}_\delta(\gamma_{0,\delta},\gamma_{1,\delta}) = 0$ for all $\delta$, whence the lim-inf inequality follows. 

Next, suppose that $\gamma_1 = \alpha \gamma_0$, for $\alpha = \mu_Y(Y)/\mu_X(X)$ ($\mu_X(X) \neq 0$). By Corollary \ref{cor:liminf_bound} and the fact that $U(\gamma_0,\gamma_1) = 0$ (from Lemma \ref{lem:H_hat_decomposition}), we have 
\[
\liminf_{\delta \to \infty} \widehat{\mathbf{H}}_\delta(\gamma_{0,\delta},\gamma_{1,\delta}) \geq \widehat{\mathbf{H}}_{\delta_1}(\gamma_0,\gamma_1) - \frac{K}{\delta_1^2}
\]
for all sufficiently large $\delta_1$. In particular, Lemma \ref{lem:H_hat_decomposition} yields
\begin{align*}
    \liminf_{\delta \to \infty} \widehat{\mathbf{H}}_\delta(\gamma_{0,\delta},\gamma_{1,\delta}) &\geq \lim_{\delta_1 \to \infty} \widehat{\mathbf{H}}_{\delta_1}( \gamma_0,\alpha \gamma_0) \\
    &\geq \lim_{\delta_1 \to \infty} 8\delta_1^2 \cdot 0 + 2 C V(\gamma_0, \gamma_1) - \frac{C'}{2\delta_1^2} W(\gamma_0, 
    \alpha \gamma_0) \\
    &= 2 C V(\gamma_0, \gamma_1) \\
    &= 2 C \alpha \int_{(X \times Y)^2} |\omega_X - \omega_Y|^2 \mathrm{d}\gamma_0 \mathrm{d}\gamma_0 = \mathbf{H}_\infty(\gamma_0,\gamma_1)
\end{align*}

Finally, consider the case where $\mathbf{H}_\infty = \infty$. By Corollary \ref{cor:liminf_bound},
    \[
    \liminf_{\delta \to \infty} \widehat{\mathbf{H}}_\delta(\gamma_{0,\delta},\gamma_{1,\delta}) \geq \widehat{\mathbf{H}}_{\delta_1}(\gamma_0,\gamma_1) + 8 (\delta_2^2 - \delta_1^2) U(\gamma_0,\gamma_1) - \frac{K}{\delta_1^2},
    \]
for all sufficiently large $\delta_1 \leq \delta_2$. Our assumptions on $(\gamma_0,\gamma_1)$ imply, by Lemma \ref{lem:H_hat_decomposition}, that $U(\gamma_0,\gamma_1) > 0$. Since we can take $\delta_2$ arbitrarily larger than $\delta_1$, it follows that 
\[
    \liminf_{\delta \to \infty} \widehat{\mathbf{H}}_\delta(\gamma_{0,\delta},\gamma_{1,\delta}) =\infty = \mathbf{H}_\infty(\gamma_0,\gamma_1),
    \]
and this completes the proof. 
\end{proof}

Next, we state a corollary linking our CGW metric defined in \eqref{eqn:def_cgw} to the $\GW_2$ metric as a consequence of Theorem \ref{thm:Gamma_convergence} (see Figure \ref{fig:matching} for a practical demonstration). This corollary is analogous to the property of WFR distance stated in \cite[Remark 5.11]{chizat2018unbalanced}. The proof is a standard argument in the context of $\Gamma$-convergence.

\begin{corollary}\label{cor:Gamma_convergence_to_GW}
    For probability measure networks $\mathcal{N}_X$ and $\mathcal{N}_Y$, 
    \begin{equation*}
\lim_{\delta \to \infty} \CGW_\delta\left(\mathcal{N}_X, \mathcal{N}_Y\right) = \sqrt{2C} \cdot \GW_2\left(\mathcal{N}_X, \mathcal{N}_Y\right),
\end{equation*}
where $C$ is as in \eqref{eqn:Omega_global_lower} and \eqref{eqn:Omega_local_bounds}. Moreover, the optimal solutions converge in the sense that any sequence of optimal semi-couplings for $\CGW_\delta$ admits a subsequence weakly converging to an optimal semi-coupling for $\GW_2$. 
\end{corollary}

\subsection{Bounds on \texorpdfstring{$\CGW$}{CGW}}
Next, we establish bounds for the $\CGW$ metric. The first is an upper bound by the $\GW_2$ distance, which is intuitively and formally obvious, as the $\CGW$ distance allows additional flexibility in probabilistic matching between the spaces.

\begin{proposition}[Upper Bound of $\CGW$ metric]\label{prop:bound} Let $\Omega(x)$ such that $\frac{x^2}{C}\geq 1-\Omega(x)$ for $C>0$. For probability measure networks $\mathcal{N}_X$ and $\mathcal{N}_Y$, we have

\begin{equation*}
\frac{2 \sqrt{2}}{\sqrt{C}}\GW_2(\mathcal{N}_X,\mathcal{N}_Y)\geq \CGW_\delta(\mathcal{N}_X,\mathcal{N}_Y).
\end{equation*}
\end{proposition}
\begin{proof}
Since $\frac{x^2}{C}\geq1-\Omega(x)$ holds in general, substituting $x = \frac{z}{2\delta}$ yields $\frac{z^2}{C4\delta^2}\geq 1-\Omega\left(\frac{z}{2\delta}\right)$, or $z^2\geq 2C\delta^2\left(2-2\Omega\left(\frac{z}{2\delta}\right)\right)$. Let $\pi$ be an optimal coupling that attains $\GW_2(\mathcal{N}_X,\mathcal{N}_Y)$. Then, taking $z = |\omega_X(x,x') - \omega_Y(y,y')|$, we have 
\begin{align*}
    &4 \cdot \GW_2(\mathcal{N}_X,\mathcal{N}_Y)^2\\
    &\qquad = \int_{(X\times Y)^2}|\omega_X(x,x')-\omega_Y(y,y')|^2 \textnormal{d}\pi(x,y) \textnormal{d}\pi(x',y')\\
    &\qquad \geq \frac{C}{2} \int_{(X\times Y)^2} 4\delta^2\left(2 - 2\Omega\left(\frac{|\omega_X(x,x')-\omega_Y(y,y')|}{2\delta}\right) \right)\textnormal{d}\pi(x,y) \textnormal{d}\pi(x',y')\\
    &\qquad = \frac{C}{2} \int_{(X\times Y)^2}\mathsf{d}_{\Co_\delta[\R]}([\omega_X(x,x'),1],[\omega_Y(y,y'),1])^2\textnormal{d}\pi(x,y) \textnormal{d}\pi(x',y')\\
    &\qquad \geq \frac{C}{2}\CGW_\delta(\mathcal{N}_X,\mathcal{N}_Y)^2.
\end{align*}
\end{proof}
\begin{remark}

    If  $\Omega(z) = \overline{\cos}(z)$ by \eqref{eq:ineq_cos}, $\frac{x^2}{2}\geq 1-\Omega(x)$. With $C=2$, our estimate becomes
    \begin{equation*}
2\cdot \GW_2(\mathcal{N}_X,\mathcal{N}_Y)\geq \CGW_\delta(\mathcal{N}_X,\mathcal{N}_Y).
\end{equation*}
If $\Omega(z) = \exp{(-z^2)}$ by \eqref{eq:ineq_exp}, $x^2\geq 1-\Omega(x)$. With $C=1$, we have \begin{equation*}
2\sqrt{2}\cdot \GW_2(\mathcal{N}_X,\mathcal{N}_Y)\geq \CGW_\delta(\mathcal{N}_X,\mathcal{N}_Y).
\end{equation*}
\end{remark}

Next we establish a lower bound by the UOT metric. This is analogous to the bounds on $\GW_2$ distance derived in \cite[Section 6]{F_Memoli_2011}. Our result solves the dual roles of establishing the Lipschitz stability of certain distributional invariants of a measure network, and offering a computationally tractable lower estimate of CGW distance.

\begin{proposition}[Lower Bound for $\CGW$] Let $\mathcal{N}_X $ and $\mathcal{N}_Y$ be measure networks. Define $\nu_X$ and $\nu_Y$ as pushforward measures on $\R$ given by 
\[
\nu_X\coloneqq(\omega_X)_\#(\mu_X \otimes \mu_X) \quad \mbox{and} \quad \nu_Y\coloneqq(\omega_Y)_\#(\mu_Y \otimes \mu_Y).
\]
If $\delta>0$, then
\begin{equation*}
\UOT_\delta(\nu_X, \nu_Y) \leq \CGW_\delta(\mathcal{N}_X,\mathcal{N}_Y).
\end{equation*}
\end{proposition}

\begin{proof}
 To prepare the proof, we note that, by Lemma \ref{lem:CGW_loss_formula},  $\CGW_\delta(\mathcal{N}_X,\mathcal{N}_Y)^2/(4\delta^2)$ is given by 
    \begin{equation}\label{eqn:lower_bound_CGW}
    \begin{aligned}
    &\mu_X(X)^2 + \mu_Y(Y)^2 \\
    &\qquad   - 2\sup_{(\gamma_0,\gamma_1) \in \Gamma(\mu_X,\mu_Y)} \int_{(X\times Y)^2} \Omega\left(\frac{|\omega_X(x,x')-\omega_Y(y,y')|}{2\delta}\right)\\
    &\qquad \qquad \qquad \cdot \sqrt{\frac{\textnormal{d}\gamma_0}{\textnormal{d}\gamma}(x,y)\frac{\textnormal{d}\gamma_0}{\textnormal{d}\gamma}(x',y')\frac{\textnormal{d}\gamma_1}{\textnormal{d}\gamma}(x,y)\frac{\textnormal{d}\gamma_1}{\textnormal{d}\gamma}(x',y')} \textnormal{d}\gamma(x,y) \textnormal{d}\gamma(x',y').\\
    \end{aligned}
    \end{equation}
A similar proof shows that  $\UOT_\delta(\nu_X,\nu_Y)^2/(4\delta^2)$ is expressed as
    \begin{equation}\label{eqn:lower_bound_UOT}
    \begin{aligned}
     &\mu_X(X)^2 + \mu_Y(Y)^2 \\
     & \qquad \qquad -2\sup_{(\pi_0,\pi_1) \in \Gamma(\nu_X,\nu_Y)}\int_{\R\times\R}\Omega\left(\frac{|u-v|}{2\delta}\right)\sqrt{\frac{\textnormal{d}\pi_0}{\textnormal{d}\pi}(u,v)}\sqrt{\frac{\textnormal{d}\pi_1}{\textnormal{d}\pi}(u,v)} \textnormal{d}\pi.
    \end{aligned}
    \end{equation}
    
Now let \((\gamma_0, \gamma_1) \in \Gamma(\mu_X, \mu_Y)\) be a semicoupling which attains \(\CGW_\delta(\mathcal{N}_X, \mathcal{N}_Y)\), with reference measure \(\gamma_0, \gamma_1 \ll \gamma\). Define $g:X\times Y\times X\times Y \to \R\times\R$ by 
$g(x,y,x',y') = (\omega_X(x,x'),\omega_Y(y,y'))$.  
Since \(g\) may not be injective, we consider the disintegration of \(\gamma \otimes \gamma\) over \(g\):
\[
\gamma \otimes \gamma = \int_{\mathbb{R}^2} (\gamma \otimes \gamma)_{(u,v)} \, \textnormal{d} g_\#(\gamma \otimes \gamma)(u,v),
\]
where \((\gamma \otimes \gamma)_{(u,v)}\) is a probability measure supported on the fiber
\[
g^{-1}(u,v) = \{ (x,y,x',y') : \omega_X(x,x') = u, \ \omega_Y(y,y') = v \}.
\]
% Similarly, we disintegrate \(\gamma_0 \otimes \gamma_0\) and \(\gamma_1 \otimes \gamma_1\):
% \[
% \gamma_0 \otimes \gamma_0 = \int_{\mathbb{R}^2} (\gamma_0 \otimes \gamma_0)_{(u,v)} \, \textnormal{d} g_\#(\gamma_0 \otimes \gamma_0)(u,v)
% \] 
% and 
% \[\gamma_1 \otimes \gamma_1 = \int_{\mathbb{R}^2} (\gamma_1 \otimes \gamma_1)_{(u,v)} \, \textnormal{d} g_\#(\gamma_1 \otimes \gamma_1)(u,v).
% \]
% Since $\gamma_0\otimes\gamma_0$, $\gamma_1\otimes\gamma_1
% \ll
% \gamma\otimes\gamma$, by compatibility of disintegration with absolute continuity, we may choose the disintegrations so that for
% \(g_\#(\gamma\otimes\gamma)\)-a.e. \((u,v)\),
% \[
% (\gamma_0\otimes\gamma_0)_{(u,v)}
% \ll
% (\gamma\otimes\gamma)_{(u,v)},
% \qquad
% (\gamma_1\otimes\gamma_1)_{(u,v)}
% \ll
% (\gamma\otimes\gamma)_{(u,v)}.
% \]
Using the formula  \eqref{eqn:disintegration_measures}, together with the Cauchy-Schwarz inequality, we have
\[
\begin{aligned}
&\int_{g^{-1}(u,v)} \sqrt{\frac{\textnormal{d}(\gamma_0 \otimes \gamma_0)}{\textnormal{d}(\gamma \otimes \gamma)} \frac{\textnormal{d}(\gamma_1 \otimes \gamma_1)}{\textnormal{d}(\gamma \otimes \gamma)}} 
\, \textnormal{d}(\gamma \otimes \gamma)_{(u,v)}\\
&\qquad \qquad \le \sqrt{\frac{\textnormal{d} g_\#(\gamma_0 \otimes \gamma_0)}{\textnormal{d} g_\#(\gamma \otimes \gamma)}(u,v) \frac{\textnormal{d} g_\#(\gamma_1 \otimes \gamma_1)}{\textnormal{d} g_\#(\gamma \otimes \gamma)}(u,v)}.\\
\end{aligned}
\]
% \[
% \begin{aligned}
% &\int_{g^{-1}(u,v)} \sqrt{\frac{\textnormal{d}(\gamma_0 \otimes \gamma_0)_{(u,v)}}{\textnormal{d}(\gamma \otimes \gamma)_{(u,v)}} \frac{\textnormal{d}(\gamma_1 \otimes \gamma_1)_{(u,v)}}{\textnormal{d}(\gamma \otimes \gamma)_{(u,v)}}} 
% \, \textnormal{d}(\gamma \otimes \gamma)_{(u,v)}\\
% &\qquad \qquad \le \sqrt{\frac{\textnormal{d} g_\#(\gamma_0 \otimes \gamma_0)}{\textnormal{d} g_\#(\gamma \otimes \gamma)}(u,v) \frac{\textnormal{d} g_\#(\gamma_1 \otimes \gamma_1)}{\textnormal{d} g_\#(\gamma \otimes \gamma)}(u,v)}.\\
% \end{aligned}
% \]
Applying this inequality, we have 
\[
\begin{aligned}
&\int_{(X \times Y)^2} \Omega\Big(\frac{|\omega_X(x,x') - \omega_Y(y,y')|}{2\delta}\Big) 
\sqrt{\frac{\textnormal{d}\gamma_0}{\textnormal{d}\gamma}(x,y) \frac{\textnormal{d}\gamma_0}{\textnormal{d}\gamma}(x',y') \frac{\textnormal{d}\gamma_1}{\textnormal{d}\gamma}(x,y) \frac{\textnormal{d}\gamma_1}{\textnormal{d}\gamma}(x',y')} \, \textnormal{d}\gamma \textnormal{d}\gamma \\
&\quad = \int_{(X \times Y)^2} \Omega\Big(\frac{|\omega_X(x,x') - \omega_Y(y,y')|}{2\delta}\Big) \\
&\qquad \qquad 
\cdot \sqrt{\frac{\textnormal{d}(\gamma_0 \otimes \gamma_0)}{\textnormal{d}(\gamma \otimes \gamma)}(x,y,x',y')\frac{\textnormal{d}(\gamma_1 \otimes \gamma_1)}{\textnormal{d}(\gamma \otimes \gamma)}(x,y,x',y')} \, \textnormal{d}(\gamma \otimes \gamma) \\
&\quad = \int_{\R^2} \int_{g^{-1}(u,v)} \Omega\Big(\frac{|u - v|}{2\delta}\Big) \\
&\qquad \qquad   
\cdot \sqrt{\frac{\textnormal{d}(\gamma_0 \otimes \gamma_0)}{\textnormal{d}(\gamma \otimes \gamma)}\frac{\textnormal{d}(\gamma_1 \otimes \gamma_1)}{\textnormal{d}(\gamma \otimes \gamma)}} \, \textnormal{d}(\gamma \otimes \gamma)_{(u,v)} \; \mathrm{d}g_\#(\gamma \otimes \gamma) (u,v) \\
&\quad \le \int_{\mathbb{R}^2} \Omega\Big(\frac{|u-v|}{2\delta}\Big) 
\sqrt{\frac{\textnormal{d} g_\#(\gamma_0 \otimes \gamma_0)}{\textnormal{d} g_\#(\gamma \otimes \gamma)}(u,v) \frac{\textnormal{d} g_\#(\gamma_1 \otimes \gamma_1)}{\textnormal{d} g_\#(\gamma \otimes \gamma)}(u,v)} \, \textnormal{d} g_\#(\gamma \otimes \gamma)(u,v) \\
&\quad \le \sup_{(\pi_0, \pi_1) \in \Gamma(\nu_X, \nu_Y)} \int_{\mathbb{R}^2} \Omega\Big(\frac{|u-v|}{2\delta}\Big) \sqrt{\frac{\textnormal{d}\pi_0}{\textnormal{d}\pi} \frac{\textnormal{d}\pi_1}{\textnormal{d}\pi}} \, \textnormal{d}\pi.
\end{aligned}
\]
In light of the expressions \eqref{eqn:lower_bound_CGW} and \eqref{eqn:lower_bound_UOT}, we conclude the desired inequality
\[
\CGW_\delta(\mathcal{N}_X,\mathcal{N}_Y) \ge \UOT_\delta(\nu_X,\nu_Y).
\]

\end{proof}

\subsection{Robustness Guarantees}

We analyze the robustness of the $\CGW$ metric under contamination by outliers. This is motivated by the known sensitivity of several optimal transport-based distances, including the Wasserstein $\W_2$ \cite[Section Discussion]{oliver2025uncertainty}, Gromov-Wasserstein ($\GW_2$) \cite{Robust_GW}, and Co-Optimal Transport ($\COT$) \cite[Proposition 2]{UCOOT} distances to perturbations in the marginal distributions. Even a small fraction of outliers can induce significant distortions, making such distances unstable in practice.

To mitigate this sensitivity, several robust variants have been introduced. For example, the Robust Gromov-Wasserstein distance ($\mathrm{RGW}$) \cite{Robust_GW} was explicitly designed to prevent divergence of the $\GW_2$ distance as the separation between inliers and outliers increases. Related approaches include Unbalanced Optimal Transport ~\cite[Lemma 1.1]{Fatras_2021}, Robust Wasserstein~\cite[Theorem 3.1]{raghvendra2024new}, Unbalanced Gromov-Wasserstein~\cite[Remark 2.4]{Robust_GW}, Partial and Robust Gromov-Wasserstein~\cite[Corollary 4.13]{chhoa2024metric} and ~\cite{baipartial,bailinear}, and Unbalanced Co-Optimal Transport \cite[Theorem 1]{UCOOT} distances, for which lower bounds on robustness have been established.

The main result quantifies the effect on CGW distance of perturbing the measure of a measure network.

\begin{theorem}\label{thm:robustness}
    Let $\mathcal{N}_X = (X,\mu_X,\omega_X)$ and $\mathcal{N}_X' = (X,\mu_X',\omega_X)$ be measure networks over the same underlying set $X$, with the same network kernel $\omega_X$. Assume $\mu'_X\ll\mu_X$, and that there exists $\varepsilon \in [0,1)$ such that
$1 - \varepsilon \le \frac{\textnormal{d}\mu'_X}{\textnormal{d}\mu_X}(x) \le 1 + \varepsilon$ for $\mu_X$-almost every  $x \in X$. Then
    $\CGW_\delta(\mathcal{N}_X,\mathcal{N}_X') \leq  2 \cdot \delta \cdot \mu_X(X) \cdot \sqrt{\varepsilon^2 + 4 \varepsilon}$.
\end{theorem}

\begin{proof}
   We have 
\begin{equation}\label{eqn:total_variation_epsilon_bound}
   (1-\varepsilon) \mu_X \leq \mu_X' \leq (1+\varepsilon)\mu_X.
\end{equation}
   Define a semi-coupling $(\gamma_0,\gamma_1)$ between $\mu_X$ and $\mu_X'$ by 
   \[
   \gamma_0 = (\mathrm{id},\mathrm{id})_\# \mu_X \quad \mbox{and} \quad \gamma_1 = (\mathrm{id},\mathrm{id})_\# \mu_X'.
   \]
   By \eqref{eqn:total_variation_epsilon_bound}, $\gamma_1 \ll \gamma_0$, so that the $\CGW$ distance can be bounded as 
   \begin{align}
       &\CGW_\delta(\mathcal{N}_X,\mathcal{N}_X')^2 \nonumber \\
       &\qquad \leq 4 \delta^2 \big(\mu_X(X)^2 + \mu_X'(X)^2\big) \nonumber \\
       &\qquad \qquad - 8\delta^2 \int_{(X \times X)^2} \Omega\left(\frac{|\omega_X(x,x') - \omega_X(x,x')|}{2\delta} \right) \nonumber \\
       &\qquad \qquad \qquad \cdot \sqrt{\frac{\textnormal{d}\gamma_0}{\textnormal{d}\gamma_0}(x,y)\frac{\textnormal{d}\gamma_1}{\textnormal{d}\gamma_0}(x,y)\frac{\textnormal{d}\gamma_0}{\textnormal{d}\gamma_0}(x',y')\frac{\textnormal{d}\gamma_1}{\textnormal{d}\gamma_0}(x',y')} \textnormal{d}\gamma_0(x,y) \textnormal{d}\gamma_0(x',y') \label{eqn:robustness1} \\
       &\qquad = 4 \delta^2 \big(\mu_X(X)^2 + \mu_X'(X)^2\big) \nonumber \\
       &\qquad \qquad - 8\delta^2 \int_{X \times X} \Omega\left(\frac{|\omega_X(x,x') - \omega_X(x,x')|}{2\delta} \right) \nonumber \\
       &\qquad \qquad \qquad \cdot \sqrt{\frac{\textnormal{d}\gamma_1}{\textnormal{d}\gamma_0}(x,x)\frac{\textnormal{d}\gamma_1}{\textnormal{d}\gamma_0}(x',x')} \textnormal{d}\mu_X(x) \textnormal{d}\mu_X(x') \label{eqn:robustness2} \\
       &\quad \leq 4\delta^2 (1+(1+\varepsilon)^2) \mu_X(X)^2 - 8\delta^2 \int_{X \times X} (1-\varepsilon) \textnormal{d}\mu_X(x) \textnormal{d}\mu_X(x')\label{eqn:robustness3} \\
       &= 4 \delta^2 \mu_X(X)^2 \big(1 + (1+\varepsilon)^2 - 2 (1-\varepsilon)\big) \nonumber \\
       &= 4 \delta^2 \mu_X(X)^2 (\varepsilon^2 + 4 \varepsilon). \nonumber
   \end{align}
   The steps are justified as follows: \eqref{eqn:robustness1} follows by the suboptimality of our choice of semi-coupling and by Lemma \ref{lem:CGW_loss_formula}, where we have used $\gamma = \gamma_0$ as the reference measure in the loss $\mathbf{H}_\delta$; \eqref{eqn:robustness2} follows by the observation that $\gamma_0$ is supported on the diagonal, where $\textnormal{d}\gamma_0(x,x) = \textnormal{d}\mu_X(x)$, as well as the fact that $\frac{\textnormal{d}\gamma_0}{\textnormal{d}\gamma_0}(x,x) = 1$; \eqref{eqn:robustness3} applies \eqref{eqn:total_variation_epsilon_bound}, together with the fact that $\Omega(0) = 1$; the rest of the calculations are basic simplifications. This completes the proof.
\end{proof}

The next corollary shows that the CGW distance is robust to perturbations of the measures of the measure networks being compared. Its proof follows easily from Theorem \ref{thm:robustness} and the triangle inequality for CGW distance.

\begin{corollary}\label{cor:CGW_robustness}
    Let $\mathcal{N}_X$ and $\mathcal{N}_Y$ be measure networks and let 
    \[
    \mathcal{N}_X' = (X,\mu_X',\omega_X) \quad \mbox{and} \quad \mathcal{N}_Y' = (Y,\mu_Y',\omega_Y),
    \]
    where $\mu_X'$ and $\mu_Y'$ be $\varepsilon$-perturbations of $\mu_X$ and $\mu_Y$, respectively. Assume $\mu'_X\ll\mu_X$, $\mu'_Y\ll\mu_Y$ and there exists $\varepsilon \in [0,1)$ such that $\frac{\textnormal{d}\mu_X'}{\textnormal{d}\mu_X}(x), \frac{\textnormal{d}\mu_Y'}{\textnormal{d}\mu_Y}(y) \in [1-\varepsilon,1+\varepsilon]$ for $\mu_X$, $\mu_Y$ almost everywhere. Then 
    \[
    |\CGW_\delta(\mathcal{N}_X,\mathcal{N}_Y) - \CGW_\delta(\mathcal{N}_X',\mathcal{N}_Y')| \leq 2\delta \sqrt{\varepsilon^2 + 4\varepsilon} (\mu_X(X) + \mu_Y(Y)). 
    \]
\end{corollary}
It is natural to ask whether classical GW distance happens to enjoy a similar robustness. However, it is not hard to construct a counterexample, as shown in the following proposition (cf.~\cite[Proposition 4.15]{chhoa2024metric}).
\begin{proposition}\label{prop:GW_not_robust}
Let $f:\R^{\geq 0} \to \R$ be a monotonically increasing function with $f(0) = 0$. For any $\varepsilon > 0$, there exist probability measure networks $\mathcal{N}_X$, $\mathcal{N}_Y$ and $\varepsilon$-perturbations $\mathcal{N}_X'$, $\mathcal{N}_Y'$ such that 
\begin{equation}\label{eqn:GW_not_robust}
|\GW_2(\mathcal{N}_X,\mathcal{N}_Y) - \GW_2(\mathcal{N}_X',\mathcal{N}_Y')| \geq f(\varepsilon).
\end{equation}
\end{proposition}

\begin{proof}
Let $\mathcal{N}_Y = \mathcal{N}_Y'$ be the one-point probability metric measure space.  
Let $\mathcal{N}_X = (X, \mathsf{d}_X, \mu_X)$ with $X = \{x_1, x_2\}$, and define the distance between the points as
\[
\mathsf{d}_X(x_1, x_2) = \sqrt{\frac{2}{(1-\varepsilon)\varepsilon}} \, f(\varepsilon),
\]
so that the scaled distance exactly compensates for the small \(\varepsilon\)-mass in the perturbed measure.  

Define the original measure as $\mu_X$ supported entirely on $x_1$, and the perturbed measure $\mu_X'$ as 
\[
\mu_X'(x_1) = 1-\varepsilon, \quad \mu_X'(x_2) = \varepsilon.
\]

By construction, the GW distance between the original network and the one-point network is zero:
\[
\GW_2(\mathcal{N}_X, \mathcal{N}_Y) = 0,
\]
because all mass in $\mu_X$ is concentrated at a single point, which can be perfectly matched with the single point in $\mathcal{N}_Y$.  

For the perturbed measure, using the two-point formula for GW distance (cf.~\cite[Theorem 5.1(f)]{F_Memoli_2011}), we have
\[
\GW_2(\mathcal{N}_X', \mathcal{N}_Y') 
= \sqrt{\frac{(1-\varepsilon)\varepsilon}{2}} \, \mathsf{d}_X(x_1,x_2)
= f(\varepsilon).
\]

Hence, by the reverse triangle inequality,
\[
|\GW_2(\mathcal{N}_X, \mathcal{N}_Y) - \GW_2(\mathcal{N}_X', \mathcal{N}_Y')| 
\geq |0 - f(\varepsilon)| = f(\varepsilon),
\]
which establishes \eqref{eqn:GW_not_robust}.
\end{proof}

\section{Conic Co-Optimal Transport Distance for Unbalanced Hypernetworks}\label{section:CCOOT} 
In this section, we define a notion of conic GW distance---dubbed \emph{conic Co-Optimal Transport}---which operates on generalized structures called \emph{measure hypernetworks}. The section begins with an explanation of the hypernetwork and co-optimal transport terminology.

\subsection{Hypernetworks and Co-Optimal Transport} 
The classical Gromov--Wasserstein framework provides a principled way to compare kernels of the form
$\omega : X \times X \to \mathbb{R}$, which encode general pairwise distances or similarities on some underlying set. Certain situations call for the need to compare more general maps of the form $\omega:X \times X' \to \R$, where the underlying sets $X$ and $X'$ may be distinct. Common examples of such maps include:
\begin{enumerate}
    \item \textbf{Data Matrices.} Data is generally represented as a rectangular matrix with, say, rows representing samples in the data set and columns representing features. Letting $X$ denote the index set for samples and $X'$ the index set for features, a data matrix can be considered as a map $\omega: X \times X' \to \R$, with $\omega(i,j)$ returning the $(i,j)$-entry of the matrix. The ability to compare maps of this form then leads to a methodology for comparing datasets, as in~\cite{Redko_2020}.
    \item \textbf{Cost Functions.} A commonly used, general formulation of the optimal transport problem is  $\inf_\pi \int_{X \times X'} \omega(x,x') \textnormal{d}\pi(x,x')$, where $\pi$ is a coupling of the underlying measures and $\omega:X \times X' \to \R$ is some fixed cost function. One may wish to understand properties of the space of cost functions, in which case it is useful to metrize this space (see~\cite{needham2024stability}).
    \item \textbf{Hypergraphs.} A hypergraph is a set of nodes $X$ together with a set $E$ of subsets of $X$; if each element of $E$ contains exactly two points, then this recovers the notion of a classical graph. This structure can be encoded via an adjacency function $\omega:X \times E \to \R$, with $\omega(x,e) = 1$ if $x \in e$ and $0$ otherwise. A framework for comparing functions of this form then yields a notion of distance between hypergraphs~\cite{Chowdhury_2024}.
\end{enumerate}
A variant of Gromov-Wasserstein distance which is able to compare structures of this form, dubbed \emph{co-optimal transport (COT)}, was introduced in~\cite{Redko_2020}. The COT framework was originally focused on comparison of data matrices; in this context, the COT comparison of $\omega_X:X \times X' \to \R$ and $\omega_Y:Y \times Y' \to \R$ yields a pair of probabilistic registrations between $X$ and $Y$ (the sample spaces) and between $X'$ and $Y'$ (the feature spaces). The original formulation of COT was generalized to the notion of a measure hypernetwork~\cite{Chowdhury_2024}, whose formalism we now recall.

\begin{mydef}[Measure Hypernetwork and Co-Optimal Transport]
    A \define{measure hypernetwork} is a structure of the form $\mathcal{H}_X = (X,\mu_X,X',\mu_X',\omega_X)$, where $(X,\mu_X)$ and $(X',\mu_X')$ are compact metric measure spaces and $\omega_X:X \times X' \to \R^{\geq 0}$ is a measurable, bounded function. If $\mu_X$ and $\mu_X'$ are probability measures, we refer to $\mathcal{H}_X$ as a \define{probability measure hypernetwork}. 

     For $p \in [1,\infty]$, the \define{$p$-Co-Optimal Transport distance} between probability measure hypernetworks $\mathcal{H}_X$ and $\mathcal{H}_Y$ is 
     \begin{equation}\label{def:coot_distance}
     \COT_p(\mathcal{H}_X,\mathcal{H}_Y) \coloneqq \frac{1}{2} \inf_{\pi \in \Pi(\mu_X,\mu_Y), \pi' \in \Pi(\mu_X',\mu_Y')} \|\omega_X - \omega_Y\|_{\textnormal{L}^p(\pi \otimes \pi')},
     \end{equation}
     where $\omega_X - \omega_Y$ is considered as a function $(X \times Y) \times (X' \times Y') \to \R$. 
\end{mydef}

\begin{remark}
    The \emph{measure hypernetwork} terminology is in reference to the \emph{measure network} terminology of~\cite{Chowdhury_2019}; the idea is that measure hypernetworks serve as natural, general models for weighted hypernetworks. Examples of hypernetworks which arise in practice include feature matrices (the original motivation for developing the framework in~\cite{Redko_2020}) and general cost functions, say, in the context of optimal transport (see~\cite{needham2024stability}).  
\end{remark}

\begin{remark}\label{rem:networks_as_hypernetworks}
        One can consider a probabilistic measure network $\mathcal{N}_X$ as a probabilistic measure hypernetwork $\mathcal{H}_X = (X,\mu_X,X,\mu_X,\omega_X)$. Given two such measure hypernetworks, the COT formula is almost the same as the one of GW, with the difference being that the couplings being optimized over are `de-coupled': we integrate over $\pi \otimes \pi'$ rather than $\pi \otimes \pi$. In this setting, one has $\mathsf{COT}_p(\mathcal{H}_X,\mathcal{H}_Y) \leq \mathsf{GW}_p(\mathcal{N}_X,\mathcal{N}_Y)$, in general, and sufficient conditions for equality are provided in~\cite{Redko_2020}. We consider similar questions for conic versions of these distances below.
\end{remark}

Theoretical and computational properties of $\COT_p$ were studied in \cite{Redko_2020,Chowdhury_2024}. In particular, $\COT_p$ defines a pseudometric on the space of probability measure hypernetworks, considered up to a certain equivalence relation, called \emph{weak isomorphism}. The following section extends this framework to the unbalanced setting and proves an analogous result.

\subsection{Conic Co-Optimal Transport} The distance we study in this section is defined as follows.

\begin{mydef}[Conic Co-Optimal Transport Distance]\label{def:ccot}
For measure hypernetworks $\mathcal{H}_X$ and $\mathcal{H}_Y$, $\delta$ a positive real number, and $\mathsf{d}_{\Co_\delta[\R]}$ a metric on the cone over $\R$, the associated \define{conic Co-Optimal Transport (CCOT) distance} is
\begin{equation*}
\CCOT_\delta(\mathcal{H}_X,\mathcal{H}_Y) \coloneqq \inf\limits_{(\gamma_0,\gamma_1)\in \Gamma(\mu_X,\mu_Y), (\gamma_0',\gamma_1') \in \Gamma(\mu'_X,\mu'_Y)} \mathbf{L}_\delta (\gamma_0,\gamma_1,\gamma_0',\gamma_1')^{1/2},
\end{equation*}
where:
\begin{itemize}[leftmargin=*]
 \item the \define{CCOT loss} $\mathbf{L}_\delta: \Gamma(\mu_X,\mu_Y) \times \Gamma(\mu_X',\mu_Y') \to \R$ is defined by 
\begin{equation}\label{eqn:L_delta}
\begin{split}
&\mathbf{L}_\delta (\gamma_0,\gamma_1,\gamma_0',\gamma_1') \coloneqq \\
&\quad \int_{X' \times Y' \times X \times Y} \mathsf{d}_{\Co_\delta[\R]}\left(p^{\gamma_0,\gamma_0',\gamma,\gamma'}_X(x,y,x',y'),p^{\gamma_1,\gamma_1',\gamma,\gamma'}_Y(x,y,x',y')\right)^2 \\
&\hspace{3in} \cdot \textnormal{d}\gamma(x,y)\textnormal{d}\gamma'(x',y');
\end{split}
\end{equation}
\item the functions $p^{\gamma_0,\gamma_0',\gamma,\gamma'}_X, p^{\gamma_1,\gamma_1',\gamma,\gamma'}_Y: X \times Y \times X' \times Y' \to \Co_\delta[\R]$ are defined by
\[
    p^{\gamma_0,\gamma_0',\gamma,\gamma'}_X(x,y,x',y') = \left[\omega_X(x,x'),\sqrt{\frac{\textnormal{d}\gamma_0}{\textnormal{d}\gamma}(x,y)\frac{\textnormal{d}\gamma_0'}{\textnormal{d}\gamma'}(x',y')}\right]
\]
and
\[
    p^{\gamma_1,\gamma_1',\gamma,\gamma'}_Y(x,y,x',y') = \left[\omega_Y(y,y'),\sqrt{\frac{\textnormal{d}\gamma_1}{\textnormal{d}\gamma}(x,y)\frac{\textnormal{d}\gamma_1'}{\textnormal{d}\gamma'}(x',y')}\right];
\]
    \item $\gamma_0,\gamma_1 \ll \gamma$ and $\gamma_0',\gamma_1' \ll \gamma'$ are any absolutely continuous measures on $X \times Y$ and $X' \times Y'$, respectively.
\end{itemize}
\end{mydef} 

The following lemma gives a useful alternative expression for the CCOT loss. The proof is analogous to that of Lemma \ref{lem:CGW_loss_formula}.

\begin{lemma}\label{lem:CCOOTequiv}
The CCOT loss \eqref{eqn:L_delta} can be equivalently expressed as
\begin{align*}
   & \mathbf{L}_\delta(\gamma_0,\gamma_1,\gamma_0',\gamma_1') \\
   &\qquad =4\delta^2\big(\mu_X(X)\mu_X'(X') + \mu_Y(Y)\mu_Y'(Y')\big)\\
   &\qquad \qquad  -8\delta^2\int_{X\times Y\times X' \times Y'} \Omega\left(\frac{|\omega_X - \omega_Y|}{2\delta}\right)\sqrt{\frac{\textnormal{d}\gamma_0}{\textnormal{d}\gamma}\frac{\textnormal{d}\gamma_0'}{\textnormal{d}\gamma'}\frac{\textnormal{d}\gamma_1}{\textnormal{d}\gamma}\frac{\textnormal{d}\gamma_1'}{\textnormal{d}\gamma'}}  \textnormal{d}\gamma \textnormal{d}\gamma',
\end{align*}
where $\gamma_0,\gamma_1 \ll\gamma$ and $\gamma_0',\gamma_1' \ll \gamma'$.
\end{lemma}

As in the case of measure networks, we will consider the proposed distance up to weak isomorphism of hypernetworks. This is a superficial generalization of the notion of weak isomorphism originally proposed in \cite{Chowdhury_2024}, reformulated analogously to Definition \ref{def:wisomorphism}.
\begin{mydef}[Weak Isomorphism of hypernetworks,
\cite{Chowdhury_2024}]\label{def:wisomorphism2}
Measure hypernetworks $\mathcal{H}_X$ and $\mathcal{H}_Y$ are called \define{weakly isomorphic} if there exists a pair of Polish spaces $Z$ and $Z'$ equipped with Borel measures $\mu_{Z}$ and $\mu_{Z}'$, respectively, and measurable maps $f:Z\to X$, $f':Z'\to X'$, $g:Z \to Y$, and $g':Z'\to Y'$ such that
 \begin{itemize}[leftmargin=*]
        \item $f_\#\mu_{Z}=\mu_X$,
        \item $f'_\#\mu_{Z}'=\mu_X'$,
        \item $g_\#\mu_{Z}=\mu_Y$,
        \item $g'_\#\mu_{Z}'=\mu_Y'$,
        \item and $\|(f,f')^\#\omega_X-(g,g')^\#\omega_Y\|_{\mathrm{L}^\infty(\mu_{Z} \otimes \mu_{Z}')} = 0$,
    \end{itemize}
where $(f,f')^\#\omega_X: Z\times Z'\to \R$ is the pullback function given by the map $(z,z') \mapsto \omega_X(f(z),f'(z'))$, and  $(g,g')^\#\omega_Y$ is defined analogously.
\end{mydef}

The next result says that the conic Co-Optimal Transport distance does in fact define a distance up to weak isomorphisms of measure hypernetworks. The proof is the same as that of Theorem \ref{thm:psuedometric}, up to adapting the arguments with more involved notation, so we omit it. The analogous result for co-optimal transport distance is proved in detail in~\cite[Theorem 1]{Chowdhury_2024}.

\begin{theorem}\label{thm:psuedometric2}
    For $\delta\in\R^{>0}$, $\CCOT_{\delta}$ is a pseudometric on the space of measure hypernetworks and, in particular, $\CCOT_{\delta}(\mathcal{H}_X,\mathcal{H}_Y)=0$ if and only if $\mathcal{H}_X$ is weakly isomorphic to $\mathcal{H}_Y$.
\end{theorem}

In particular, the proof of Theorem \ref{thm:psuedometric2} would make use of the following lemma, which is analogous to Lemma \ref{lemma:lsc_functional} and which is proved via similar arguments. We include the statement of the lemma here, because it will be used in other arguments below.

\begin{lemma}\label{lem:CCOT_realized}
    The infimum defining $\CCOT_\delta$ is always realized. 
\end{lemma}

\subsection{Relationship with UCOT}
We now examine the relationship between the conic Co-Optimal Transport (CCOT) distance and the Unbalanced Co-Optimal transport (UCOT) introduced in \cite{UCOOT}, whose definition we restate in the context of hypernetworks below.

\begin{mydef}[Unbalanced Co-Optimal Transport \cite{UCOOT}]  \label{def:UCOOT} 
The \define{Unbalanced Co-Optimal transport (UCOT) distance} between  hypernetworks $\mathcal{H}_X$ and $\mathcal{H}_Y$ is defined as 
    \begin{align*}
    \UCOT(\mathcal{H}_X,\mathcal{H}_Y)^2 &\coloneqq \inf_{\gamma\in\Mm(X \times Y),\gamma'\in\Mm(X' \times Y')}\int_{X' \times Y' \times X \times Y } |\omega_X -\omega_Y|^2\operatorname{d}\gamma\operatorname{d}\gamma' \\
    &\qquad \qquad +\operatorname{KL}((\operatorname{Pr}_0)_\#\gamma\otimes(\operatorname{Pr}_0)_\#\gamma' | \mu_X\otimes\mu_X') \\
    &\qquad \qquad \qquad + \operatorname{KL}((\operatorname{Pr}_1)_\#\gamma\otimes(\operatorname{Pr}_1)_\#\gamma' | \mu_Y\otimes\mu_Y'),
    \end{align*}
    where $\operatorname{KL}(\cdot \mid \cdot)$ denotes Kullback-Leibler divergence.
\end{mydef}
We can now state a comparison result between UCOT and CCOT. This result mirrors the relationship between the $\CGW$ and the \emph{unbalanced Gromov-Wasserstein distance}, defined similarly to UCOT in \cite{NEURIPS2021}.  The  proof follows a similar argument. 

\begin{proposition}
    Let $\mathcal{H}_X$ and $\mathcal{H}_Y$ be measure hypernetworks and let the function defining $\mathsf{d}_{\Co[\R]}$ be $\Omega(z) = \exp(-z^2/2)$. Then 
    \[
    \UCOT(\mathcal{H}_X,\mathcal{H}_Y)\geq \CCOT_{1/2}(\mathcal{H}_X,\mathcal{H}_Y).
    \]
\end{proposition}

\begin{proof}
Let $\gamma\in\Mm(X\times Y)$ and $\gamma'\in\Mm(X'\times Y')$ attain the infimum defined by $\UCOT(\mathcal{H}_X,\mathcal{H}_Y)$. We note that the existence of these minimizers is guaranteed by adapting the argument of~\cite[Proposition 2]{NEURIPS2021} to the co-optimal transport setting. We define densities via Lebesgue decompositions as
\begin{align*}
    \mu_X = f (\operatorname{Pr}_0)_\# \gamma + \mu_X^\perp &\qquad \mu_X' = f' (\operatorname{Pr}_0)_\# \gamma' + (\mu_X')^\perp \\
    \mu_Y = g (\operatorname{Pr}_1)_\# \gamma + \mu_Y^\perp &\qquad \mu_Y' = g' (\operatorname{Pr}_1)_\# \gamma' + (\mu_Y')^\perp.
\end{align*}
Using these functions we can define semi-couplings $(\gamma_0,\gamma_1)\in \Gamma(\mu_X,\mu_Y)$ and $(\gamma_0',\gamma_1')\in \Gamma(\mu_X',\mu_Y')$ by, for any $U\subseteq X\times Y$ and $U'\subseteq X'\times Y'$,
\begin{align*}
    \gamma_0(U) &=\int_{U} f(x) \textnormal{d}\gamma(x,y) + (\mu_X^\perp \otimes \overline{\mu}_Y)(U)\\ \gamma_0'(U') &=\int_{U'} f'(x') \textnormal{d}\gamma'(x',y')+ ((\mu_X')^\perp \otimes \overline{\mu}_Y')(U') \\
    \gamma_1(U) &= \int_{U} g(y) \textnormal{d}\gamma(x,y) + (\overline{\mu}_X \otimes \mu_Y^\perp)(U)\\
    \gamma_1'(U')&= \int_{U'} g'(y') \textnormal{d}\gamma'(x',y')+ (\overline{\mu}_X' \otimes (\mu_Y')^\perp)(U'),
\end{align*}
where $\overline{\mu}_X$, $\overline{\mu}_X'$, $\overline{\mu}_Y$, and $\overline{\mu}_Y'$ are arbitrary choices of probability measures on the appropriate spaces. Then 
\[
\mu_X \otimes \mu_X' = ff' (\mathrm{Pr}_0)_\# \gamma \otimes (\mathrm{Pr_0})_\#\gamma' + (\mu_X \otimes \mu_X')^\perp,
\]
so that 
\begin{align*}
&\operatorname{KL}((\mathrm{Pr}_0)_\# \gamma \otimes (\mathrm{Pr_0})_\#\gamma' \mid \mu_X \otimes \mu_X') \\
& \qquad =\int_{X\times X'} \log \frac{1}{ff'} - 1 + f f' \; \mathrm{d}((\mathrm{Pr}_0)_\# \gamma \otimes (\mathrm{Pr_0})_\#\gamma')  + (\mu_X \otimes \mu_X')^\perp (X \times X')\\
& \qquad =\int_{X'\times Y'} \int_{X \times Y} \log \frac{1}{ff'} - 1 + f f' \; \mathrm{d}\gamma \mathrm{d}\gamma' + (\mu_X \otimes \mu_X')^\perp (X \times X'),
\end{align*}
with a similar expression holding for $\operatorname{KL}((\mathrm{Pr}_1)_\# \gamma \otimes (\mathrm{Pr_1})_\#\gamma' \mid \mu_Y \otimes \mu_Y')$. Applying the elementary inequality
\[
z^2+\log\frac1A-1+A+\log\frac1B-1+B
\geq
A+B-2\sqrt{AB}\,e^{-z^2/2}
\]
pointwise to the integrands, with $z = |\omega_X - \omega_Y|$, $A = ff'$, $B = gg'$ (for $A,B > 0$, and using the standard extended-valued convention where necessary), then gives
\begin{align}
    \mathsf{UCOT}(\mathcal{H}_X,\mathcal{H}_Y)^2 &\geq \iint ff' \mathrm{d}\gamma \mathrm{d}\gamma' + \iint gg' \,\mathrm{d}\gamma \mathrm{d}\gamma' \nonumber \\
    &\qquad - 2 \iint \exp\left(-\frac{|\omega_X - \omega_Y|^2}{2}\right) \sqrt{ff'gg'} \,\mathrm{d}\gamma \mathrm{d}\gamma' \nonumber \\
    &\qquad \qquad + (\mu_X \otimes \mu_X')^\perp(X \times X') + (\mu_Y \otimes \mu_Y')^\perp (Y \times Y') \nonumber \\
    &= \mu_X(X)\mu_X'(X') + \mu_Y(Y)\mu_Y'(Y') \label{eqn:UCOT_bound_1} \\
    &\qquad - 2 \iint \exp\left(-\frac{|\omega_X - \omega_Y|^2}{2}\right) \sqrt{ff'gg'} \,\mathrm{d}\gamma \mathrm{d}\gamma' \nonumber \\
    &\geq \mathbf{L}_{1/2}(\gamma_0,\gamma_1,\gamma_0',\gamma_1') \label{eqn:UCOT_bound_2} \\
    &\geq \mathsf{CCOT}_{1/2}(\mathcal{H}_X,\mathcal{H}_Y)^2, \nonumber 
\end{align}
where \eqref{eqn:UCOT_bound_1} uses the fact that 
\[
\iint ff' \mathrm{d}\gamma \mathrm{d}\gamma' + (\mu_X \otimes \mu_X')^\perp(X \times X') = \mu_X(X)\mu_X'(X'),
\]
plus a similar identity for $\mathcal{H}_Y$, and where \eqref{eqn:UCOT_bound_2} uses Lemma \ref{lem:CCOOTequiv} plus the fact that we have an upper bound by only evaluating continuous parts of the measures.
\end{proof}

\subsection{Relationship with CGW Distances}\label{sec:CCOOT=CGW}
Consider measure networks $\mathcal{N}_X=(X,\mu_X,\omega_X)$ and observe that we may construct a hypernetwork from $\mathcal{N}_X$ as follows $\mathcal{H}_X=(X,\mu_X,X,\mu_X,\omega_X)$. In the following section we will investigate the relationship between $\CCOT(\mathcal{H}_X,\mathcal{H}_Y)$ and $\CGW(\mathcal{N}_X,\mathcal{N}_Y)$ (cf.~Remark~\ref{rem:networks_as_hypernetworks}). First, we observe that conic Co-Optimal Transport lower bounds the conic Gromov-Wasserstein distance, in general.

\begin{proposition}\label{prop:CCOT_less_than_CGW}
    Let $\mathcal{N}_X,\mathcal{N}_Y$ be measure networks with associated measure hypernetworks $\mathcal{H}_X,\mathcal{H}_Y$. Then
    \begin{equation*}
        \CCOT_\delta(\mathcal{H}_X,\mathcal{H}_Y)\leq\CGW_\delta(\mathcal{N}_X,\mathcal{N}_Y). 
    \end{equation*}
\end{proposition}

\begin{proof}
    Let $(\gamma_0,\gamma_1)\in\Gamma(\mu_X,\mu_Y)$ and note that $\mathbf{H}_\delta(\gamma_0,\gamma_1) = \mathbf{L}_\delta(\gamma_0,\gamma_1,\gamma_0,\gamma_1)$. Therefore, \begin{align*}
        \inf\limits_{(\gamma_0,\gamma_1)\in \Gamma(\mu_X,\mu_Y)} \mathbf{H}_\delta(\gamma_0,\gamma_1) &= \inf\limits_{(\gamma_0,\gamma_1)\in \Gamma(\mu_X,\mu_Y)} \mathbf{L}_\delta(\gamma_0,\gamma_1,\gamma_0,\gamma_1)\\
        &\geq\inf\limits_{(\gamma_0,\gamma_1),(\gamma_0',\gamma_1')\in \Gamma(\mu_X,\mu_Y)}
         \mathbf{L}_\delta(\gamma_0,\gamma_1,\gamma_0',\gamma_1').
    \end{align*}
\end{proof}

We now provide sufficient conditions under which this lower bound becomes an equality, i.e., when CCOT and CGW coincide. 

\begin{theorem}\label{thm:CGW_CCOOT_equivalence}
Let $\mathcal{N}_X,\mathcal{N}_Y$ be measure networks and consider the associated hypernetworks, $\mathcal{H}_X,\mathcal{H}_Y$ respectively. Suppose that $\omega_X$, $\omega_Y$, and $\Omega$ have the property that the map $k:(X \times Y) \times (X \times Y) \to \R$ given by 
\[
k \coloneqq \Omega\left(\frac{|\omega_X-\omega_Y|}{2\delta}\right)
\]
defines a positive-definite kernel. Then for $(\gamma_0,\gamma_1)$, $(\gamma_0',\gamma_1')\in\Gamma(\mu_X,\mu_Y)$ such that $\CCOT_\delta(\mathcal{H}_X,\mathcal{H}_Y) = \mathbf{L}_\delta(\gamma_0,\gamma_1,\gamma_0',\gamma_1')$, we have
\begin{equation*}
\CCOT_\delta(\mathcal{H}_X,\mathcal{H}_Y)^2= \mathbf{L}_\delta(\gamma_0,\gamma_1,\gamma_0,\gamma_1) = \mathbf{L}_\delta(\gamma_0',\gamma_1',\gamma_0',\gamma_1').
\end{equation*}
It follows that 
\begin{equation*}
    \CCOT_\delta(\mathcal{H}_X,\mathcal{H}_Y)^2= \mathbf{H}_{\delta}(\gamma_0,\gamma_1)= \mathbf{H}_{\delta}(\gamma_0',\gamma_1') = \CGW(\mathcal{N}_X,\mathcal{N}_Y)^2.
\end{equation*}
\end{theorem}

We note that this result is closely related to \cite[Theorem 3]{NEURIPS2021}, which is stated in terms of a bi-convex relaxation of CGW distance, rather than co-optimal transport. One approach to proving the theorem would be to show that the CCOT distance agrees with the bi-convex relaxation of CGW presented in \cite{NEURIPS2021} when the measure hypernetworks $\mathcal{H}_X$ and $\mathcal{H}_Y$ are induced by metric measure spaces (this would be analogous to Proposition \ref{prop:CGW_semi_coupling}), and to then apply \cite[Theorem 3]{NEURIPS2021}, checking that the arguments do not depend on the assumption that the functions $\omega_X$ and $\omega_Y$ are metrics. We instead give a more direct proof below.

\begin{proof}
Assume that $(\gamma_0,\gamma_1),(\gamma_0',\gamma_1') \in\Gamma(\mu_X,\mu_Y)$ such that $\CCOT_\delta(\mathcal{H}_X,\mathcal{H}_Y) = \mathbf{L}_\delta(\gamma_0,\gamma_1,\gamma_0',\gamma_1')$. By optimality, we have
\begin{equation*}    \mathbf{L}_\delta(\gamma_0,\gamma_1,\gamma_0',\gamma_1')\leq \mathbf{L}_\delta(\gamma_0,\gamma_1,\gamma_0,\gamma_1) \quad \text{and} \quad \mathbf{L}_\delta(\gamma_0,\gamma_1,\gamma_0',\gamma_1')\leq \mathbf{L}_\delta(\gamma_0',\gamma_1',\gamma_0',\gamma_1'),
\end{equation*}
so that
\begin{equation}\label{eqn:negative_definite1}   \mathbf{L}_\delta(\gamma_0',\gamma_1',\gamma_0',\gamma_1') + \mathbf{L}_\delta(\gamma_0,\gamma_1,\gamma_0,\gamma_1) - 2\mathbf{L}_\delta(\gamma_0,\gamma_1,\gamma_0',\gamma_1')\geq 0.
\end{equation}
Choosing the reference measure $\gamma$ so that $\gamma_0,\gamma_1,\gamma_0',\gamma_1' \ll\gamma$, it follows from the expression for $\mathbf{L}_\delta$ given in Lemma \ref{lem:CCOOTequiv} and some simplification that \eqref{eqn:negative_definite1} implies
\begin{align*}
    &-8\delta^2 \left(\int_{X\times Y\times X \times Y} \Omega\left(\frac{|\omega_X - \omega_Y|}{2\delta}\right)\sqrt{\frac{\textnormal{d}\gamma_0'}{\textnormal{d}\gamma}\frac{\textnormal{d}\gamma_0'}{\textnormal{d}\gamma}\frac{\textnormal{d}\gamma_1'}{\textnormal{d}\gamma}\frac{\textnormal{d}\gamma_1'}{\textnormal{d}\gamma'}}  \textnormal{d}\gamma \textnormal{d}\gamma\right. \\
    &\qquad + \int_{X\times Y\times X \times Y} \Omega\left(\frac{|\omega_X - \omega_Y|}{2\delta}\right)\sqrt{\frac{\textnormal{d}\gamma_0}{\textnormal{d}\gamma}\frac{\textnormal{d}\gamma_0}{\textnormal{d}\gamma}\frac{\textnormal{d}\gamma_1}{\textnormal{d}\gamma}\frac{\textnormal{d}\gamma_1}{\textnormal{d}\gamma}}  \textnormal{d}\gamma \textnormal{d}\gamma \\
    & \qquad \qquad \left.\qquad\qquad-2\int_{X\times Y\times X \times Y} \Omega\left(\frac{|\omega_X - \omega_Y|}{2\delta}\right)\sqrt{\frac{\textnormal{d}\gamma_0}{\textnormal{d}\gamma}\frac{\textnormal{d}\gamma_0'}{\textnormal{d}\gamma}\frac{\textnormal{d}\gamma_1}{\textnormal{d}\gamma}\frac{\textnormal{d}\gamma_1'}{\textnormal{d}\gamma}}  \textnormal{d}\gamma \textnormal{d}\gamma\right)\geq 0.
\end{align*}
Therefore,
\begin{align*}
    &\int_{X\times Y\times X \times Y} \Omega\left(\frac{|\omega_X - \omega_Y|}{2\delta}\right)\sqrt{\frac{\textnormal{d}\gamma_0'}{\textnormal{d}\gamma}\frac{\textnormal{d}\gamma_0'}{\textnormal{d}\gamma}\frac{\textnormal{d}\gamma_1'}{\textnormal{d}\gamma}\frac{\textnormal{d}\gamma_1'}{\textnormal{d}\gamma'}}  \textnormal{d}\gamma \textnormal{d}\gamma \\
    & \qquad+ \int_{X\times Y\times X \times Y} \Omega\left(\frac{|\omega_X - \omega_Y|}{2\delta}\right)\sqrt{\frac{\textnormal{d}\gamma_0}{\textnormal{d}\gamma}\frac{\textnormal{d}\gamma_0}{\textnormal{d}\gamma}\frac{\textnormal{d}\gamma_1}{\textnormal{d}\gamma}\frac{\textnormal{d}\gamma_1}{\textnormal{d}\gamma}}  \textnormal{d}\gamma \textnormal{d}\gamma \\
    &\qquad\qquad-2\int_{X\times Y\times X \times Y} \Omega\left(\frac{|\omega_X - \omega_Y|}{2\delta}\right)\sqrt{\frac{\textnormal{d}\gamma_0}{\textnormal{d}\gamma}\frac{\textnormal{d}\gamma_0'}{\textnormal{d}\gamma}\frac{\textnormal{d}\gamma_1}{\textnormal{d}\gamma}\frac{\textnormal{d}\gamma_1'}{\textnormal{d}\gamma}}  \textnormal{d}\gamma \textnormal{d}\gamma\leq 0.
\end{align*}
Substituting $k=\Omega\left(\frac{|\omega_X - \omega_Y|}{2\delta}\right)$ and factoring, we obtain

\begin{equation}\label{eqn:negative_definite2}
\begin{aligned}
    \int_{X\times Y} \int_{X\times Y} k \left(\sqrt{\frac{\textnormal{d}\gamma_0}{\textnormal{d}\gamma}\frac{\textnormal{d}\gamma_1}{\textnormal{d}\gamma}} - \sqrt{\frac{\textnormal{d}\gamma_0'}{\textnormal{d}\gamma}\frac{\textnormal{d}\gamma_1'}{\textnormal{d}\gamma}}\right) \textnormal{d}\gamma\left(\sqrt{\frac{\textnormal{d}\gamma_0}{\textnormal{d}\gamma}\frac{\textnormal{d}\gamma_1}{\textnormal{d}\gamma}} - \sqrt{\frac{\textnormal{d}\gamma_0'}{\textnormal{d}\gamma}\frac{\textnormal{d}\gamma_1'}{\textnormal{d}\gamma}}\right) \textnormal{d}\gamma
    \leq 0.
\end{aligned}
\end{equation}
Since $k$ is positive-definite, this implies
\begin{equation}\label{eqn:negative_definite3}
\begin{aligned}
    \int_{X\times Y}  \int_{X\times Y}  k \left(\sqrt{\frac{\textnormal{d}\gamma_0}{\textnormal{d}\gamma}\frac{\textnormal{d}\gamma_1}{\textnormal{d}\gamma}} - \sqrt{\frac{\textnormal{d}\gamma_0'}{\textnormal{d}\gamma}\frac{\textnormal{d}\gamma_1'}{\textnormal{d}\gamma}}\right) \textnormal{d}\gamma
    \left(\sqrt{\frac{\textnormal{d}\gamma_0}{\textnormal{d}\gamma}\frac{\textnormal{d}\gamma_1}{\textnormal{d}\gamma}} - \sqrt{\frac{\textnormal{d}\gamma_0'}{\textnormal{d}\gamma}\frac{\textnormal{d}\gamma_1'}{\textnormal{d}\gamma}}\right) \textnormal{d}\gamma
    = 0,\\
\end{aligned}
\end{equation}
and, in turn, that
\begin{equation*}
    \sqrt{\frac{\textnormal{d}\gamma_0}{\textnormal{d}\gamma}\frac{\textnormal{d}\gamma_1}{\textnormal{d}\gamma}} =\sqrt{\frac{\textnormal{d}\gamma_0'}{\textnormal{d}\gamma}\frac{\textnormal{d}\gamma_1'}{\textnormal{d}\gamma}}.
\end{equation*}
Therefore,
\begin{equation*}    \mathbf{L}_\delta(\gamma_0,\gamma_1,\gamma_0',\gamma_1') = \mathbf{L}_\delta(\gamma_0,\gamma_1,\gamma_0,\gamma_1) = \mathbf{L}_\delta(\gamma_0',\gamma_1',\gamma_0',\gamma_1').
\end{equation*}

The work above, plus the fact that the infimum in the definition of $\CCOT_\delta$ is always realized (Lemma \ref{lem:CCOT_realized}), proves the second statement of the theorem. That is, suppose $\CCOT(\mathcal{H}_X,\mathcal{H}_Y)$ is realized by $(\gamma_0,\gamma_1)$ and $(\gamma_0',\gamma_1')$. Then 
\begin{align*}
\CCOT_\delta(\mathcal{H}_X,\mathcal{H}_Y) &= \mathbf{L}_\delta(\gamma_0,\gamma_1,\gamma_0',\gamma_1') \\
&= \mathbf{L}_\delta(\gamma_0,\gamma_1,\gamma_0,\gamma_1) \\
&= \mathbf{H}_\delta(\gamma_0,\gamma_1) \geq \CGW(\mathcal{N}_X,\mathcal{N}_Y).
\end{align*}
The reverse inequality always holds, by Proposition \ref{prop:CCOT_less_than_CGW}, so this completes the proof.
\end{proof}

For examples of kernels satisfying the hypotheses of the theorem, see~\cite{NEURIPS2021}; in particular, the discussion surrounding Theorem 3 and Proposition 11 therein.

\subsection{Algorithm for CCOT}\label{section:algorithm}
In this section, we examine the conic Co-Optimal Transport (CCOT) problem between discrete hypernetworks and develop a computational method for finding the optimal pair of semi-couplings. Our method closely follows that of \cite{bauer2022SRNF} (in the context of the WFR distance), employing a cyclic block coordinate ascent algorithm to compute optimal semi-couplings, with a closed form for the optimal solution on each block while keeping the other blocks fixed. To consider this algorithm we first require the notion of a discrete semi-coupling from \cite[Definition 2.4]{bauer2022SRNF}.

\begin{mydef}[Discrete Semi-Couplings \cite{bauer2022SRNF}]\label{def:discrete_semicoupling}
Let $(X,\mu_X), (Y,\mu_Y)$ be two measure spaces with finitely supported measures expressed as $\mu_X=\sum_{i=1}^m a_i \delta_{x_i}$ and $\mu_Y=\sum_{j=1}^n b_j\delta_{y_j}$. A \define{discrete semi-coupling} of $\mu_X$ and $\mu_Y$ is a pair of $m\times n$ matrices $(A,B)$ satisfying the properties:
\begin{itemize}[leftmargin=*]
\item for all $i,j$, $A_{ij}\geq 0$ and $B_{ij}\geq0$;\label{property_a}
\item for $1\leq i\leq m$, $a_i\geq\sum_{j=1}^{n}A_{ij}$;\label{property_b}
\item for $1\leq j \leq n$, $b_j\geq \sum_{i=1}^{m}B_{ij}$.\label{property_c}
\end{itemize}
We denote the set of all discrete semi-couplings of $\mu_X$ and $\mu_Y$ by $\Gamma_\mathrm{d}(\mu_X,\mu_Y)$.
\end{mydef}

\begin{remark}\label{rem:proper_subset}
    The space of matrix pairs $\Gamma_\mathrm{d}(\mu_X,\mu_Y)$ naturally corresponds  to a subset of the space of measure pairs $\Gamma(\mu_X,\mu_Y)$, justifying our notation. In general, this subset is proper: even though $\mu_X$ and $\mu_Y$ are finitely supported, measure couplings $(\gamma_0,\gamma_1)$ can be non-finitely supported. This definition for discrete semi-couplings differs slightly from that of \cite{bauer2022SRNF} by allowing for inequality constraints and dropping the 0th rows and columns of the discrete semi-couplings. In the definition of \cite{bauer2022SRNF}, the 0-th rows and columns serve as dummy supports for the creation and destruction of mass. By contrast, the inequality relaxation in Definition \ref{def:discrete_semicoupling} allows for the creation and destruction of mass with a slightly simpler formulation.
\end{remark}

In light of Remark \ref{rem:proper_subset}, it is not apparent that the CCOT distance can be written as an optimization over the space of discrete semi-couplings, even between two measure hypernetworks whose measures are finitely supported. Our first result of this section will be to show that in fact, it is sufficient to optimize over these spaces.
\begin{proposition}\label{prop:discrete_CCOT}
For measure hypernetworks $\mathcal{H}_X$ and $\mathcal{H}_Y$ with finitely supported measures
\begin{equation*}
\mu_X=\sum_{i=1}^n a_i \delta_{x_i},\,
\mu_X'=\sum_{j=1}^{n'}a_j' \delta_{x_j'},\,
\mu_Y=\sum_{k=1}^{m} b_k \delta_{y_k},\,
\mu_Y'=\sum_{l=1}^{m'} b_l' \delta_{y_l'},
\end{equation*}
the CCOT distance between them is given by    
\begin{align*}
    &\CCOT_\delta(\mathcal{H}_X,\mathcal{H}_Y)^2 \\
    & \qquad = \min_{A,B,A',B'}  4\delta^2\Bigg( \mu_X(X)\mu_X'(X') +\mu_Y(Y)\mu_Y'(Y') \\
    & \hspace{2in} -2\sum_{i=1}^{n}\sum_{j=1}^{n'}\sum_{k=1}^{m}\sum_{l=1}^{m'}   \Omega_{ijkl}^\delta\sqrt{A_{ik}B_{ik}A'_{jl}B'_{jl}}\Bigg),
\end{align*}
where the minimum is over $(A,B)\in \Gamma_\mathrm{d}(\mu_X,\mu_Y), (A',B')\in \Gamma_\mathrm{d}(\mu_X',\mu_Y')$ and  the distortion cost tensor is defined by $\Omega^\delta_{ijkl} = \Omega\left(\frac{|\omega(x_i,x_j')-\omega'(y_k,y_l')|}{2\delta}\right)$. 
\end{proposition}
\begin{proof}
Let $(\gamma_0,\gamma_1)\in\Gamma(\mu_X,\mu_Y)$ and $(\gamma_0',\gamma_1')\in\Gamma(\mu_X',\mu_Y')$ and let $\gamma_0,\gamma_1 \ll \gamma$ and $\gamma'_0,\gamma'_1 \ll \gamma'$.
Construct $(A,B)$ and $(A',B')$ as follows: 
\[
A_{ik} = \frac{\mathrm{d}\gamma_0}{\mathrm{d}\gamma}(x_i,y_k), \, B_{ik} = \frac{\mathrm{d}\gamma_1}{\mathrm{d}\gamma}(x_i,y_k), \, A'_{jl} = \frac{\mathrm{d}\gamma_0'}{\mathrm{d}\gamma'}(x_j',y_l'), \, B'_{jl} = \frac{\mathrm{d}\gamma_1'}{\mathrm{d}\gamma'}(x_j',y_l').
\]
Then it is easy to verify that the quantity $\mathbf{L}_\delta(\gamma_0,\gamma_1,\gamma_0',\gamma_1')/(4\delta^2)$ is given by 
\[
\mu_X(X)\mu_X'(X') + \mu_Y(Y)\mu_{Y}'(Y')-2\sum_{i=1}^{n}\sum_{j=1}^{n'}\sum_{k=1}^{m}\sum_{l=1}^{m'}   \Omega_{ijkl}^\delta\sqrt{A_{ik}B_{ik}A'_{jl}B'_{jl}}.
\]
The result follows by Lemma \ref{lem:CCOOTequiv}.
\end{proof}

The following computationally convenient fact is immediate.

\begin{corollary} In the case of finitely supported measures, computing $\CCOT_\delta$ is equivalent to finding the maximizer of the function
\begin{equation}\label{eqn:CCOT_functional}
\begin{split}
    F: \Gamma_\mathrm{d}(\mu_X,\mu_Y) \times \Gamma_\mathrm{d}(\mu_X',\mu_Y')&\to \R\\
    (A,B,A',B')&\mapsto \sum_{i=1}^{n}\sum_{j=1}^{n'}\sum_{k=1}^{m}\sum_{l=1}^{m'}   \Omega^\delta_{ijkl}\sqrt{A_{ik}B_{ik}A'_{jl}B'_{jl}}.
\end{split}
\end{equation}
\end{corollary}

We propose to approximate the CCOT distance via block coordinate ascent, that is, by iteratively maximizing functions of the form 
\begin{equation}\label{eqn:BCA_functions}
F(-,B,A',B'), \, F(A,-,A',B'), \, F(A,B,-,B'), \, F(A,B,A',-)
\end{equation}
(i.e., holding all but one argument constant), where $F$ is as in \eqref{eqn:CCOT_functional}. The rest of this subsection is devoted to technical results which show that this is well posed and numerically feasible. These results all assume finitely-supported hypernetworks with the same notation as Proposition \ref{prop:discrete_CCOT}.

\begin{lemma}\label{def:tuple_constraints}
Consider the subset $\overline{\Gamma}_\mathrm{d}$ of $\Gamma_\mathrm{d}(\mu_X,\mu_X') \times \Gamma_\mathrm{d}(\mu_Y,\mu_Y')$ consisting of pairs of discrete semi-couplings $(A,B)$ and $(A',B')$ satisfying the conditions: 
\begin{enumerate}
  \item $A_{ik}=B_{ik}=0$ whenever
\[
\begin{aligned}
\sum_{j=1}^{n'}\sum_{l=1}^{m'} \Omega^\delta_{ijkl} = 0
\end{aligned}
\]

\item $A'_{jl}=B'_{jl}=0$ whenever
\[
\begin{aligned}
\sum_{i=1}^{n}\sum_{k=1}^{m} \Omega^\delta_{ijkl} = 0
\end{aligned}
\]
    \item for $i=1,\ldots, n$, \[\sum_{k=1}^{m}A_{ik}=\begin{cases}
        0 & \text{ if } \sum_{i=1}^{n}\sum_{j=1}^{n'}\sum_{l=1}^{m'}   \Omega^\delta_{ijkl} = 0\\
        a_i & \text{ otherwise.}\end{cases}
        \]
    \item for $k=1,\ldots, m$, \[\sum_{i=1}^{n}B_{ik}=\begin{cases}
        0 & \text{ if } \sum_{j=1}^{n'}\sum_{k=1}^{m}\sum_{l=1}^{m'}   \Omega^\delta_{ijkl} = 0\\
        b_k & \text{ otherwise.}\end{cases}
        \]
    \item for $j=1,\ldots, n'$, \[\sum_{l=1}^{m'}A'_{jl}=\begin{cases}
        0 & \text{ if } \sum_{i=1}^{n}\sum_{j=1}^{n'}\sum_{k=1}^{m}   \Omega^\delta_{ijkl} = 0\\
        a'_j & \text{ otherwise.}\end{cases}
        \]
    \item for $l=1,\ldots, m'$, \[\sum_{j=1}^{m}B'_{jl}=\begin{cases}
        0 & \text{ if } \sum_{i=1}^{n}\sum_{k=1}^{m}\sum_{l=1}^{m'}   \Omega^\delta_{ijkl} = 0\\
        b'_l & \text{ otherwise.}\end{cases}
        \]
\end{enumerate}
The function $F$ \eqref{eqn:CCOT_functional} obtains its maximum on $\overline{\Gamma}_\mathrm{d}$.
\end{lemma}
\begin{proof}
Let $(A,B)\in \Gamma_\mathrm{d}(\mu_X,\mu_X')$ and $(A',B') \in \Gamma_\mathrm{d}(\mu_Y,\mu_Y')$ be arbitrary discrete semi-couplings. We perform a two step construction of $\big((\hat A,\hat B),(\hat A',\hat B')\big)\in \overline{\Gamma}_\mathrm{d}$ such that $F(\hat A,\hat B,\hat A',\hat B')>F(A,B,A',B')$. For the first step, we construct $\big((\Tilde{A},\Tilde{B}),(\Tilde{A}',\Tilde{B}')\big)\in \Gamma_\mathrm{d}(\mu_X,\mu_Y) \times \Gamma_\mathrm{d}(\mu_X',\mu_Y')$ as follows:
\begin{itemize}[leftmargin=*]
\item For $i = 1, \dots, n$ and $k = 1, \dots, m$, define:
\begin{align*}
\Tilde{A}_{ik} &= 
\begin{cases}
0 & \text{if } \sum_{j=1}^{n'} \sum_{l=1}^{m'} \Omega^\delta_{ijkl} = 0 \\
A_{ik} & \text{otherwise}
\end{cases}
\qquad \text{and} \nonumber \\
\Tilde{B}_{ik} &=
\begin{cases}
0 & \text{if } \sum_{j=1}^{n'} \sum_{l=1}^{m'} \Omega^\delta_{ijkl} = 0 \\
B_{ik} & \text{otherwise}
\end{cases}.
\end{align*}
\item For $j = 1, \dots, n'$ and $l = 1, \dots, m'$, define:
\begin{align*}
\Tilde{A}'_{jl} &= 
\begin{cases}
0 & \text{if } \sum_{i=1}^{n} \sum_{k=1}^{m} \Omega^\delta_{ijkl} = 0 \\
A'_{jl} & \text{otherwise}
\end{cases}
\qquad \text{and} \nonumber \\
\Tilde{B}'_{jl} &= 
\begin{cases}
0 & \text{if } \sum_{i=1}^{n} \sum_{k=1}^{m} \Omega^\delta_{ijkl} = 0 \\
B'_{jl} & \text{otherwise}
\end{cases}.
\end{align*}
\end{itemize}
For the second step, we construct $(\hat A,\hat B), (\hat A',\hat B')$ by:
\begin{itemize}[leftmargin=*]
    \item For $i = 1, \dots, n$ and $k = 1, \dots, m$, define:
    \begin{align*}
    \hat{A}_{ik} &=
    \begin{cases}
    \dfrac{a_i \Tilde{A}_{ik}}{\sum_{r=1}^{m} \Tilde{A}_{ir}} & \text{if } \sum_{r=1}^{m} \Tilde{A}_{ir} > 0 \\
    \Tilde{A}_{ik} & \text{otherwise}
    \end{cases}
    \qquad \text{and} \nonumber \\
    \hat{B}_{ik} &=
    \begin{cases}
    \dfrac{b_k \Tilde{B}_{ik}}{\sum_{r=1}^{n} \Tilde{B}_{rk}} & \text{if } \sum_{r=1}^{n} \Tilde{B}_{rk} > 0 \\
    \Tilde{B}_{ik} & \text{otherwise}
    \end{cases}.
    \end{align*}

    \item For $j = 1, \dots, n'$ and $l = 1, \dots, m'$, define:
    \begin{align*}
    \hat{A}'_{jl} &=
    \begin{cases}
    \dfrac{a'_j \Tilde{A}'_{jl}}{\sum_{r=1}^{m'} \Tilde{A}'_{jr}} & \text{if } \sum_{r=1}^{m'} \Tilde{A}'_{jr} > 0 \\
    \Tilde{A}'_{jl} & \text{otherwise}
    \end{cases}
    \qquad \text{and} \nonumber \\
    \hat{B}'_{jl} &=
    \begin{cases}
    \dfrac{b'_l \Tilde{B}'_{jl}}{\sum_{r=1}^{n'} \Tilde{B}'_{rl}} & \text{if } \sum_{r=1}^{n'} \Tilde{B}'_{rl} > 0 \\
    \Tilde{B}'_{jl} & \text{otherwise}
    \end{cases}.
    \end{align*}
\end{itemize}

By construction, $(\hat A,\hat B,\hat A',\hat B')\in \overline{\Gamma}_\mathrm{d}$ and $F(\hat A,\hat B,\hat A',\hat B')\geq F(A,B,A',B')$. Thus $F$ obtains its maximum on the subset $\overline{\Gamma}_\mathrm{d}$.
\end{proof}
Our computational approach for determining the CCOT distance employs a straightforward cyclic block coordinate ascent method. However, the key to our method is that when three of the blocks are fixed, the optimal solution for the remaining block can be expressed in a closed-form. In the following lemma, we demonstrate how to compute this optimal solution for each block while keeping the others fixed.

\begin{lemma}\label{lem:maximizers}
     The block coordinate ascent functions \eqref{eqn:BCA_functions} have unique maximizers on the relative  interior of $\overline{\Gamma}_\mathrm{d}$, which are  described explicitly as follows. For $\big((A,B),(A',B')\big)$ in the interior $\overline{\Gamma}_\mathrm{d}$, the $n\times m$ matrix, $P$, is defined by \begin{equation*}
        P_{ik} = \sum_{j=1}^{n'}\sum_{l=1}^{m'}\Omega^\delta_{ijkl}\sqrt{A'_{jl}B'_{jl}}
    \end{equation*}
    and the $n'\times m'$ matrix, $Q$, is defined by
    \begin{equation*}
        Q_{jl} = \sum_{i=1}^{n}\sum_{k=1}^{m}\Omega^\delta_{ijkl}\sqrt{A_{ik}B_{ik}}.
    \end{equation*}
 The functions \eqref{eqn:BCA_functions} satisfy:
\begin{align*}
\textbf{1.} \qquad &  \arg\max F(-, B, A', B')= E \\
& \text{where } E_{ik} =
\begin{cases}
a_i \dfrac{B_{ik} \left(P_{ik}\right)^2}{\sum_{r=1}^{m} B_{ir} \left(P_{ir}\right)^2} & \text{if } \sum_{r=1}^{m} B_{ir} \left(P_{ir}\right)^2 > 0 \\
0 & \text{otherwise}
\end{cases} \\[1em]
\textbf{2.} \qquad &  \arg\max F(A, -, A', B')= E \\
& \text{where } E_{ik} =
\begin{cases}
b_k \dfrac{A_{ik} \left(P_{ik}\right)^2}{\sum_{r=1}^{n} A_{rk} \left(P_{rk}\right)^2} & \text{if } \sum_{r=1}^{n} A_{rk} \left(P_{rk}\right)^2 > 0 \\
0 & \text{otherwise}
\end{cases} \\[1em]
\textbf{3.} \qquad &  \arg\max F(A, B, -, B') =E \\
& \text{where } E_{jl} =
\begin{cases}
a_j' \dfrac{B'_{jl} \left(Q_{jl}\right)^2}{\sum_{r=1}^{m'} B'_{jr} \left(Q_{jr}\right)^2} & \text{if } \sum_{r=1}^{m'} B'_{jr} \left(Q_{jr}\right)^2 > 0 \\
0 & \text{otherwise}
\end{cases} \\[1em]
\textbf{4.} \qquad &  \arg\max F(A, B, A', -) =E \\
& \text{where } E_{jl} =
\begin{cases}
b_l' \dfrac{A'_{jl} \left(Q_{jl}\right)^2}{\sum_{r=1}^{n'} A'_{rl} \left(Q_{rl}\right)^2} & \text{if } \sum_{r=1}^{n'} A'_{rl} \left(Q_{rl}\right)^2 > 0 \\
0 & \text{otherwise}
\end{cases}
\end{align*}
Here, maximization of $F(-,B,C,D)$ is over the set 
\begin{equation}\label{eqn:BCA_domain}
\{E \in \R^{n \times m} \mid \big((E,B),(A',B')\big) \in \overline{\Gamma}_\mathrm{d}\},
\end{equation}
with the domains of the other functions defined similarly.
\end{lemma}

\begin{proof} We give the proof for \textbf{1.}, with the remaining claims following by symmetric arguments. With $E$ as defined in \textbf{1.}, we wish to show that $F(A,B,A',B') \leq F(E,B,A',B')$ for all $A$ in the set \eqref{eqn:BCA_domain}. We will prove the following

\smallskip
\noindent {\bf Claim:} For all $i =1,\ldots,n$, 
\[
\sum_{k=1}^m \sqrt{A_{ik} B_{ik}} P_{ik} \leq \sum_{k=1}^m \sqrt{E_{ik}B_{ik}} P_{ik}.
\]

\smallskip
\noindent  Assuming the claim, the proof that $E$ is a maximizer follows easily, as 
\[
F(A,B,A',B') = \sum_{i=1}^n \sum_{k=1}^m \sqrt{A_{ik} B_{ik}} P_{ik} \leq \sum_{i=1}^n  \sum_{k=1}^m \sqrt{E_{ik}B_{ik}}  P_{ik}  = F(E,B,A',B').
\]

To prove the claim, first suppose that $\sum_{k=1}^m B_{ik}P_{ik}^2 = 0$. Then, for all $k=1,\ldots,m$, either $B_{ik} = 0$ or $P_{ik}=0$, in which case
\[
\sum_{k=1}^m \sqrt{A_{ik} B_{ik}} P_{ik} = 0 =  \sum_{k=1}^m \sqrt{E_{ik}B_{ik}} P_{ik}.
\]

It remains to prove the claim for the case that $\sum_{k=1}^m B_{ik}P_{ik}^2 > 0$. By Cauchy-Schwarz, and the definition of $\overline{\Gamma}_\mathrm{d}$, we
have,  
\begin{align*}
    \bsum{k=1}{m}\sqrt{A_{ik}B_{ik}}P_{ik}&\leq \sqrt{\left(\bsum{k=1}{m}A_{ik}\right)\left(\bsum{k=1}{m}B_{ik}P_{ik}^2\right) }  = \sqrt{a_i\bsum{k=1}{m}B_{ik}P_{ik}^2 } \\
    &= \frac{\sqrt{a_i}\bsum{k=1}{m}B_{ik}P_{ik}^2}{\sqrt{\left(\bsum{k=1}{m}B_{ik}P_{ik}^2\right) }}  =\bsum{k=1}{m}\sqrt{\frac{a_iB_{ik}P_{ik}^2}{\bsum{k=1}{m} B_{ik}P_{ik}^2}B_{ik}}P_{ik}=\bsum{k=1}{m}\sqrt{E_{ik}B_{ik}}P_{ik}.
\end{align*}
This completes the proof of the claim, hence that $E$ is a maximizer.

It remains to prove uniqueness. The established inequality $F(A,B,A',B') \leq F(E,B,A',B')$ is strict unless, for each $1\leq i \leq n$ such that $\sum_{r=1}^{m} B_{ir}P_{ir}^2>0$, 
the vector $\{A_{ik}\}_k$ is a scalar multiple of $\{B_{ik}P_{ik}^2\}_k$. Since we know $\sum_{k=1}^{m}A_{ik}=a_i$, equality holds if and only if 
  \begin{equation*}
      \{A_{ik}\}=\left\{\frac{a_i B_{ik}P_{ik}^2}{\sum_{r=1}^{m}B_{ir}P_{ir}^2}\right\}=\{E_{ik}\}.
  \end{equation*}  
This proves that a maximizer must agree with $E$ on indices $i$ with $\sum_{k=1}^m B_{ik}P_{ik}^2 > 0$. We now consider those indices $i$ where $\sum_{k=1}^m B_{ik}P_{ik}^2 = 0$. By Definition of $\overline{\Gamma}_\mathrm{d}$, since $\big((A,B),(A',B')\big)$ is in the interior of $\overline{\Gamma}_\mathrm{d}$, $B_{ik}=0$ if and only if $P_{ik}=0$. Thus, for all $k$, $B_{ik}=0$. Once again using that $\big((A,B),(A',B')\big)$ is interior, this implies that, for all $k$, $\sum_{j=1}^{n'}\sum_{l=1}^{m'}   \Omega^\delta_{ijkl} = 0$. Furthermore, this implies that $A_{ik}=0$, for all $k$; in particular, $\sum_{k=1}^mA_{ik}=0$. However, since $(A,B,A',B')$ is in the interior of $\overline{\Gamma}_\mathrm{d}$, this is true if and only if $\sum_{i=1}^{n}\sum_{j=1}^{n'}\sum_{l=1}^{m'}   \Omega^\delta_{ijkl} = 0$. Thus, by definition of $\overline{\Gamma}_\mathrm{d}$ we must have $\sum_{k=1}^mE_{ik}=0$, so that $E_{ik} = 0$ for all $k$, hence $A_{ik} = E_{ik}$, and we have completed the proof of uniqueness.
\end{proof}

\begin{corollary}\label{corollary:stable_point}
    When initialized on the interior of $\overline{\Gamma}_\mathrm{d}$, the block coordinate ascent algorithm described above converges to a stable point of $F$. 
\end{corollary}
\begin{proof}
When $F$ is restricted to $\overline{\Gamma}_\mathrm{d}$ it attains a unique maximum in each coordinate block. Therefore, by  \cite[Theorem 4.1]{Tseng_2001} (c.f. \cite{Luenberger_1984}) the cyclic block coordinate ascent algorithm will converge to a stable point of $F$.
\end{proof}
\begin{remark}\label{rem:equality}
    Recall from Theorem \ref{thm:CGW_CCOOT_equivalence} that $\CCOT(\mathcal{H}_X,\mathcal{H}_Y) = \CGW(\mathcal{N}_X,\mathcal{N}_Y)$ for discrete networks $\mathcal{N}_X,\mathcal{N}_Y$ such that $\Omega\left(\frac{|\omega_X-\omega_Y|}{2\delta}\right)$ defines a positive definite kernel. Therefore, we can also use our algorithm to approximate the $\CGW$ distance between discrete measure networks satisfying this property.
\end{remark}

\section{Numerical Experiments}\label{section:Numerical_Experiments}
 We now present the results of numerical experiments to support the theory established for the $\CGW$ and $\CCOT$ metrics.

\subsection{Single-Cell Multi-Omics Alignment}

\begin{figure}
\centering
\resizebox{1.\textwidth}{!}{\includegraphics{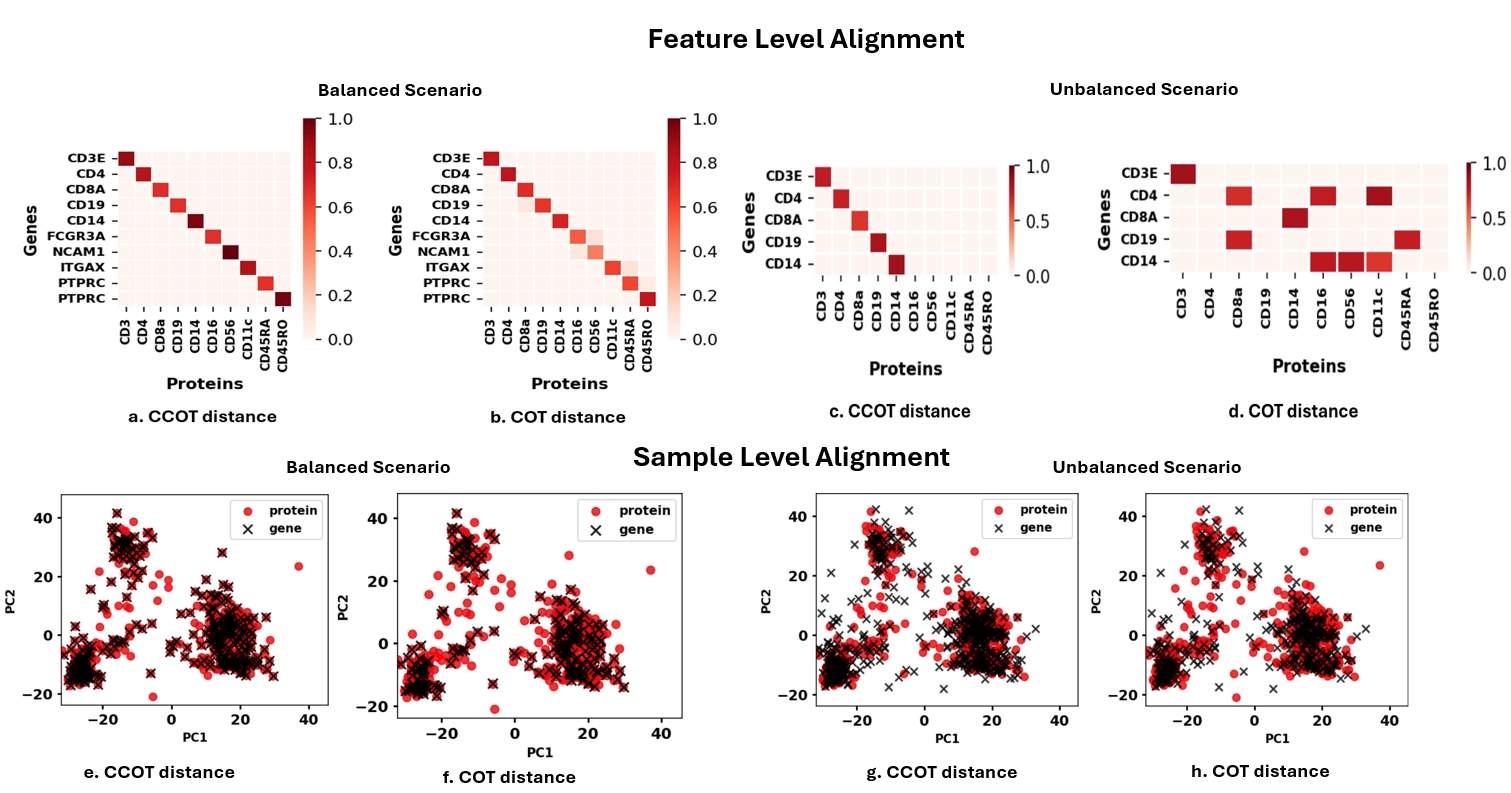}}
\caption{Alignment of a single-cell multi-omics dataset using CCOT (with kernel $\Omega(z)=\overline{\cos}(z)$) and COT distances. Top row (a–d): Feature-level alignment between genes and proteins. Each heatmap represents the transport plan, where rows correspond to genes and columns to proteins. High values indicate strong correspondence between a protein and its encoding gene. (a,b) Balanced setting, where the number of genes and proteins is equal, leading to a near one-to-one alignment structure for a fixed number of cells at 1000. (c,d) Unbalanced setting, where the number of genes and proteins differs while the number of cells is fixed at 1000. In this case, alignment is more challenging due to feature mismatch.
Comparing methods, CCOT (a,c) yields a sharper, near-diagonal structure, indicating more accurate gene–protein matching, whereas COT (b,d) produces a more diffuse transport plan, particularly in the unbalanced setting.
\\Bottom row (e–h): Cell-level alignment across modalities. Each point represents a cell (sample), with red dots corresponding to the protein modality and black crosses to the gene modality. The plots show the distance embeddings after applying standard PCA to understand cell alignment across different modalities (gene and protein). (e,f) Balanced setting, where both CCOT and COT distances achieve reasonable alignment, though CCOT shows tighter overlap between both gene and protein modalities (matching features).(g,h) Unbalanced setting, where the protein modality is downsampled by 25\%, introducing additional mismatch. In this case, CCOT (g) maintains better alignment between the two modalities, while COT (h) exhibits greater dispersion and reduced overlap. Overall, the figure illustrates that CCOT improves both feature-level correspondence (genes–proteins) and sample-level alignment (cells), particularly in challenging unbalanced scenarios.}
\label{fig:single_cell_analysis}
\end{figure}

\begin{figure}
    \centering
    \begin{minipage}{0.4\textwidth}
        \centering
        \includegraphics[width=\linewidth]{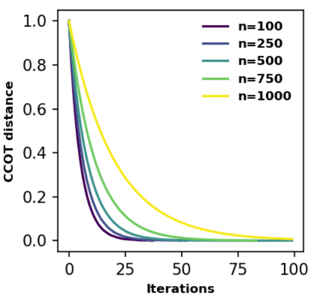}
        \subcaption{}
        \label{fig:convergence_unbalanced}
    \end{minipage}\hfill
    \begin{minipage}{0.4\textwidth}
        \centering
        \includegraphics[width=\linewidth]{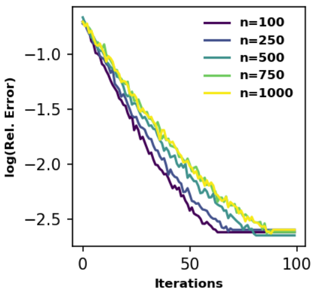}
        \subcaption{}
        \label{fig:logerror_unbalanced}
    \end{minipage}
    \caption{Convergence diagnostics for the unbalanced scenario with different features (5 genes and 10 proteins), across varying numbers of samples (cells). (a) Evolution of the CCOT distance over iterations, demonstrating the solver’s progressive stabilization. (b) Corresponding trends in the relative log error, where a decreasing value signals convergence toward a stable solution.}
    \label{fig:convergence_diagnostics}
\end{figure}

We demonstrate the use of the CCOT metric (as defined in \ref{def:ccot}) and COT metric (as defined in \eqref{def:coot_distance}) to align single-cell multi-omics datasets in scenarios where the cellular populations are only partially shared across modalities. This implies that some cell types are present in both modalities, while others appear in only one modality. As a result, the two datasets do not contain a perfect one-to-one correspondence between cells, making the alignment problem more challenging. Such incomplete overlap often arises due to technical limitations (e.g., assay throughput or capture efficiency) and biological constraints (e.g., cell-type exclusivity in sampling), which prevent comprehensive multi-omics profiling of individual cells. We utilize a benchmark CITE-seq dataset \cite{Stoeckius}, which is closely related to the dataset used in \cite{UCOOT}. While the experimental setups are similar, slight differences exist due to the random sampling of the biological entities such as cells, genes, and proteins. Nonetheless, this allows for an approximate comparison of our qualitative results using CCOT against the UCOT (as defined in \ref{def:UCOOT}) results in \cite{UCOOT}. 

This dataset contains profiles of gene expression and surface protein abundance simultaneously in 1,000 human peripheral blood mononuclear cells (PBMCs) each profiled for 17,014 genes and 10 surface proteins. This dataset is chosen because it contains known biological correspondences between features (i.e., genes and their encoded proteins, such as the CD4 gene and CD4 protein), allowing us to rigorously evaluate CCOT’s ability to jointly align both cells and features. Although gene expression and protein abundance are distinct biological measurements with imperfect correlation, these known correspondences serve as ground truth for assessing feature alignment. 

Here, we denote the gene expression modality by the measure hypernetwork \(\mathcal{H}_X = (X, \mu_X, X', \mu_X', \omega_X)\), where \(X\) corresponds to the set of cells profiled for gene expression (samples) and \(\mu_X\) is the empirical measure over these cells. The set of features \(X'\) represents the genes expressed (features), with \(\mu_X'\) their associated empirical measure. The kernel \(\omega_X\) encodes the relationships between cells and genes. Similarly, the protein modality is represented by the hypernetwork \(\mathcal{H}_Y = (Y, \mu_Y, Y', \mu_Y', \omega_Y)\), where \(Y\) denotes the set of cells profiled for surface proteins (samples) with corresponding measure \(\mu_Y\). The protein set \(Y'\) and measure \(\mu_Y'\) are defined analogously, with \(\omega_Y\) capturing similarity between proteins.

Pre-processing utilizes the Muon package to handle each modality appropriately: the RNA data undergoes quality filtering, normalization, log-transformation, and selection of highly variable genes, while protein data is normalized using centered log-ratio (CLR) transformation to adjust for technical variability. Dimensionality reduction is performed independently on each modality to obtain low-dimensional embeddings suitable for alignment.

To evaluate alignment quality, we use the Fraction of Samples Closer to the True Match (FOSCTTM) metric for cells \cite{Cao,Liu,Demetci}, where a lower score indicates better cell-level alignment, and the proportion of correctly matched gene-protein pairs for feature-level alignment. To test robustness, we simulate three experimental conditions: a balanced scenario with equal numbers of cells (fixed at 1,000) and matched features (10 genes paired with their corresponding 10 proteins); an unbalanced scenario (features) with equal numbers of cells (1,000), but unequal numbers of features (5 genes versus 10 proteins), simulating missing or unmatched features; and unbalanced scenario (cells) where the number of cells differs across modalities to mimic partial cellular overlap.

\paragraph{Balanced Scenario (Features)}
We select equal numbers of cells (1000 fixed) and features (10 genes and their corresponding 10 proteins). Both COT and CCOT accurately align features (Figures \ref{fig:single_cell_analysis} (a)-(b)), but CCOT achieves better cell alignment with a lower FOSCTTM score (0.0059 vs. 0.0234). CCOT’s improved performance potentially stems from its robustness to noise, resulting in a clearer diagonal (darker red) alignment pattern relative to COT distance (as demonstrated theoretically in Theorems \ref{thm:CGW_CCOOT_equivalence} and \ref{thm:robustness}).

\paragraph{Unbalanced Scenario (Features)}
We select an equal number of cells (1000) and align ten proteins against five genes, creating an unbalanced setting. COT fails to recover correct feature correspondences, dispersing alignment away from the diagonal (Figure  \ref{fig:single_cell_analysis} (d)). In contrast, CCOT’s relaxation of the mass conservation constraint allows down-weighting of unmatched protein cells, improving alignment quality (Figure \ref{fig:single_cell_analysis} (c)).

\paragraph{Balanced and Unbalanced Scenario (Samples)}
We also consider the case when there are equal and unequal numbers of cells across both modalities. For the balanced case, we align 1000 cells for both the modalities. For the unbalanced case, we downsample the protein modality by 25\% and perform alignment with the full set of cells in the gene modality. We calculate the FOSCTTM score for all cells with known true matches in the dataset and report the average values. CCOT maintains a low FOSCTTM score (0.0073 versus 0.0059 in the balanced scenario), indicating robust performance, whereas COT had a more pronounced decline (0.1104 compared to 0.0234 in the balanced scenario). To further illustrate the alignment quality, we visualize the aligned cells using Principal Component Analysis (PCA) plots (Figure \ref{fig:single_cell_analysis} (e)-(h)), which reveal how well cells from both modalities cluster together after alignment.

Finally, in Figure \ref{fig:convergence_diagnostics}, we aim to align the unequal numbers of features (5 genes and 10 proteins) by varying the number of cells fixed for both modalities, we notice a decreasing trend in convergence diagnostics, including both the CCOT distance and the relative log errors in iterations. This behavior indicates that the solver progressively stabilizes its solution as the number of iterations increases. The observed reduction in these metrics reflects increasingly consistent and diminishing updates to the solver variables, providing evidence that the algorithm is approaching a stable point (see Corollary \ref{corollary:stable_point}). Consequently, these results support the reliability and robustness of the solver and the metric in achieving a meaningful and stable alignment between the datasets.

\subsection{Equality of Semi-Couplings for CCOT for Measure Networks}\label{subsubsection:COOT_equivalence}

\begin{figure}
\centering
\includegraphics[width=12.5cm]{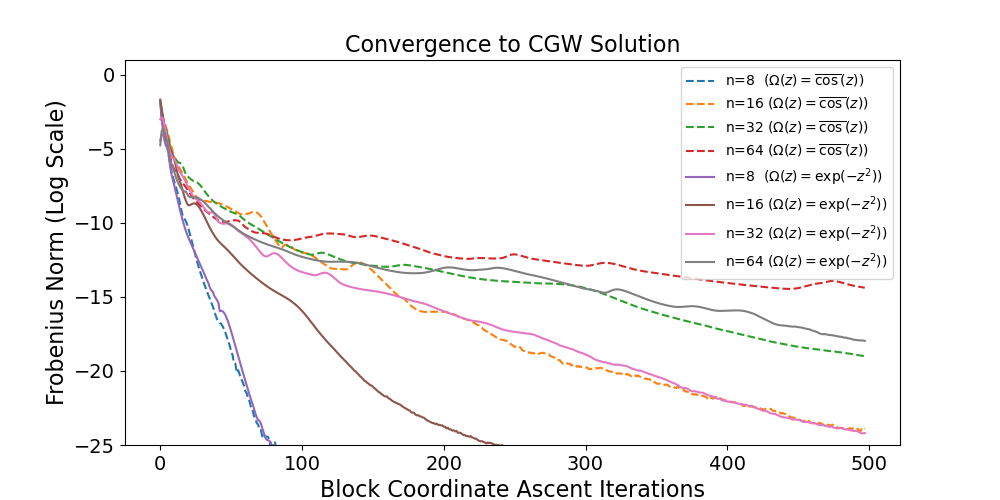}
\caption{Figure of the logarithm of average Frobenius distance between the semi-couplings at each iteration of the proposed algorithm for 20 random pairs of point clouds with $n$ points.}\label{fig:frob_plot}
\end{figure}
In Remark~\ref{rem:equality}, we discuss how our algorithm can be applied to compute the conic Gromov-Wasserstein distance between measure networks under certain conditions. Our method converges to a valid solution when the optimal pair of semi-couplings $(\gamma_0, \gamma_1), (\gamma_0', \gamma_1') \in \Gamma(\mu_X, \mu_Y)$ are equal---that is, $(\gamma_0, \gamma_1) = (\gamma_0', \gamma_1')$. In the discrete setting, for discrete semi-coupling pairs $(A, B), (A', B') \in \Gamma_\mathrm{d}(\mu_X, \mu_Y)$, this condition translates to the Frobenius norm identity:
\[\|A - A'\|_F^2 + \|B - B'\|_F^2 = 0.\]
Figure~\ref{fig:frob_plot} shows the evolution of the Frobenius norm of the difference between successive semi-couplings during iterations of the block coordinate ascent algorithm and the Frobenius norm between them converges toward zero. The results are averaged over networks constructed from point clouds in $\mathbb{R}^3$, using the squared Euclidean distance.
We present results using $\Omega(z) = \exp(-z^2)$ and $\Omega(z) = \overline{\cos}(z)$. Although no theoretical guarantee ensures equality of the optimal semi-couplings, the Frobenius norm between them still converges toward zero over iterations of the block coordinate ascent algorithm, offering empirical evidence of convergence of the algorithm to conic Gromov-Wasserstein distance.
\subsection{Digit Classification Comparison}
\begin{figure}
    \centering
    \includegraphics[width=1.\linewidth]{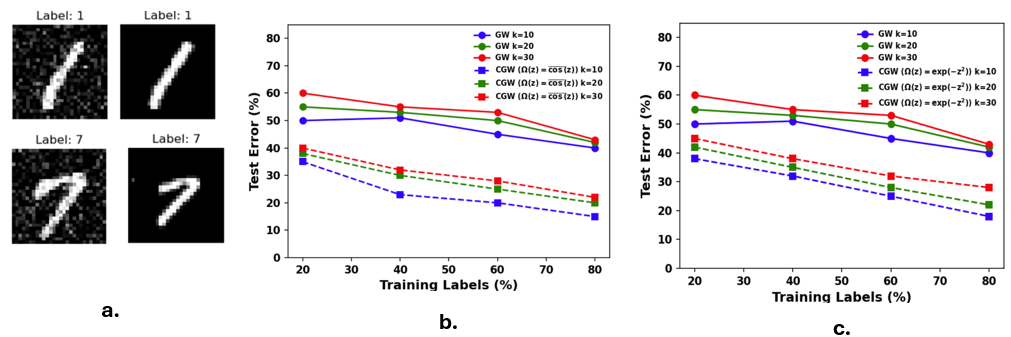}
    \caption{(a). 2D representation of the MNIST dataset. (b)-(c) Comparison of test error on the MNIST digit dataset by varying the number of training labels for a $k$-NN classifier with different values of $k$. Experiments were conducted using two kernel choices: $\Omega(z)=\overline{\cos}(z)$ and  $\Omega(z)=\exp(-z^2)$.}
    \label{fig:MNIST_digits}
\end{figure}

We address a binary classification task using the MNIST dataset \cite{LeCun_1998}, which contains 70,000 grayscale images of handwritten digits from $0$ to $9$, each sized $28 \times 28$ pixels. For our experimental setup, we focus exclusively on the digits 1 and 7, selecting $N=1000$ samples from each class to form a dataset of size $2N = 2000$. To enhance variability within the data, additive Gaussian noise is introduced to the pixel intensities of half the images (see Figure~\ref{fig:MNIST_digits}(a)).

Each image is modelled as a measure network $ \mathcal{N} = (V, \mu, \omega) $, where the set of nodes $V = \{ v_i \}_{i=1}^{784} \subset \mathbb{R}^2$ corresponds to fixed pixel coordinates, with $ v_i $ denoting the spatial location of pixel $ i $. The discrete probability measure $\mu$ on $V$ is defined by normalizing pixel intensities:
\[
\mu = \sum_{i=1}^{784} p_i \, \delta_{v_i}, \quad \text{where} \quad p_i = \frac{I(v_i)}{\sum_{j=1}^{784} I(v_j)},
\]
and $ I(v_i) $ denotes the grayscale intensity at pixel $ v_i $.

The adjacency kernel $\omega = (\omega_{ij}) \in \{0,1\}^{784 \times 784}$ represents the 8-neighborhood connectivity structure between pixels denoted by,
\[
\omega_{ij} = \begin{cases}
1, & \text{if } v_j \text{ is one of the 8 nearest neighbors of } v_i \text{ based on Euclidean distance}, \\
0, & \text{otherwise}.
\end{cases}
\]
Here, $\mathcal{N}$ denotes a generic measure network representing any image in the dataset.

More precisely, measure networks $\mathcal{N}_X$ and $\mathcal{N}_Y^{(m)}$ for $m=1,\dots,2000$ are all constructed according to the general definition of $\mathcal{N}$ above. In particular, $\mathcal{N}_X = (X, \mu_X, \omega_X)$ corresponds to the first image which is digit 1 as a reference, and each $\mathcal{N}_Y^{(m)} = (Y, \mu^{(m)}_Y, \omega^{(m)}_Y)$ is constructed analogously from the $m$-th image in the dataset.

We compute the $\CGW(\mathcal{N}_X,\mathcal{N}_Y^{(m)})$ as defined in \eqref{eqn:def_cgw}, yielding the vector of distances:
\[
d_{\CGW} := \bigl(\CGW(\mathcal{N}_X,\mathcal{N}_Y^{(m)})\bigr)_{m=1}^{2000} \in \mathbb{R}^{2000}.
\]

For comparison, the standard $\GW_2(\mathcal{N}_X,\mathcal{N}_Y^{(m)})$ is computed similarly:
\[
d_{\GW_2} := \bigl(\GW_2(\mathcal{N}_X,\mathcal{N}_Y^{(m)}) \bigr)_{m=1}^{2000} \in \mathbb{R}^{2000}.
\]

Both $d_{\CGW}$ and $d_{\GW_2}$ serve as feature vector (inputs) for  subsequent classification using a $k$-nearest neighbors (k-NN) classifier, with $ k \in \{10, 20, 30\} $ and varying training label rates from 20\% to 80\%. For each training label rate, 100 independent trials are conducted by randomly subsampling the training data, and the average test error is reported. Results are summarized in Figure~\ref{fig:MNIST_digits}(b)-(c) for different choice of the kernel $\Omega(z)$. We note that changing the reference image to digit 7 and following the same process outlined to compute the feature vector yields similar classification accuracy, with a difference well within the standard error margin of approximately $\pm1.5\%$. This indicates that the classification performance is robust to the choice of reference measure network.

Three salient observations emerge from Figure~\ref{fig:MNIST_digits}(c). Firstly, the $\CGW$ metric is more robust to outliers than the $\GW_2$ metric (see Theorem \ref{thm:robustness}), since it takes into account the variation in pixel intensity and the spatial coordinates, whilst the $\GW_2$ only handles the latter. We observe that as the number of training labels increases, the test error decreases. Secondly, the feature vector corresponding to the $\CGW$ metric serves as a better latent space for applying state-of-the-art machine learning algorithms for better classification performance. Thirdly, from a graph perspective \cite{MCAO_2025,MCAO_2026}, this can be interpreted as the error decreasing with the reduction in the graph's length scale (choice of $k$). This suggests that with increased training data, the model more effectively captures local geometric relationships within the data (cf. \cite[Theorem 2.1, Theorem 2.4]{calder2023rates}).

\subsection{Comparison of Algorithms for CGW}
\begin{figure}
    \centering
    \includegraphics[width=\linewidth]{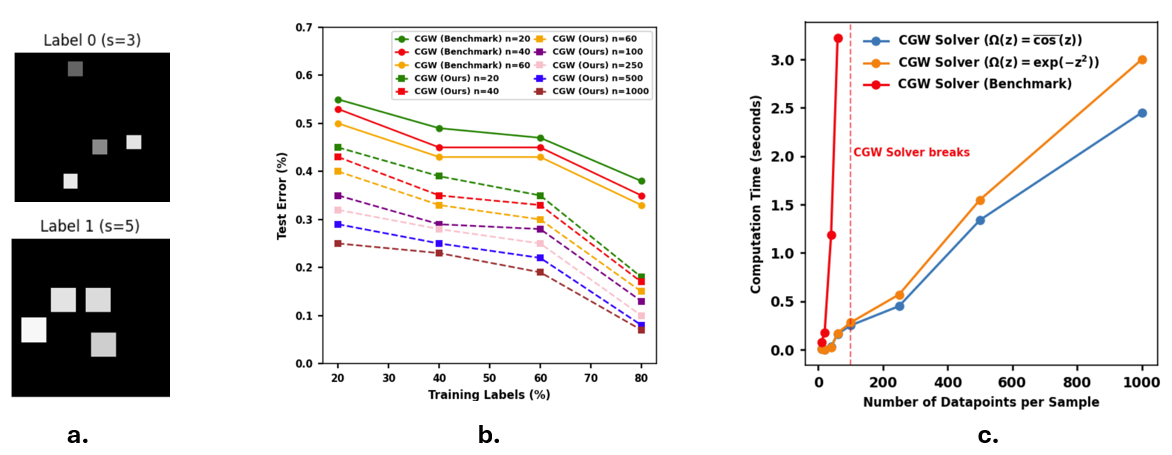}
    \caption{(a) 2D representations of images composed of square regions with side lengths $s = 3$ and $s = 5$, where each square is assigned an intensity value independently sampled from a uniform distribution over the interval [0,1]. (b) Error plots showing test error versus the number of training labels for different values of data points $n$ per image, using a $k$-NN classifier with $k=15$ (fixed). The choice of the kernel was $\Omega(z)=\overline{\cos}(z)$. (c) Computation time for the $\CGW$ distance using our proposed method with both kernels, namely, $\Omega(z)=\overline{\cos}(z)$ and $\Omega(z)=\exp(-z^2)$. We compare the performance relative to the benchmark algorithm from \cite{NEURIPS2021}, highlighting the scalability and computational efficiency of our approach.}
    \label{fig:squares_dataset}
\end{figure}

In this experiment, we consider the $g$-squares dataset \cite{Jan_2022}, where each sample is a grayscale image of fixed size \(32 \times 32\) pixels containing $g \in \mathbb{N}$ distinct non-overlapping squares. In this study, we fix $g=4$ and assume that all squares within each image share the same side length $s=3$ or $s=5$. The brightness value $b_h$ for each square $h$, with \(h = 1, \ldots, 4\), is independently drawn from a uniform distribution over the interval \([0,1]\), i.e., \(b_h \sim \mathrm{Uniform}(0,1)\).

For each pixel coordinate $v \in \{1,\ldots,32\}^2$, the pixel intensity $I(v)$ is defined as the aggregate brightness contributed by all squares containing the pixel $v$:
\[
I(v) = \sum_{h=1}^4 b_h \, \mathbf{1}_{\{v \in \mathrm{square}_h\}},
\]
where $\mathbf{1}_{\{v \in \mathrm{square}_h\}}$ denotes the indicator function that evaluates to 1, if the pixel $v$ lies within the $h$-th square; and 0 otherwise. 

For each generated image, a subset of pixels $V = \{v_i\}_{i=1}^n$ is sampled uniformly at random, where $n=\{20,40,60,100,250,500,1000\}$ denote the number of data points (pixels). Corresponding to the sampled pixels, we define a measure network $\mathcal{N} = (V, \mu, \omega)$, where the node set $V$ consists of the sampled pixel coordinates $v_i \in \mathbb{R}^2$. The discrete probability measure $\mu$ supported on $V$ is given by normalizing the pixel intensities at the sampled points:
\[
\mu = \sum_{i=1}^n p_i \, \delta_{v_i}, \quad \text{where} \quad p_i = \frac{I(v_i)}{\sum_{j=1}^n I(v_j)}.
\]
The adjacency matrix $\omega \in \{0,1\}^{n \times n}$ encodes spatial connectivity among the sampled pixels, with entries defined as
\[
\omega_{ij} = \begin{cases}
1, & \text{if } v_j \text{ is one of the 4 nearest neighbors of } v_i \text{ based on Euclidean distance}, \\
0, & \text{otherwise}.
\end{cases}
\]
Here, $\mathcal{N}$ denotes a generic measure network representing any image in the dataset.

A dataset comprising $M=100$ images is constructed as described above, with 4 non-overlapping squares. Half of the images contain 4 squares with side length $s=3$ and the other half contain 4 squares with side length $s=5$. The binary classification task is to discriminate images containing squares of side length 3 from those containing squares of side length 5.

For a fixed $n$, let $\mathcal{N}_X = (X, \mu_X, \omega_X)$ denote the measure network associated with the first image ($s=3$), and let $\mathcal{N}_Y^{(m)} = (Y, \mu_Y^{(m)}, \omega_Y^{(m)})$ for $m=1,\ldots,100$ denote the measure networks corresponding to the remaining images in the dataset. Both $\mathcal{N}_X$ and $\mathcal{N}_Y^{(m)}$ are constructed in accordance with the general definition of the measure network $\mathcal{N}$ described above. We compute the $\CGW(\mathcal{N}_X, \mathcal{N}_Y^{(m)})$, as defined in~\eqref{eqn:def_cgw}, resulting in the distance vector
\[
d_{\CGW} := \bigl(\CGW(\mathcal{N}_X, \mathcal{N}_Y^{(m)})\bigr)_{m=1}^{100} \in \mathbb{R}^{100}.
\]

These distances serve as feature vector (inputs) for a $k$-nearest neighbors classifier with fixed $k=15$ to perform binary classification between the two classes of images. Our results demonstrate that the proposed algorithm significantly improves computational efficiency, enabling the use of larger data points $n$ per image. In contrast, the benchmark $\CGW$ solver introduced in \cite{NEURIPS2021} encounters computational limitations and fails to process datasets exceeding $n=100$ points per image, as illustrated in Figure~\ref{fig:squares_dataset}(c). 

Furthermore, Figure~\ref{fig:squares_dataset}(b) shows that as the number of training labels increases for images with $n=1000$ data points, the relative classification error of the $k$-NN classifier decreases, as indicated by the brown dotted lines. We note that changing the reference measure network $\mathcal{N}_X$ (image with squares of side length $s=5$) and following the same process outlined to compute the feature vector yields similar classification accuracy, with a difference well within the standard error margin of approximately $\pm2.5\%$. This indicates that the classification performance is robust to the choice of reference measure network. Conversely, classification results for the benchmark solver are unavailable beyond $n=100$ due to its computational infeasibility, underscoring the scalability advantage of our proposed method in this classification task.

\section*{Acknowledgments}
MCAO acknowledges receipt of funding from the Health Data Research UK-The Alan Turing Institute Wellcome studentship (Grant Ref: 218529/Z/19/Z) and the Cambridge Trust scholarship from the Cambridge Commonwealth, European and International Trust (CCEIT). TN acknowledges support from NSF grants DMS--2107808, DMS--2324962 and CIF--2526630.

\bibliographystyle{plain}
\bibliography{references}

\end{document}